\documentclass[11pt]{article}

\usepackage[margin=1in]{geometry}
\usepackage[utf8]{inputenc}
\usepackage[T1]{fontenc}
\usepackage{hyperref}
\hypersetup{
  hidelinks,
  pdftitle={Finite-Iteration Local Dynamics and Warm Starts for Alternating Power Iteration in Spiked Tensor PCA},
  pdfauthor={Yanjin Xiang and Zhihua Zhang},
  pdfkeywords={Tensor PCA, alternating power iteration, warm starts, centered Gram initialization, finite-iteration dynamics}
}
\usepackage{url}
\usepackage{booktabs}
\usepackage{tabularx}
\usepackage{nicefrac}
\usepackage{amsmath,amssymb,amsfonts,amsthm,mathtools,bm}
\usepackage{microtype}
\usepackage{xcolor}
\usepackage[numbers,compress]{natbib}
\usepackage{comment}

\newtheorem{theorem}{Theorem}[section]
\newtheorem{proposition}[theorem]{Proposition}
\newtheorem{lemma}[theorem]{Lemma}

\newtheorem{definition}[theorem]{Definition}
\newtheorem{corollary}[theorem]{Corollary}
\theoremstyle{remark}
\newtheorem{remark}[theorem]{Remark}

\newcommand{\Sph}{\mathbb S}
\newcommand{\R}{\mathbb R}
\newcommand{\Pbb}{\mathbb P}
\newcommand{\Ebb}{\mathbb E}
\newcommand{\calA}{\mathcal A}
\newcommand{\calB}{\mathcal B}
\newcommand{\calE}{\mathcal E}
\newcommand{\calF}{\mathcal F}
\newcommand{\Proj}{\bm{\Pi}}
\newcommand{\noise}{\bm{Z}}
\newcommand{\op}{\mathrm{op}}
\newcommand{\loc}{\mathrm{loc}}
\newcommand{\aff}{\mathrm{aff}}
\newcommand{\ctr}{\mathrm{ctr}}

\title{Finite-Iteration Local Dynamics and Warm Starts for Alternating Power Iteration in Spiked Tensor PCA}

\author{
  Yanjin Xiang \\
  Peking University \\
  \texttt{2401110086@stu.pku.edu.cn}
  \and
  Zhihua Zhang \\
  Peking University \\
  \texttt{zhzhang@math.pku.edu.cn}
}

\begin{document}

\maketitle

\begin{abstract}
We study simultaneous alternating power iteration for fixed-order asymmetric rank-one spiked tensor models.  Our main contribution is a finite-iteration local theory that is independent of any particular initialization.  Once the iterates enter a sufficiently small neighborhood of the planted rank-one direction, their error decomposes into a geometrically decaying transient and an intrinsic noise floor caused by fixed orthogonal noise contractions at the planted point.  The deterministic finite-sample conditions are stated explicitly, but under a coarse fixed-order multilinear noise event they reduce to a conservative high-signal regime for fixed or slowly expanding local radii.

We then separate the warm-start mechanism from any specific spectral construction.  A generic one-sweep principle shows that, if a sign-compatible initializer has correlation \(\gamma_N\), first-sweep noise level \(a_N\), and \(a_N/(\gamma_N^{d-1}\omega_{N,d})\to0\), then one can choose an expanding radius \(r_N=o(\omega_{N,d})\) for which the first sweep enters the local basin.  After entry, the local affine contraction yields convergence to the unique informative local fixed point in that basin.  For centered-Gram initialization, we verify the required correlation and same-sample first-sweep noise bound under i.i.d. finite-fourth-moment noise by a signal-preserving noise-only leave-one comparison and an averaged leave-one slice-contraction estimate, which we call a pressed-back estimate.  The leave-one comparison keeps the spike fixed and averages over the deleted coordinate, so planted coordinates enter through \(\ell_2\)-weighted sums rather than worst-case incoherence bounds.
\end{abstract}

\section{Introduction}\label{sec:introduction}

Tensor methods provide a natural language for multiway data and higher-order moments.  Their
algorithmic and perturbation-theoretic foundations go back to the variational definitions of
tensor singular values and eigenvectors \citep{Lim2005} and to the broad literature on CP and
multilinear tensor decompositions \citep{KoldaBader2009}.  In statistics and machine learning,
tensor power iterations play a central role in moment-based latent-variable estimation, where
orthogonal or nearly orthogonal tensor decompositions can be recovered by robust power methods
\citep{AnandkumarGeHsuKakadeTelgarsky2014}.  A complementary line of work studies spiked
Tensor PCA, where an unknown rank-one signal is observed through a high-dimensional random
tensor.  The rank-one plus noise model was formalized as a statistical model for Tensor PCA by
\citet{MontanariRichard2014}; its statistical thresholds and maximum-likelihood behavior have
been further analyzed in, for example, \citet{JagannathLopattoMiolane2020}.  On the algorithmic
side, sum-of-squares and spectral methods reveal a sharp distinction between information-theoretic
and polynomial-time recovery regimes \citep{HopkinsShiSteurer2015}.  Random tensor spectral
norm estimates, such as those of \citet{TomiokaSuzuki2014}, are also indispensable in understanding
which multilinear noise contractions can be controlled uniformly.

In this paper we study a local algorithmic question that is different from the global statistical
threshold problem.  We fix the tensor order \(d\ge3\) and consider the asymmetric rank-one spiked
tensor.  For each mode \(i\in[d]:=\{1,\ldots,d\}\), let
\(\bm{x}^{(i)}\in\Sph^{n_i-1}\) be deterministic and let
\(\bm W\in\R^{n_1\times\cdots\times n_d}\) be a noise tensor with independent centered
unit-variance entries.  Set \(N=\sum_{i=1}^d n_i\).  The order-\(d\) asymmetric rank-one spiked
tensor is
\begin{equation}
\bm{T}
=
\beta \bm{x}^{(1)}\otimes\cdots\otimes \bm{x}^{(d)}
+
\frac1{\sqrt N}\bm{W}.
\end{equation}

For $U=(\bm{u}^{(1)},\ldots,\bm{u}^{(d)})$, define
\[
\bm{T}(U^{(-i)})
:=
\bm{T}(\bm{u}^{(1)},\ldots,\bm{u}^{(i-1)},\cdot,\bm{u}^{(i+1)},\ldots,\bm{u}^{(d)})
\in\R^{n_i}.
\]
Throughout this paper we write \(\bm Z:=N^{-1/2}\bm W\) for the normalized noise tensor, so that
\(\bm Z(U^{(-i)})\) denotes a mode-wise contraction of the normalized noise.  The algorithm studied
here is the simultaneous alternating power map, which normalizes each mode-wise tensor contraction
at every step:
\begin{equation}
\calA(U)
=
\left(
\frac{\bm{T}(U^{(-1)})}{\|\bm{T}(U^{(-1)})\|},
\ldots,
\frac{\bm{T}(U^{(-d)})}{\|\bm{T}(U^{(-d)})\|}
\right).
\end{equation}

\paragraph{Standing assumptions and regimes.}
The paper uses the following layered assumptions and results.  The local recursion and contraction results are
deterministic once the subset-indexed noise event \(\mathcal F_{d,N}(L,\Theta)\) holds; no
initializer is built into those statements.  The crude verification of this local event is stated for
independent centered unit-variance noise entries with uniformly bounded fourth moments and fixed
aspect ratios \(n_i\asymp N\).  The same-sample centered-Gram warm-start verification is more
specialized and is stated conservatively under i.i.d. centered unit-variance entries with finite
fourth moment.  Throughout the high-signal reading, the tensor order \(d\ge3\) is fixed and
\[
\omega_{N,d}:=\frac{\beta}{N^{(d-2)/4}}\to\infty.
\]
All probability verifications are asymptotic in \(N\).  The finite-iteration aspect refers to
deterministic finite-sample inequalities that hold for every fixed iteration time once the
corresponding noise-control event is in force.

\begin{center}
\begingroup
\small
\setlength{\tabcolsep}{4pt}
\renewcommand{\arraystretch}{1.15}
\begin{tabular}{@{}
>{\raggedright\arraybackslash}p{0.17\textwidth}
>{\raggedright\arraybackslash}p{0.33\textwidth}
>{\raggedright\arraybackslash}p{0.20\textwidth}
>{\raggedright\arraybackslash}p{0.22\textwidth}@{}}
\toprule
Result & Input & Signal scale & Output \\
\midrule
Local dynamics & deterministic event \(\mathcal F_{d,N}(L,\Theta)\) and a point in
\(\calB_r^{(d)}\) & finite-sample inequalities & affine recursion and local fixed point \\
Crude event verification & independent centered unit-variance entries with bounded fourth
moments & fixed-order high-signal reading & \(L=O(1)\), \(\Theta=O(N^{(d-2)/2})\) \\
Generic warm start & sign-compatible correlation \(\gamma_N\) and
\(L_{\rm init}^{(d)}=O_{\mathbb P}(a_N)\) & \(a_N/(\gamma_N^{d-1}\omega_{N,d})\to0\) &
one-sweep entry for some \(r_N=o(\omega_{N,d})\) \\
Centered-Gram warm start & i.i.d. finite-fourth-moment noise & \(\omega_{N,d}\to\infty\) &
same-sample entry and \(O_{\mathbb P}(1/\beta)\) local error \\
\bottomrule
\end{tabular}
\endgroup
\end{center}

Our focus is not on threshold-optimal global recovery.  The goal is to identify the finite-sample
local law once the iterate is close to the planted rank-one direction and to verify one natural
same-sample warm start in a robust finite-fourth-moment setting.  Two features make this local
problem nontrivial.  First, the natural local recursion is affine rather
than homogeneous.  Even if the current iterate exactly equals the planted signal, fixed orthogonal
noise contractions such as \(N^{-1/2}\bm{W}(\bm{x}^{(1)},\ldots,\cdot,\ldots,\bm{x}^{(d)})\) are typically nonzero.
Consequently, the alternating power iterates do not converge to zero error; instead, they have a
geometric transient plus an intrinsic \(O(1/\beta)\) noise floor.  Second, initialization should be
separated from the local dynamics.  The local theory needs only a warm start whose mode-wise
correlations with the planted directions and whose first-sweep noise contraction jointly satisfy the
scale relation in Theorem~\ref{thm:warm-start-d}; in the constant-correlation case this reduces to
\[
L_{\rm init}^{(d)}(U_0) :=\max_i\|N^{-1/2}\bm{W}(U_0^{(-i)})\|
\]
being tight.  Centered-Gram initialization is one same-sample way to verify these abstract inputs, but
it is not built into the alternating-power local theory.  The verification below uses a
signal-preserving noise-only leave-one comparison and an averaged slice-contraction estimate, so it does not
require coordinate-incoherence or an additional centered-Gram feasibility assumption on the planted
directions.

\paragraph{Summary of results.}
Conditional on the subset-indexed multilinear event \(\mathcal F_{d,N}(L,\Theta)\), the simultaneous
alternating-power map satisfies a deterministic local recursion of the affine form
\[
\delta_{t+1}
\le
\frac{C_dL}{\beta}
+
\left(
\frac{C_dL}{\beta}
+
\frac{C_d\Theta r}{\beta^2}
\right)\delta_t .
\]
The affine term comes from fixed orthogonal noise contractions at the planted point and produces an
intrinsic statistical floor; the coefficient of \(\delta_t\) controls local stability.

Solving this recursion gives, for every finite pair of local times \(s\le t\),
\begin{equation}\label{eqn:local-error-bound}
\delta_t
\le
q_{\beta,d}^{\,t-s}\delta_s
+
\frac{A_{\beta,d}(1-q_{\beta,d}^{\,t-s})}{1-q_{\beta,d}},
\qquad
A_{\beta,d}:=\frac{C_dL}{\beta},
\quad
q_{\beta,d}:=\frac{C_dL}{\beta}
+
\frac{C_d\Theta r}{\beta^2}.
\end{equation}
Thus the deterministic bounds are finite-time in the iteration variable: after entry into the local
basin, they control every finite pair of times \(s\le t\).  Formula
\eqref{eqn:local-error-bound} separates the geometrically decaying transient from the nonzero
\(O(1/\beta)\) noise floor, and a two-point contraction gives uniqueness of the local informative
fixed point.

Equivalently, in the conservative fixed-order high-signal regime used here,
\[
\omega_{N,d}:=\frac{\beta}{N^{(d-2)/4}}\to\infty,
\]
the finite-sample conditions appearing in the deterministic statements should be read as
scale conditions ensuring signal dominance, self-mapping of the local basin, and
two-point contraction.  With \(L=O(1)\), \(\Theta=O(N^{(d-2)/2})\), and either fixed
\(r\) or \(r_N=o(\omega_{N,d})\), they imply
\[
q_{\beta,d}
=
O\!\left(\frac1\beta+\frac{r_N}{\omega_{N,d}^2}\right)=o(1),
\]
so the local statement becomes
\[
\delta_t\le q_{\beta,d}^{\,t-s}\delta_s+O(\beta^{-1}).
\]

The warm-start part is deliberately formulated at the generic level first.  For a prescribed local
radius \(r_N\), Theorem~\ref{thm:warm-start-d} requires
\(a_N/(\gamma_N^{d-1}r_N)\to0\) and \(r_N=o(\omega_{N,d})\).  Equivalently, if after a
sign-compatible choice of the planted tuple an initializer has mode-wise correlation
\(\gamma_N\) and first-sweep noise level \(a_N\) satisfying
\[
\frac{a_N}{\gamma_N^{d-1}\omega_{N,d}}\to0.
\]
then Corollary~\ref{cor:generic-warm-start-local-d} chooses a suitable expanding radius
\(r_N=o(\omega_{N,d})\) for which the first sweep enters the local basin.  For centered mode-Gram
eigenvectors, we verify both inputs under i.i.d. finite-fourth-moment noise.  The only
centered-Gram-specific step is the same-sample proof of
\(L_{\rm init}^{(d)}(U_0)=O_{\mathbb P}(1)\), where a signal-preserving noise-leave-one comparison
reduces the dependence to averaged leave-one eigenvector motion and a pressed-back
slice-contraction estimate.

\begin{center}
\fbox{\begin{minipage}{0.94\textwidth}
\paragraph{Informal main theorem.}
Combining centered-Gram weak recovery, the same-sample first-sweep bound, and the local
contraction theorem gives the following consequence: under i.i.d. finite-fourth-moment noise and
\(\omega_{N,d}\to\infty\), one centered-Gram initialization followed by one simultaneous
alternating-power sweep enters every basin \(\calB_{r_N}^{(d)}\) with
\(r_N\to\infty\) and \(r_N=o(\omega_{N,d})\).  From then on the iterates have a geometric
transient and an \(O(\beta^{-1})\) local noise floor.
\end{minipage}}
\end{center}

\paragraph{Related work.}
Tensor spectral problems have several closely related but conceptually distinct lines of
literature.  The deterministic multilinear algebra of tensor singular values, eigenvectors,
and tensor decompositions goes back to the foundational work of
\citet{Lim2005} and \citet{Qi2005}, and to the broader CP/Tucker decomposition literature
summarized by \citet{KoldaBader2009}.  Alternating least-squares and power-type methods are
classical computational tools in this area; see, for example, \citet{DeLathauwerDeMoorVandewalle2000}
and the survey of \citet{KoldaBader2009}.  These works provide the variational and
algorithmic background for viewing Tensor PCA as a nonlinear singular-vector problem, but
they do not by themselves address the high-dimensional random perturbation behavior of the
alternating power map.

In statistics and machine learning, tensor power methods became especially prominent through
moment-based estimation for latent-variable models.  In that setting, one typically works with
orthogonal or nearly orthogonal tensor decompositions, and robust tensor power iteration can
recover the underlying components under perturbation assumptions; see
\citet{AnandkumarGeHsuKakadeTelgarsky2014}.  The model studied in the present paper is
different in two respects.  First, we consider an asymmetric rank-one spike corrupted by a
full high-dimensional noise tensor rather than an approximately orthogonal finite-rank tensor.
Second, our goal is not merely to state a perturbation guarantee for a tensor decomposition
algorithm, but to resolve the finite-iteration local recursion of simultaneous alternating
power iteration and to identify its nonzero noise floor.

The spiked Tensor PCA model was formulated as a high-dimensional statistical model by
\citet{MontanariRichard2014}.  A large subsequent literature has studied its statistical
thresholds, likelihood landscape, and algorithmic barriers.  From the statistical side,
works such as \citet{JagannathLopattoMiolane2020} analyze thresholds and maximum-likelihood
behavior, while spin-glass methods and random landscape analyses, including
\citet{AuffingerBenArousCerny2013}, \citet{BenArousMeiMontanariNica2019}, and related
work, clarify the geometry of high-dimensional tensor objectives.  These results are mostly
concerned with global statistical structure, free-energy or likelihood landscapes, and the
location of informative critical points.  By contrast, our results are local and algorithmic:
conditional on entering a neighborhood of the planted direction, we give an explicit
finite-time recursion for the alternating power iterates and a generic warm-start principle;
centered-Gram is treated separately as one same-sample verification of the warm-start inputs under
i.i.d. finite-fourth-moment noise.

There is also a substantial algorithmic literature on the gap between information-theoretic
and polynomial-time recovery for Tensor PCA.  Sum-of-squares methods and related spectral
algorithms, such as those of \citet{HopkinsShiSteurer2015} and
\citet{HopkinsSchrammShiSteurer2016}, establish algorithmic guarantees at signal strengths
well above the information-theoretic threshold and explain why low-degree or spectral
procedures face intrinsic barriers.  Our high-signal scale is of a different nature: it is the
scale at which a local alternating-power analysis with a coarse but sufficient multilinear
noise-control event becomes contractive.  The paper therefore does not attempt to sharpen
the global algorithmic threshold.  Instead, it isolates what happens after a spectral
warm start has supplied weak correlation and then quantifies how one same-sample alternating
sweep enters the local \(O(1/\beta)\) basin.

Recent work has also analyzed the dynamics of tensor power iteration more directly.  For example,
\citet{HuangLiuYangZhu2022} study power iteration for asymmetric Tensor PCA and its statistical
inference consequences, while \citet{WuZhou2024} give a sharp analysis of tensor power iteration
for spiked Tensor PCA, including the random-initialization dynamics.  Our focus is complementary:
we do not try to characterize the full random-initialization trajectory.  Instead, once an initializer has
entered a weakly correlated chart, we derive a finite-iteration local affine recursion and verify,
for the same-sample centered-Gram initializer, the one-sweep entry into that local basin.

Random tensor norm estimates form another important input.  Bounds such as those of
\citet{TomiokaSuzuki2014} and more general results on Gaussian and non-Gaussian chaoses
control worst-case multilinear contractions of random tensors.  In the local part of this
paper, we deliberately use crude matricization bounds, since they are sufficient for fixed
order and lead to transparent deterministic recursions.  The centered-Gram warm-start
analysis, however, requires a more delicate distinction: the first-sweep noise is not a
fixed-direction contraction because the centered-Gram eigenvectors depend on the same
noise tensor.  A worst-case full-slice injective norm is too crude for fixed order
\(d\ge4\).  We therefore use signal-preserving noise-leave-one comparisons and a pressed-back
coordinate-sum estimate: singleton perturbations produce vector contractions, pair perturbations
produce matrix contractions, and higher-order terms are controlled by averaged leave-one stability
without replacing the full coordinate sum by a worst-case slice norm.  This is
the main technical difference between our warm-start argument and a direct random tensor norm
estimate.  Because the leave-one construction removes only noise and keeps the deterministic spike
unchanged, it does not require the planted vectors to be coordinate-incoherent.

Finally, the order-three specialization in this paper serves as a bridge between the general
fixed-order notation and the familiar matrix-slice case.  When \(d=3\), the noise-leave-one slice
contains ordinary random matrices, so the same-sample centered-Gram argument has a familiar
matrix interpretation.  For \(d\ge4\), the corresponding slice is a
higher-order tensor, and a direct injective-norm bound would lose the scale needed for the
natural high-signal regime.  The general fixed-order proof replaces that worst-case control by
averaged directional estimates, which is why the warm-start section is separated from the
deterministic local recursion.

\paragraph{Organization of the paper and proofs.}

The remainder of the paper is organized as follows.  Section~\ref{sec:order-d-extension}
contains the fixed-order local expansion, affine recursion, contraction theorem, local fixed point,
and KKT connection, together with a high-signal reading of the deterministic conditions.
Section~\ref{sec:order-three-model} gives the order-three deterministic specialization as a fully
expanded example with parallel subsections.  Section~\ref{sec:order-d-worked-example} gives a compressed worked interpretation of the subset expansion and the high-signal scale accounting behind the scale
\(\omega_{N,d}=\beta/N^{(d-2)/4}\).  Section~\ref{sec:warm-starts-final} then gives the generic
warm-start principle first and then verifies same-sample centered-Gram initialization under
i.i.d. finite-fourth-moment noise.  The final section concludes.

To keep the main line readable, all proofs of formal statements are deferred to appendices organized by their corresponding main-text sections.  Appendix~\ref{app:proofs-order-d} contains the fixed-order local theory proofs, Appendix~\ref{app:proofs-order-three} contains the order-three worked-example proofs, and Appendix~\ref{app:proofs-warm-starts} contains the warm-start and centered-Gram proofs.  The main text first states the deterministic finite-iteration local theory and its probabilistic event verification; generic warm starts and the i.i.d. finite-fourth-moment centered-Gram verification are collected afterwards in Section~\ref{sec:warm-starts-final}.
In the same-sample centered-Gram appendix, the non-elementary random-matrix input is isolated as a
rectangular covariance/Lindeberg package in Lemmas~\ref{lem:aux-rect-cov-moment-d}--\ref{lem:aux-residual-cov-d}.  The dependence created by using the same
sample is handled internally by leave-one reduced-resolvent expansions and pressed-back diagram
counts, rather than by treating the rows as an abstract independent-row input.

\section{Fixed-order \texorpdfstring{$d$}{d}-mode theory}\label{sec:order-d-extension}

We now present a fixed-order \(d\)-mode theory, which
is the main technical block of our work.  It treats a general fixed tensor order
\(d\ge3\) from the beginning, rather than deriving the result by analogy with the
order-three case.  The proof sequence mirrors the usual low-order presentation: 
we first introduce a subset-indexed noise-control
event, then prove the deterministic local expansion and contraction, solve the resulting affine
recursion for every finite iteration time, and record the local fixed-point and KKT
consequences.  Initialization is collected later in Section~\ref{sec:warm-starts-final}.

The local remainder is a finite sum over all subsets containing two or more perturbation
directions; the order-three bilinear remainder appears later only as a worked
specialization.  We keep all finite-sample constants and basin parameters explicit because
the same formulas are used in the centered-Gram warm-start analysis.
Throughout this section, $d\ge3$ is fixed.  Constants denoted by $C_d,c_d,\rho_d$ may
change from line to line and depend only on $d$ and on fixed aspect-ratio bounds.

\subsection{Local coordinates}

We write
\begin{equation}\label{eq:local-coord-d}
\bm{u}^{(i)}
=
\alpha_i \bm{x}^{(i)}+\bm{a}_i,
\qquad
\bm{a}_i\in(\bm{x}^{(i)})^\perp,
\qquad
\alpha_i=\sqrt{1-\|\bm{a}_i\|^2}.
\end{equation}
The local sign is chosen so that $\alpha_i\ge0$.  Define
\begin{equation}\label{eq:delta-def-d}
\delta(U):=\max_{1\le i\le d}\|\bm{a}_i\|,
\end{equation}
and
\begin{equation}\label{eq:basin-d}
\calB_r^{(d)}
:=
\left\{
U\in\prod_{i=1}^d\Sph^{n_i-1}:
\delta(U)\le\frac r\beta
\right\}.
\end{equation}
For $U,\widetilde U\in\calB_r^{(d)}$, let
\begin{equation}\label{eq:local-distance-d}
d_{\loc}(U,\widetilde U)
:=
\max_{1\le i\le d}\|\bm{a}_i-\widetilde{\bm{a}}_i\|.
\end{equation}

\begin{lemma}[Elementary coordinate bounds]\label{lem:alpha-product-d}
Set
\[
\rho_d:=\min\left\{\frac14,\frac{1}{2\sqrt d}\right\}.
\]
Whenever
\[
\delta(U),\delta(\widetilde U)\le \rho_d,
\]
the following estimates hold:
\begin{equation}\label{eq:alpha-lower-d}
\prod_{j\ne i}\alpha_j\ge\frac34,
\qquad i=1,\ldots,d,
\end{equation}
and for every subset \(J\subseteq[d]\),
\begin{equation}\label{eq:alpha-product-lip-d}
\left|
\prod_{j\in J}\alpha_j
-
\prod_{j\in J}\widetilde\alpha_j
\right|
\le
C_d\bigl(\delta(U)+\delta(\widetilde U)\bigr)d_{\loc}(U,\widetilde U).
\end{equation}
In particular, if
\[
\delta(U),\delta(\widetilde U)\le r/\beta
\quad\text{and}\quad
r/\beta\le\rho_d,
\]
then
\begin{equation}\label{eq:alpha-product-lip-basin-d}
\left|
\prod_{j\in J}\alpha_j
-
\prod_{j\in J}\widetilde\alpha_j
\right|
\le
C_d\frac r\beta d_{\loc}(U,\widetilde U).
\end{equation}
\end{lemma}

\noindent\emph{Proof deferred to Appendix~\ref{app:proof-01}.}

\subsection{Subset-indexed noise contractions}

Let
\[
\bm{\Pi}_i:=\bm{I}_{n_i}-\bm{x}^{(i)}(\bm{x}^{(i)})^\top,
\qquad
I_i:=[d]\setminus\{i\}.
\]
For each $i$ and $S\subseteq I_i$, define
\begin{equation}\label{eq:subset-vector-contraction-d}
\mathfrak W_{i,S}(\bm{h}_S)
:=
\frac1{\sqrt N}
\bm{W}(\bm{v}^{(1)},\ldots,\bm{v}^{(i-1)},\cdot,\bm{v}^{(i+1)},\ldots,\bm{v}^{(d)})
\in\R^{n_i},
\end{equation}
where
\[
\bm{v}^{(j)}
=
\begin{cases}
\bm{h}_j, & j\in S,\\
\bm{x}^{(j)}, & j\in I_i\setminus S.
\end{cases}
\]
For $S=\varnothing$ this is a fixed vector.  Split
\begin{equation}\label{eq:subset-par-perp-d}
\eta_{i,S}(\bm{h}_S)
:=
\langle\mathfrak W_{i,S}(\bm{h}_S),\bm{x}^{(i)}\rangle,
\qquad
G_{i,S}(\bm{h}_S)
:=
\bm{\Pi}_i\mathfrak W_{i,S}(\bm{h}_S).
\end{equation}
For the singular value estimate, define full scalar contractions.  For
$S\subseteq[d]$, let
\begin{equation}\label{eq:subset-scalar-contraction-d}
\omega_S(\bm{h}_S)
:=
\frac1{\sqrt N}\bm{W}(\bm{v}^{(1)},\ldots,\bm{v}^{(d)}),
\qquad
\bm{v}^{(j)}
=
\begin{cases}
\bm{h}_j, & j\in S,\\
\bm{x}^{(j)}, & j\notin S.
\end{cases}
\end{equation}
The scalar contraction symbol \(\omega_S\) is local notation for a subset-indexed noise
contraction and should not be confused with the signal scale \(\omega_{N,d}\).
For a multilinear map $B$ indexed by $S$, write
\[
\|B\|_{\op}:=
\sup_{\|\bm{h}_j\|\le1,\ j\in S}\|B(\bm{h}_S)\|,
\]
with absolute value for scalar maps.

\begin{definition}[Raw order-$d$ multilinear control event]\label{def:raw-event-d}
For $L\ge1$ and $\Theta\ge1$, let $\calF_{d,N}(L,\Theta)$ be the event on which, for
every admissible $i$ and $S$,
\begin{align}
|\eta_{i,\varnothing}|+\|G_{i,\varnothing}\|+|\omega_\varnothing|
&\le L, \label{eq:raw-q0-d}\\
\max_{|S|=1}\bigl(\|\eta_{i,S}\|_{\op}+\|G_{i,S}\|_{\op}\bigr)
+
\max_{|S|=1}\|\omega_S\|_{\op}
&\le L, \label{eq:raw-q1-d}\\
\max_{2\le |S|\le d-1}\bigl(\|\eta_{i,S}\|_{\op}+\|G_{i,S}\|_{\op}\bigr)
+
\max_{2\le |S|\le d}\|\omega_S\|_{\op}
&\le \Theta. \label{eq:raw-qge2-d}
\end{align}
If the first maximum in \eqref{eq:raw-qge2-d} is over an empty family, it is interpreted
as zero.
\end{definition}

\begin{proposition}[A crude probabilistic verification for fixed order]\label{prop:raw-event-hp-d}
Assume that $d\ge3$ is fixed, that $n_i\asymp N$ for every $i$, and that the entries of
$\bm{W}$ are independent, centered, have unit variance, and have uniformly bounded fourth
moments.  Then there are constants $L,C<\infty$, depending only on $d$, the aspect-ratio
bounds, and the fourth-moment bound, such that with
\[
\Theta=C N^{(d-2)/2}
\]
one has
\[
\Pbb\{\calF_{d,N}(L,\Theta)\}\longrightarrow1.
\]
\end{proposition}

\noindent\emph{Proof deferred to Appendix~\ref{app:proof-02}.}

\paragraph{Fixed-order scale from the crude event.}
The role of Proposition~\ref{prop:raw-event-hp-d} is to translate the deterministic recursion into a
simple high-signal condition.  This verification is intentionally conservative: the deterministic
local theory below only uses the consequence \(\Theta r/\beta^2=o(1)\), and sharper tensor-norm
estimates can be substituted without changing the local recursion.  Under the crude verification
\begin{equation}
\Theta=O\!\left(N^{(d-2)/2}\right),
\end{equation}
the higher-order contribution to the local Lipschitz coefficient becomes
\begin{equation}
\frac{\Theta r}{\beta^2}
=
O\!\left(\frac{rN^{(d-2)/2}}{\beta^2}\right)
=
O\!\left(\frac{r}{\omega_{N,d}^2}\right),
\qquad
\omega_{N,d}:=\frac{\beta}{N^{(d-2)/4}}.
\end{equation}
Thus, for fixed local radius, and also for slowly expanding radii \(r_N=o(\omega_{N,d})\),
this term is \(o(1)\) whenever \(\omega_{N,d}\to\infty\).  This is the fixed-order analogue
of the order-three calculation in which the bilinear remainder has scale \(N^{1/2}\) and the
local scale is \(\beta/N^{1/4}\to\infty\).

\subsection{Exact expansion and deterministic bounds}

For $S\subseteq I_i$, define
\[
A_{i,S}(U):=\prod_{\ell\in I_i\setminus S}\alpha_\ell.
\]
For $S\subseteq[d]$, define
\[
A_S^{\rm full}(U):=\prod_{\ell\in[d]\setminus S}\alpha_\ell.
\]
Empty products equal one.

\begin{proposition}[Exact subset expansion]\label{prop:exact-subset-expansion-d}
For each $i\in[d]$,
\begin{equation}\label{eq:mode-decomp-exact-d}
\bm{T}(U^{(-i)})
=
s_i(U)\bm{x}^{(i)}+\bm{g}_i(U),
\qquad
\bm{g}_i(U)\in(\bm{x}^{(i)})^\perp,
\end{equation}
where
\begin{align}
s_i(U)
&=
\beta\prod_{j\ne i}\alpha_j+\zeta_i(U), \label{eq:si-def-d}\\
\zeta_i(U)
&=
\sum_{S\subseteq I_i}A_{i,S}(U)\eta_{i,S}(\bm{a}_S), \label{eq:zeta-exact-d}\\
\bm{g}_i(U)
&=
\sum_{S\subseteq I_i}A_{i,S}(U)G_{i,S}(\bm{a}_S). \label{eq:g-exact-d}
\end{align}
Moreover,
\begin{equation}\label{eq:omega-exact-d}
\Omega(U)
:=
\frac1{\sqrt N}\bm{W}(\bm{u}^{(1)},\ldots,\bm{u}^{(d)})
=
\sum_{S\subseteq[d]}A_S^{\rm full}(U)\omega_S(\bm{a}_S).
\end{equation}
\end{proposition}

\noindent\emph{Proof deferred to Appendix~\ref{app:proof-03}.}

\begin{lemma}[One-point and two-point subset bounds]\label{lem:subset-bounds-d}
Assume that $\calF_{d,N}(L,\Theta)$ holds and that
$\delta(U),\delta(\widetilde U)\le\rho_d$.  Then, for every mode $i$,
\begin{align}
|\zeta_i(U)|+\|\bm{g}_i(U)\|
&\le
C_dL(1+\delta(U))+C_d\Theta\delta(U)^2,
\label{eq:zeta-g-one-bound-d}\\
|\zeta_i(U)-\zeta_i(\widetilde U)|
+\|\bm{g}_i(U)-\bm{g}_i(\widetilde U)\|
&\le
\bigl(C_dL+C_d\Theta(\delta(U)+\delta(\widetilde U))\bigr)
d_{\loc}(U,\widetilde U).
\label{eq:zeta-g-two-bound-d}
\end{align}
Moreover,
\begin{align}
|\Omega(U)|
&\le
C_dL(1+\delta(U))+C_d\Theta\delta(U)^2,
\label{eq:omega-one-bound-d}\\
|\Omega(U)-\Omega(\widetilde U)|
&\le
\bigl(C_dL+C_d\Theta(\delta(U)+\delta(\widetilde U))\bigr)
d_{\loc}(U,\widetilde U).
\label{eq:omega-two-bound-d}
\end{align}
\end{lemma}

\noindent\emph{Proof deferred to Appendix~\ref{app:proof-04}.}

\paragraph{How to read the fixed-order expansion.}
The subset notation in Proposition~\ref{prop:exact-subset-expansion-d} is meant to keep the
order-\(d\) proof genuinely parallel to the familiar order-three calculation.  For the update of
mode \(i\), the set \(I_i=[d]\setminus\{i\}\) contains all modes that are contracted.  A subset
\(S\subseteq I_i\) records exactly which of those contracted modes contribute a local
perturbation direction rather than the planted vector.  Thus:
\begin{itemize}
\item \(S=\varnothing\) gives the fixed noise contraction at the planted point.  After division by
      the signal size \(\beta\), this term is responsible for the unavoidable \(O(1/\beta)\) floor.
\item \(|S|=1\) gives the linear noise terms.  These terms enter the Lipschitz coefficient with
      size \(L/\beta\).
\item \(|S|\ge2\) gives the genuinely multilinear remainders.  In the deterministic event
      \(\mathcal F_{d,N}(L,\Theta)\), all such higher-order remainders are summarized by the
      single parameter \(\Theta\).
\end{itemize}
This is why the local recursion has the affine form
\begin{equation}
\delta(U_{t+1})
\le
\frac{C_dL}{\beta}
+
\left(\frac{C_dL}{\beta}+\frac{C_d\Theta r}{\beta^2}\right)\delta(U_t).
\end{equation}
The first term is the fixed-contraction floor, the first term in the parentheses is the linear
noise sensitivity, and the second term in the parentheses is the higher-order subset remainder
after using the basin inequality \(\delta(U_t)^2\le (r/\beta)\delta(U_t)\).  This interpretation is
independent of the order \(d\); only the probabilistic size of \(\Theta\) changes with \(d\).

\subsection{Expanded fixed-order subset algebra}\label{subsec:expanded-d-subset-algebra}

For $d=4$ and mode $i=1$, the set of other modes is $I_1=\{2,3,4\}$.  The first-mode
noise contraction expands as
\begin{align*}
\noise(U^{(-1)})
&=
\alpha_2\alpha_3\alpha_4\noise(\bm{x}^{(2)},\bm{x}^{(3)},\bm{x}^{(4)})\\
&\quad+
\alpha_3\alpha_4\noise(\bm{a}_2,\bm{x}^{(3)},\bm{x}^{(4)})
+
\alpha_2\alpha_4\noise(\bm{x}^{(2)},\bm{a}_3,\bm{x}^{(4)})
+
\alpha_2\alpha_3\noise(\bm{x}^{(2)},\bm{x}^{(3)},\bm{a}_4)\\
&\quad+
\alpha_4\noise(\bm{a}_2,\bm{a}_3,\bm{x}^{(4)})
+
\alpha_3\noise(\bm{a}_2,\bm{x}^{(3)},\bm{a}_4)
+
\alpha_2\noise(\bm{x}^{(2)},\bm{a}_3,\bm{a}_4)\\
&\quad+
\noise(\bm{a}_2,\bm{a}_3,\bm{a}_4).
\end{align*}
Here the first line is the fixed contraction, the second line consists of one-perturbation
contractions, and the last two lines consist of terms with at least two perturbation
directions.  The general order-$d$ formula is the same expansion with one term for each
subset $S\subseteq I_i$.

For a $q$-linear map $B$ and $S=\{j_1,\ldots,j_q\}$, the exact telescoping identity is
\begin{align*}
&B(\bm{a}_{j_1},\ldots,\bm{a}_{j_q})
-
B(\widetilde{\bm{a}}_{j_1},\ldots,\widetilde{\bm{a}}_{j_q})
\\
&\quad=
\sum_{\ell=1}^q
B(\widetilde{\bm{a}}_{j_1},\ldots,\widetilde{\bm{a}}_{j_{\ell-1}},
\bm{a}_{j_\ell}-\widetilde{\bm{a}}_{j_\ell},
\bm{a}_{j_{\ell+1}},\ldots,\bm{a}_{j_q}).
\end{align*}
If $\|B\|_{\op}\le\Theta$ and
$\delta(U),\delta(\widetilde U)\le\rho_d\le1$, then
\[
\|B(\bm{a}_S)-B(\widetilde{\bm{a}}_S)\|
\le
C_d\Theta(\delta(U)+\delta(\widetilde U))^{q-1}
d_{\loc}(U,\widetilde U).
\]
For $q\ge2$, after decreasing $\rho_d$ if necessary, this is bounded by
\[
C_d\Theta(\delta(U)+\delta(\widetilde U))d_{\loc}(U,\widetilde U).
\]
This is the estimate used in Lemma~\ref{lem:subset-bounds-d}.

\begin{lemma}[Order-\texorpdfstring{$d$}{d} normalization estimate]\label{lem:norm-pert-d}
Let \(\bm x\) be a unit vector, let \(\bm g,\widetilde{\bm g}\in\bm x^\perp\), and let
\(s,\widetilde s>0\).  Define
\[
\Phi_{\bm x}(s,\bm g):=\frac{s\bm x+\bm g}{\|s\bm x+\bm g\|}.
\]
If \(\|\bm g\|\le s/2\), then
\[
\left\|(\bm I-\bm x\bm x^\top)\Phi_{\bm x}(s,\bm g)\right\|
\le
\frac{2\|\bm g\|}{s}.
\]
Moreover, if
\[
\|\bm g\|\vee\|\widetilde{\bm g}\|\le \frac12(s\wedge\widetilde s),
\]
then
\[
\|\Phi_{\bm x}(s,\bm g)-\Phi_{\bm x}(\widetilde s,\widetilde{\bm g})\|
\le
C\frac{\|\bm g-\widetilde{\bm g}\|}{s\wedge\widetilde s}
+
C\frac{\|\bm g\|+\|\widetilde{\bm g}\|}{(s\wedge\widetilde s)^2}|s-\widetilde s|.
\]
\end{lemma}

\noindent\emph{Proof deferred to Appendix~\ref{app:proof-norm-pert-d}.}

\subsection{Order-\texorpdfstring{$d$}{d} recursion and natural decay}

Define
\begin{align}
G_{r,d}(\beta)
&:=
C_dL\left(1+\frac r\beta\right)
+
C_d\Theta\frac{r^2}{\beta^2},
\label{eq:G-r-d}\\
\kappa_{r,d}^{\aff}(\beta)
&:=
\frac{C_dL}{\beta}
+
\frac{C_d\Theta r}{\beta^2}.
\label{eq:kappa-aff-d}
\end{align}

\begin{theorem}[Order-$d$ local one-step recursion]\label{thm:local-recursion-d}
Assume that $\calF_{d,N}(L,\Theta)$ holds.  Fix \(r>0\) and \(\beta>0\) such that
\[
\frac r\beta\le \rho_d.
\]
Assume further that
\begin{equation}\label{eq:signal-dominance-d}
G_{r,d}(\beta)\le\frac\beta4.
\end{equation}
Then $\calA$ is well defined on $\calB_r^{(d)}$, and for every
$U_t\in\calB_r^{(d)}$,
\begin{equation}\label{eq:main-recursion-d}
\delta(U_{t+1})
\le
\frac{C_dL}{\beta}
+
\kappa_{r,d}^{\aff}(\beta)\delta(U_t).
\end{equation}
\end{theorem}

\noindent\emph{Proof deferred to Appendix~\ref{app:proof-05}.}

\begin{proposition}[Finite-iteration affine error bound]\label{prop:explicit-affine-d}
Let
\[
A_{\beta,d}:=\frac{C_dL}{\beta},
\qquad
q_{\beta,d}:=\kappa_{r,d}^{\aff}(\beta).
\]
Suppose that $U_s,U_{s+1},\ldots,U_t$ remain in $\calB_r^{(d)}$ and that
$0\le q_{\beta,d}<1$.  Then, for every finite pair of local times $s\le t$,
\begin{equation}\label{eq:delta-closed-form-d}
\delta(U_t)
\le
q_{\beta,d}^{t-s}\delta(U_s)
+
\frac{A_{\beta,d}(1-q_{\beta,d}^{t-s})}{1-q_{\beta,d}}.
\end{equation}
Equivalently,
\begin{equation}\label{eq:delta-floor-split-d}
\delta(U_t)-\frac{A_{\beta,d}}{1-q_{\beta,d}}
\le
q_{\beta,d}^{t-s}
\left(\delta(U_s)-\frac{A_{\beta,d}}{1-q_{\beta,d}}\right).
\end{equation}
In particular,
\begin{equation}\label{eq:delta-floor-d}
\limsup_{t\to\infty}\delta(U_t)
\le
\frac{A_{\beta,d}}{1-q_{\beta,d}}.
\end{equation}
Moreover, the number of local iterations needed to reduce the transient term below
$\varepsilon/\beta$ is bounded explicitly by
\begin{equation}\label{eq:hitting-time-floor-d}
t-s
\ge
\frac{\log(\beta\delta(U_s)/\varepsilon)}{\log(1/q_{\beta,d})}
\quad\Longrightarrow\quad
q_{\beta,d}^{t-s}\delta(U_s)\le \frac{\varepsilon}{\beta}.
\end{equation}
\end{proposition}

\noindent\emph{Proof deferred to Appendix~\ref{app:proof-06}.}

Define
\begin{align}
D_{r,d}(\beta)
&:=
C_dL+C_d\Theta\frac r\beta,
\label{eq:D-r-d}\\
S_{r,d}(\beta)
&:=
C_dr+C_dL+C_d\Theta\frac r\beta,
\label{eq:S-r-d}\\
\kappa_{r,d}^{\ctr}(\beta)
&:=
C_d\frac{D_{r,d}(\beta)}{\beta}
+
C_d\frac{G_{r,d}(\beta)S_{r,d}(\beta)}{\beta^2}.
\label{eq:kappa-ctr-d}
\end{align}

\begin{proposition}[Order-$d$ two-point contraction]\label{prop:two-point-d}
Assume that $\calF_{d,N}(L,\Theta)$ holds, $r/\beta\le\rho_d$, and
\eqref{eq:signal-dominance-d} holds.  Then for all
$U,\widetilde U\in\calB_r^{(d)}$,
\begin{equation}\label{eq:two-point-d}
d_{\loc}(\calA(U),\calA(\widetilde U))
\le
\kappa_{r,d}^{\ctr}(\beta)d_{\loc}(U,\widetilde U).
\end{equation}
Under the additional mild scale conditions $\beta\ge C_d(1+L)$ and $r/\beta\le c_d$,
\begin{equation}\label{eq:kappa-simplified-d}
\kappa_{r,d}^{\ctr}(\beta)
\le
\frac{C_dL}{\beta}
+
\frac{C_d\Theta r}{\beta^2}.
\end{equation}
\end{proposition}

\noindent\emph{Proof deferred to Appendix~\ref{app:proof-07}.}

\begin{corollary}[Order-$d$ fixed point and convergence]\label{cor:fixed-point-d}
Assume $\calF_{d,N}(L,\Theta)$, $r/\beta\le\rho_d$, and
\eqref{eq:signal-dominance-d}.  Suppose
\[
\kappa_{r,d}^{\aff}(\beta)<1,
\qquad
\kappa_{r,d}^{\ctr}(\beta)<1,
\]
and
\begin{equation}\label{eq:self-map-d}
\frac{C_dL}{\beta}
+
\kappa_{r,d}^{\aff}(\beta)\frac r\beta
\le
\frac r\beta.
\end{equation}
Then $\calA$ maps $\calB_r^{(d)}$ into itself and has a unique fixed point
\[
U_\star=(\bm{u}_\star^{(1)},\ldots,\bm{u}_\star^{(d)})\in\calB_r^{(d)}.
\]
Moreover,
\begin{equation}\label{eq:fixed-error-d}
\delta(U_\star)
\le
\frac{C_dL}{\beta(1-\kappa_{r,d}^{\aff}(\beta))},
\end{equation}
and for every $U_s\in\calB_r^{(d)}$ and every $t\ge s$,
\begin{equation}\label{eq:fixed-distance-d}
d_{\loc}(U_t,U_\star)
\le
\bigl(\kappa_{r,d}^{\ctr}(\beta)\bigr)^{t-s}d_{\loc}(U_s,U_\star).
\end{equation}
Equivalently, for every tolerance $\varepsilon>0$,
\begin{equation}\label{eq:fixed-distance-hitting-d}
t-s
\ge
\frac{\log(d_{\loc}(U_s,U_\star)/\varepsilon)}
{\log(1/\kappa_{r,d}^{\ctr}(\beta))}
\quad\Longrightarrow\quad
d_{\loc}(U_t,U_\star)\le \varepsilon.
\end{equation}
\end{corollary}

\noindent\emph{Proof deferred to Appendix~\ref{app:proof-08}.}

\paragraph{High-signal consequence of the local theory.}
The deterministic assumptions in Theorem~\ref{thm:local-recursion-d},
Proposition~\ref{prop:two-point-d}, and Corollary~\ref{cor:fixed-point-d} have a simple
asymptotic interpretation.  Under Proposition~\ref{prop:raw-event-hp-d},
\(L=O(1)\) and \(\Theta=O(N^{(d-2)/2})\) with high probability.  Hence, if
\[
\omega_{N,d}:=\frac{\beta}{N^{(d-2)/4}}\to\infty
\]
and \(r_N\) is fixed or grows slowly enough that \(r_N=o(\omega_{N,d})\), then
\[
\kappa_{r_N,d}^{\aff}(\beta)
=
O\!\left(\frac1\beta+\frac{r_N}{\omega_{N,d}^2}\right)
=o(1),
\qquad
\kappa_{r_N,d}^{\ctr}(\beta)=o(1),
\]
and the signal-dominance and self-map inequalities hold for all sufficiently large \(N\)
after choosing the local radius large enough.  Thus the many finite-sample inequalities above
are not separate structural assumptions; they are scale conditions ensuring that, once
\(U_s\in\calB_{r_N}^{(d)}\),
\[
\delta(U_t)
\le
q_{N,d}^{\,t-s}\delta(U_s)+O(\beta^{-1}),
\qquad
q_{N,d}=O\!\left(\frac1\beta+\frac{r_N}{\omega_{N,d}^2}\right)=o(1).
\]
This is the form used later to combine any verified warm start with the local
finite-iteration dynamics; the centered-Gram construction is only one way to verify the warm-start inputs.

\begin{corollary}[Order-$d$ KKT equations and singular value]\label{cor:kkt-d}
Under the assumptions of Corollary~\ref{cor:fixed-point-d}, the local fixed point satisfies
\begin{equation}\label{eq:kkt-d}
\bm{T}(U_\star^{(-i)})=\lambda_\star \bm{u}_\star^{(i)},
\qquad i=1,\ldots,d,
\end{equation}
where
\[
\lambda_\star:=\bm{T}(\bm{u}_\star^{(1)},\ldots,\bm{u}_\star^{(d)}).
\]
Furthermore,
\begin{equation}\label{eq:lambda-bound-raw-d}
|\lambda_\star-\beta|
\le
C_d\beta\delta(U_\star)^2
+
C_dL(1+\delta(U_\star))
+
C_d\Theta\delta(U_\star)^2.
\end{equation}
Consequently,
\begin{equation}\label{eq:lambda-bound-d}
|\lambda_\star-\beta|
\le
C_dL
+
\frac{C_dL^2}{\beta(1-\kappa_{r,d}^{\aff}(\beta))^2}
+
\frac{C_d\Theta L^2}{\beta^2(1-\kappa_{r,d}^{\aff}(\beta))^2}.
\end{equation}
\end{corollary}

\begin{remark}[First-order nature of the KKT statement]\label{rem:kkt-d-first-order}
Corollary~\ref{cor:kkt-d} is a local first-order statement.  It says that the informative local
fixed point satisfies the KKT equations, but it does not by itself assert strict local maximality
or uniqueness among all stationary points of the tensor objective.
\end{remark}

\noindent\emph{Proof deferred to Appendix~\ref{app:proof-09}.}

\section{Order-three specialization and worked example}\label{sec:order-three-model}

This section records the order-three specialization of the fixed-order theory.  It is expository
and not logically needed for the fixed-order theorem, but it is included as a complete worked
example because the subset-indexed formulas reduce to familiar fixed, linear, and bilinear noise
contractions.  The organization deliberately parallels Section~\ref{sec:order-d-extension}: model
and notation, local coordinates and noise control, deterministic local dynamics, and scale
accounting.  Initialization is deferred to Section~\ref{sec:warm-starts-final}.
Readers who are only interested in the fixed-order theorem may skip this section and proceed
directly to Section~\ref{sec:order-d-worked-example}.

Let
\begin{equation}\label{eq:model}
\bm{T}=\beta\,\bm{x}\otimes \bm{y}\otimes \bm{z}+\frac1{\sqrt N}\bm{W}\in \R^{m\times n\times p},
\qquad
N:=m+n+p,
\end{equation}
where $\bm{x}\in\Sph^{m-1}$, $\bm{y}\in\Sph^{n-1}$, and $\bm{z}\in\Sph^{p-1}$ are deterministic unit
vectors.  The entries of $\bm{W}$ are independent, centered, and have unit variance.  When the
high-probability event in Subsection~\ref{sec:noise-event} is invoked, we assume in addition
the rectangular spectral-norm input stated in Lemma~\ref{lem:rectangular-op-input}; this holds,
for instance, under independent entries with uniformly bounded fourth moments.  The
centered-Gram initialization argument below uses a standard wide sample-covariance estimate;
in the same-sample theorem we state this part conservatively under i.i.d. finite-fourth-moment
entries.

For $\bm{u}\in\Sph^{m-1}$, $\bm{v}\in\Sph^{n-1}$, and $\bm{w}\in\Sph^{p-1}$, define the simultaneous,
or Jacobi-type, alternating power map
\begin{equation}\label{eq:alt-map}
\calA(\bm{u},\bm{v},\bm{w})
:=
\left(
\frac{\bm{T}(\bm{v},\bm{w})}{\|\bm{T}(\bm{v},\bm{w})\|},
\frac{\bm{T}(\bm{u},\bm{w})}{\|\bm{T}(\bm{u},\bm{w})\|},
\frac{\bm{T}(\bm{v})^\top \bm{u}}{\|\bm{T}(\bm{v})^\top \bm{u}\|}
\right),
\end{equation}
where, for example, $\bm{T}(\bm{v},\bm{w})\in\R^m$ denotes contraction of $\bm{T}$ along the second and third
modes.  Also, \(\bm T(\bm v)\in\mathbb R^{m\times p}\) denotes contraction along the second mode,
so that \(\bm T(\bm v)^\top\bm u\in\mathbb R^p\).  Starting from an initialization $(\bm{u}_0,\bm{v}_0,\bm{w}_0)$, define
\begin{equation}\label{eq:alt-iter}
(\bm{u}_{t+1},\bm{v}_{t+1},\bm{w}_{t+1})=\calA(\bm{u}_t,\bm{v}_t,\bm{w}_t),
\qquad t\ge0.
\end{equation}
The associated scalar iterate is
\[
\lambda_t:=\bm{T}(\bm{u}_t,\bm{v}_t,\bm{w}_t).
\]

The asymptotic regime emphasized in this paper is not the constant-level critical scaling.
For order three we work above the local scale in the sense that
\begin{equation}\label{eq:high-signal-3}
\omega_N:=\frac{\beta}{N^{1/4}}\longrightarrow\infty .
\end{equation}
For fixed order \(d\ge3\), the analogous conservative high-signal regime is
\begin{equation}\label{eq:high-signal-d-intro}
\omega_{N,d}:=\frac{\beta}{N^{(d-2)/4}}\longrightarrow\infty .
\end{equation}
The deterministic statements below are formulated for finite \(N\) in terms of explicit
parameters such as \(r\), \(L_0\), and \(\Theta\).  Conditions such as
\(\beta\ge C N^{1/4}\) or \(\beta\ge C_d N^{(d-2)/4}\) should therefore be read as
finite-sample sufficient inequalities, while \eqref{eq:high-signal-3} and
\eqref{eq:high-signal-d-intro} are the high-signal asymptotic assumptions under which the
local contraction coefficients become \(o(1)\).

\begin{remark}[Relation to stationary-point analyses]
Stationary-point analyses for spiked tensor likelihoods study KKT systems and informative
branches of critical points.  The object here is different: we analyze the alternating power map
itself.  The local noise-control event and the remainder estimates are therefore formulated as
deterministic lemmas for a finite number of algorithmic steps.  Centered-Gram initialization is
analyzed separately in Section~\ref{sec:warm-starts-final}; the deterministic local theory itself
is independent of the particular initializer.
\end{remark}

The model and algorithm are closely related to tensor power methods and spiked Tensor PCA
\citep{AnandkumarGeHsuKakadeTelgarsky2014,MontanariRichard2014,HopkinsShiSteurer2015},
and to spectral-norm estimates for random tensors \citep{TomiokaSuzuki2014}.
The main contribution here has two layers.  First, all iterate-wise local bounds below are
deterministic finite-sample statements for finite iteration times once the stated noise-control
event holds.
Second, initialization is handled later through a generic warm-start principle.  The order-three
section is a fully expanded illustration of the fixed-order argument rather than a logical
prerequisite for the general theorem.  For order three we verify the local event in full detail.
For general fixed order \(d\), we give both the
deterministic mechanism and a crude high-probability verification; sharper tensor-norm estimates
could be substituted without changing the local dynamics.

\subsection{Natural decay scale}

The proof distinguishes three related errors.  First, the local coordinate error
\[
\delta(U):=\max\{\|\bm{a}\|,\|\bm{b}\|,\|\bm{c}\|\}
\]
measures the distance from the planted spike in the positive local chart.  This is the most
convenient quantity for one-step analysis, because every mode update has the form
\[
\frac{s \bm{x}+\bm{g}}{\|s \bm{x}+\bm{g}\|},
\]
where $s$ is a scalar part of order $\beta$ and $\bm{g}$ is the orthogonal noise-contaminated
part.  Second, the two-point local distance
\[
d(U,\widetilde U)
:=
\max\{\|\bm{a}-\widetilde{\bm{a}}\|,\|\bm{b}-\widetilde{\bm{b}}\|,\|\bm{c}-\widetilde{\bm{c}}\|\}
\]
is used to prove uniqueness and convergence to the local fixed point.  Third, the
statistical error relative to the planted directions can be written as $\sin\angle(\bm{u},\bm{x})$,
$\|\bm{u}-\bm{x}\|$, or $1-\langle \bm{u},\bm{x}\rangle^2$; these quantities are equivalent in the positive
local chart.

The recurrence for $\delta_t$ is affine rather than homogeneous.  Even at the exact
planted point, the fixed orthogonal noise contractions $\bm{\xi}_u,\bm{\xi}_v,\bm{\xi}_w$ are generally
nonzero.  After normalization by a signal of size $\beta$, these fixed contractions create
an error of order $1/\beta$.  Therefore the natural local behavior is
\[
\text{geometric transient}
\quad + \quad
\text{nonzero noise floor of order }1/\beta,
\]
not decay all the way to zero.  In particular, the factor $q_\beta^t$ is the intrinsic
discrete-time decay supplied by the proof.  It can be rewritten as
$\exp(-t\log(1/q_\beta))$, but no artificial polynomial envelope is needed.

\subsection{Local coordinates and the noise-control event}\label{sec:noise-event}

For vectors near the planted spike, write
\begin{equation}\label{eq:local-coord}
\bm{u}=\alpha_x \bm{x}+\bm{a},\qquad
\bm{v}=\alpha_y \bm{y}+\bm{b},\qquad
\bm{w}=\alpha_z \bm{z}+\bm{c},
\end{equation}
where
\[
\bm{a}\in \bm{x}^\perp,\qquad \bm{b}\in \bm{y}^\perp,\qquad \bm{c}\in \bm{z}^\perp,
\]
and
\[
\alpha_x=\sqrt{1-\|\bm{a}\|^2},
\qquad
\alpha_y=\sqrt{1-\|\bm{b}\|^2},
\qquad
\alpha_z=\sqrt{1-\|\bm{c}\|^2}.
\]
For $r>0$, define
\begin{equation}\label{eq:Br}
B_r^{\mathrm{coord}}
:=
\Bigl\{(\bm{a},\bm{b},\bm{c})\in \bm{x}^\perp\times \bm{y}^\perp\times \bm{z}^\perp:
\max\{\|\bm{a}\|,\|\bm{b}\|,\|\bm{c}\|\}\le r/\beta
\Bigr\}.
\end{equation}
Let $\calB_r$ denote the corresponding spherical basin, namely the set of triples
$(\bm{u},\bm{v},\bm{w})$ on the product of spheres whose local coordinates belong to
\(B_r^{\mathrm{coord}}\), with the positive local signs chosen as in \eqref{eq:local-coord}.

\begin{lemma}[Coordinate, angle, and Euclidean equivalences]\label{lem:coord-equivalence-3}
Let $\bm{u}=\alpha_x\bm{x}+\bm{a}$ with $\bm{a}\in \bm{x}^\perp$ and $\alpha_x=\sqrt{1-\|\bm{a}\|^2}\ge0$.  Then
\begin{align}
\sin\angle(\bm{u},\bm{x})&=\|\bm{a}\|, \label{eq:sin-exact}\\
1-\langle \bm{u},\bm{x}\rangle^2&=\|\bm{a}\|^2, \label{eq:one-minus-corr}\\
\|\bm{u}-\bm{x}\|^2&=(1-\alpha_x)^2+\|\bm{a}\|^2. \label{eq:euc-coord-exact}
\end{align}
If $\|\bm{a}\|\le1/2$, then
\begin{equation}\label{eq:euc-coordinate-equiv}
\|\bm{a}\|\le \|\bm{u}-\bm{x}\|\le 2\|\bm{a}\|,
\qquad
1-\langle \bm{u},\bm{x}\rangle\le \|\bm{a}\|^2.
\end{equation}
Consequently, inside any local basin with $r/\beta\le1/2$,
\begin{equation}\label{eq:triple-error-equiv}
\max\{\sin\angle(\bm{u},\bm{x}),\sin\angle(\bm{v},\bm{y}),\sin\angle(\bm{w},\bm{z})\}=\delta(U),
\end{equation}
and
\begin{equation}\label{eq:triple-euc-equiv}
\delta(U)
\le
\max\{\|\bm{u}-\bm{x}\|,\|\bm{v}-\bm{y}\|,\|\bm{w}-\bm{z}\|\}
\le 2\delta(U).
\end{equation}
\end{lemma}

\noindent\emph{Proof deferred to Appendix~\ref{app:proof-14}.}

Set
\[
\noise:=\frac1{\sqrt N}\bm{W}.
\]
Let $\Proj_x$, $\Proj_y$, and $\Proj_z$ denote the orthogonal projections onto
$\bm{x}^\perp$, $\bm{y}^\perp$, and $\bm{z}^\perp$, respectively.

For the $\bm{u}$-update, define
\begin{align*}
\eta_{u,0}&:=\langle \noise(\bm{y},\bm{z}),\bm{x}\rangle,
&\ell_u^{(y)}(\bm{c})&:=\langle \noise(\bm{y},\bm{c}),\bm{x}\rangle,
&\ell_u^{(z)}(\bm{b})&:=\langle \noise(\bm{b},\bm{z}),\bm{x}\rangle,\\
\rho_u(\bm{b},\bm{c})&:=\langle \noise(\bm{b},\bm{c}),\bm{x}\rangle,
&\bm{\xi}_u&:=\Proj_x\noise(\bm{y},\bm{z}),
&\bm{M}_u^{(y)}\bm{c}&:=\Proj_x\noise(\bm{y},\bm{c}),\\
\bm{M}_u^{(z)}\bm{b}&:=\Proj_x\noise(\bm{b},\bm{z}),
&\bm{R}_u(\bm{b},\bm{c})&:=\Proj_x\noise(\bm{b},\bm{c}).
\end{align*}
For the $\bm{v}$-update, define
\begin{align*}
\eta_{v,0}&:=\langle \noise(\bm{x},\bm{z}),\bm{y}\rangle,
&\ell_v^{(x)}(\bm{c})&:=\langle \noise(\bm{x},\bm{c}),\bm{y}\rangle,
&\ell_v^{(z)}(\bm{a})&:=\langle \noise(\bm{a},\bm{z}),\bm{y}\rangle,\\
\rho_v(\bm{a},\bm{c})&:=\langle \noise(\bm{a},\bm{c}),\bm{y}\rangle,
&\bm{\xi}_v&:=\Proj_y\noise(\bm{x},\bm{z}),
&\bm{M}_v^{(x)}\bm{c}&:=\Proj_y\noise(\bm{x},\bm{c}),\\
\bm{M}_v^{(z)}\bm{a}&:=\Proj_y\noise(\bm{a},\bm{z}),
&\bm{R}_v(\bm{a},\bm{c})&:=\Proj_y\noise(\bm{a},\bm{c}).
\end{align*}
For the $\bm{w}$-update, define
\begin{align*}
\eta_{w,0}&:=\langle \noise(\bm{x},\bm{y}),\bm{z}\rangle,
&\ell_w^{(x)}(\bm{b})&:=\langle \noise(\bm{x},\bm{b}),\bm{z}\rangle,
&\ell_w^{(y)}(\bm{a})&:=\langle \noise(\bm{a},\bm{y}),\bm{z}\rangle,\\
\rho_w(\bm{a},\bm{b})&:=\langle \noise(\bm{a},\bm{b}),\bm{z}\rangle,
&\bm{\xi}_w&:=\Proj_z\noise(\bm{x},\bm{y}),
&\bm{M}_w^{(x)}\bm{b}&:=\Proj_z\noise(\bm{x},\bm{b}),\\
\bm{M}_w^{(y)}\bm{a}&:=\Proj_z\noise(\bm{a},\bm{y}),
&\bm{R}_w(\bm{a},\bm{b})&:=\Proj_z\noise(\bm{a},\bm{b}).
\end{align*}

\begin{definition}[Local multilinear control event]\label{def:EL0}
For $L_0\ge1$, let $\calE_{L_0}$ be the event on which
\begin{align}
&|\eta_{u,0}|+|\eta_{v,0}|+|\eta_{w,0}|
+\|\bm{\xi}_u\|+\|\bm{\xi}_v\|+\|\bm{\xi}_w\|
\le L_0, \label{eq:EL0-fixed}
\end{align}
\begin{align}
&\|\ell_u^{(y)}\|+\|\ell_u^{(z)}\|+\|\bm{M}_u^{(y)}\|+\|\bm{M}_u^{(z)}\|
+\|\ell_v^{(x)}\|+\|\ell_v^{(z)}\|+\|\bm{M}_v^{(x)}\|+\|\bm{M}_v^{(z)}\|
\notag\\
&\qquad
+\|\ell_w^{(x)}\|+\|\ell_w^{(y)}\|+\|\bm{M}_w^{(x)}\|+\|\bm{M}_w^{(y)}\|
\le L_0, \label{eq:EL0-linear}
\end{align}
and, for all admissible arguments,
\begin{align}
|\rho_u(\bm{b},\bm{c})|+\|\bm{R}_u(\bm{b},\bm{c})\|
&\le L_0N^{1/2}\|\bm{b}\|\,\|\bm{c}\|, \label{eq:EL0-Ru-one}\\
|\rho_v(\bm{a},\bm{c})|+\|\bm{R}_v(\bm{a},\bm{c})\|
&\le L_0N^{1/2}\|\bm{a}\|\,\|\bm{c}\|, \label{eq:EL0-Rv-one}\\
|\rho_w(\bm{a},\bm{b})|+\|\bm{R}_w(\bm{a},\bm{b})\|
&\le L_0N^{1/2}\|\bm{a}\|\,\|\bm{b}\|. \label{eq:EL0-Rw-one}
\end{align}
\end{definition}

\begin{lemma}[Rectangular operator-norm input]\label{lem:rectangular-op-input}
Let $\bm{A}=\bm{A}_N\in\R^{M\times K}$ have independent centered entries with unit variances and
uniformly bounded fourth moments.  Suppose $M,K\to\infty$.  Then there is a constant
$C<\infty$, depending only on the fourth-moment bound, such that
\[
\Pbb\{\|\bm{A}\|_{\op}\le C(\sqrt M+\sqrt K)\}\longrightarrow1.
\]
The same conclusion remains valid when the entries of $\bm{A}$ are fixed unit-vector linear
combinations of independent entries of a larger tensor, provided the resulting entries are
independent and have uniformly bounded fourth moments.
\end{lemma}

\noindent\emph{Proof deferred to Appendix~\ref{app:proof-15}.}

\begin{lemma}[High-probability origin of $\calE_{L_0}$]\label{lem:EL0-hp}
Under the rectangular operator-norm input of Lemma~\ref{lem:rectangular-op-input} and fixed
aspect ratios, there exists a constant
$L_0>0$, depending only on the aspect ratios and the fourth-moment bound, such that
\[
\Pbb(\calE_{L_0})\longrightarrow1.
\]
\end{lemma}

\noindent\emph{Proof deferred to Appendix~\ref{app:proof-16}.}

\subsection{Local deterministic dynamics}\label{sec:local-dynamics}

\begin{lemma}[Normalized perturbation]\label{lem:norm-pert}
Let $\bm{x}\in\Sph^{m-1}$, let $\bm{g}\in \bm{x}^\perp$, and let $s>0$.  Define
\[
\bm{u}^+:=\frac{s\bm{x}+\bm{g}}{\|s\bm{x}+\bm{g}\|}.
\]
If $\|\bm{g}\|\le s/2$, then $\bm{u}^+=\alpha \bm{x}+\bm{a}$ with $\bm{a}\in \bm{x}^\perp$, and
\begin{equation}\label{eq:norm-pert-basic}
\|\bm{a}\|\le\frac{2\|\bm{g}\|}{s},
\qquad
1-\langle \bm{u}^+,\bm{x}\rangle^2\le\frac{4\|\bm{g}\|^2}{s^2}.
\end{equation}
Moreover, if
\[
\widetilde{\bm{u}}^+:=
\frac{\widetilde s \bm{x}+\widetilde{\bm{g}}}{\|\widetilde s \bm{x}+\widetilde{\bm{g}}\|},
\qquad
\widetilde{\bm{g}}\in \bm{x}^\perp,
\qquad
\|\bm{g}\|\vee\|\widetilde{\bm{g}}\|\le \frac12(s\wedge\widetilde s),
\]
then
\begin{equation}\label{eq:norm-pert-diff}
\|\bm{u}^+-\widetilde{\bm{u}}^+\|
\le
\frac{C}{s\wedge\widetilde s}\|\bm{g}-\widetilde{\bm{g}}\|
+
\frac{C(\|\bm{g}\|+\|\widetilde{\bm{g}}\|)}{(s\wedge\widetilde s)^2}|s-\widetilde s|.
\end{equation}
\end{lemma}

\noindent\emph{Proof deferred to Appendix~\ref{app:proof-17}.}

\begin{proposition}[Local multilinear expansion with scalar noise]\label{prop:local-expansion}
Assume that $\calE_{L_0}$ holds.  Let
\[
\bm{u}=\alpha_x\bm{x}+\bm{a},\qquad \bm{v}=\alpha_y\bm{y}+\bm{b},\qquad \bm{w}=\alpha_z\bm{z}+\bm{c}
\]
be as in \eqref{eq:local-coord}.  Then
\begin{equation}\label{eq:T-vw-sg}
\bm{T}(\bm{v},\bm{w})=s_u(\bm{v},\bm{w})\bm{x}+\bm{g}_u(\bm{v},\bm{w}),
\qquad
\bm{g}_u(\bm{v},\bm{w})\in \bm{x}^\perp,
\end{equation}
where
\begin{align}
s_u(\bm{v},\bm{w})
&=
\beta\alpha_y\alpha_z+\eta_u(\bm{b},\bm{c}), \label{eq:su-def}\\
\eta_u(\bm{b},\bm{c})
&=
\alpha_y\alpha_z\eta_{u,0}
+\alpha_y\ell_u^{(y)}(\bm{c})
+\alpha_z\ell_u^{(z)}(\bm{b})
+\rho_u(\bm{b},\bm{c}), \label{eq:eta-u-def}\\
\bm{g}_u(\bm{v},\bm{w})
&=
\alpha_y\alpha_z\bm{\xi}_u
+\alpha_y\bm{M}_u^{(y)}\bm{c}
+\alpha_z\bm{M}_u^{(z)}\bm{b}
+\bm{R}_u(\bm{b},\bm{c}). \label{eq:gu-def}
\end{align}
Analogous decompositions hold for the $\bm{v}$- and $\bm{w}$-updates.  If
\[
\delta:=\max\{\|\bm{a}\|,\|\bm{b}\|,\|\bm{c}\|\},
\]
then for each mode
\begin{equation}\label{eq:eta-g-onepoint}
|\eta_i|+\|\bm{g}_i\|
\le
CL_0\bigl(1+\delta+N^{1/2}\delta^2\bigr),
\qquad i\in\{\bm{u},\bm{v},\bm{w}\}.
\end{equation}
\end{proposition}

\noindent\emph{Proof deferred to Appendix~\ref{app:proof-18}.}

\begin{lemma}[Two-point bounds for scalar and vector quadratic remainders]\label{lem:two-point-remainders}
On the event $\calE_{L_0}$, for all admissible vectors,
\begin{align}
&|\rho_u(\bm{b},\bm{c})-\rho_u(\widetilde{\bm{b}},\widetilde{\bm{c}})|
+\|\bm{R}_u(\bm{b},\bm{c})-\bm{R}_u(\widetilde{\bm{b}},\widetilde{\bm{c}})\|
\notag\\
&\qquad\le
L_0N^{1/2}
\bigl(
\|\bm{b}-\widetilde{\bm{b}}\|\|\bm{c}\|
+\|\widetilde{\bm{b}}\|\|\bm{c}-\widetilde{\bm{c}}\|
\bigr), \label{eq:two-point-u}\\
&|\rho_v(\bm{a},\bm{c})-\rho_v(\widetilde{\bm{a}},\widetilde{\bm{c}})|
+\|\bm{R}_v(\bm{a},\bm{c})-\bm{R}_v(\widetilde{\bm{a}},\widetilde{\bm{c}})\|
\notag\\
&\qquad\le
L_0N^{1/2}
\bigl(
\|\bm{a}-\widetilde{\bm{a}}\|\|\bm{c}\|
+\|\widetilde{\bm{a}}\|\|\bm{c}-\widetilde{\bm{c}}\|
\bigr), \label{eq:two-point-v}\\
&|\rho_w(\bm{a},\bm{b})-\rho_w(\widetilde{\bm{a}},\widetilde{\bm{b}})|
+\|\bm{R}_w(\bm{a},\bm{b})-\bm{R}_w(\widetilde{\bm{a}},\widetilde{\bm{b}})\|
\notag\\
&\qquad\le
L_0N^{1/2}
\bigl(
\|\bm{a}-\widetilde{\bm{a}}\|\|\bm{b}\|
+\|\widetilde{\bm{a}}\|\|\bm{b}-\widetilde{\bm{b}}\|
\bigr). \label{eq:two-point-w}
\end{align}
\end{lemma}

\noindent\emph{Proof deferred to Appendix~\ref{app:proof-19}.}

Define
\begin{align}
\Gamma_r(\beta)
&:=
C_\Gamma L_0
\left(
1+\frac r\beta+\frac{N^{1/2}r^2}{\beta^2}
\right), \label{eq:Gamma-r}\\
\kappa_r^{\aff}(\beta)
&:=
\frac{C_1}{\beta}
+
\frac{C_2rN^{1/2}}{\beta^2}, \label{eq:kappa-aff-3}\\
D_r(\beta)
&:=
C_D L_0
\left(
1+\frac{rN^{1/2}}{\beta}
\right), \label{eq:D-r}\\
S_r(\beta)
&:=
C_S
\left(
r+L_0+\frac{L_0rN^{1/2}}{\beta}
\right), \label{eq:S-r}\\
\kappa_r^{\ctr}(\beta)
&:=
C\left(
\frac{D_r(\beta)}{\beta}
+
\frac{\Gamma_r(\beta)S_r(\beta)}{\beta^2}
\right). \label{eq:kappa-ctr-3}
\end{align}

\begin{theorem}[Local one-step recursion]\label{thm:local-recursion-3}
Assume that $\calE_{L_0}$ holds.  Let
\[
\bm{u}_t=\alpha_{x,t}\bm{x}+\bm{a}_t,\qquad
\bm{v}_t=\alpha_{y,t}\bm{y}+\bm{b}_t,\qquad
\bm{w}_t=\alpha_{z,t}\bm{z}+\bm{c}_t,
\]
and set
\[
\delta_t:=\max\{\|\bm{a}_t\|,\|\bm{b}_t\|,\|\bm{c}_t\|\}.
\]
Fix a radius parameter $r\ge1$.  There exist constants $C_0,C_1,C_2$ and
$\rho_3\in(0,1/2]$, depending only on $L_0$ and the aspect ratios, such that if
$(\bm{a}_t,\bm{b}_t,\bm{c}_t)\in B_r^{\mathrm{coord}}$,
\begin{equation}\label{eq:local-radius-3}
\frac r\beta\le \rho_3,
\end{equation}
and
\begin{equation}\label{eq:signal-dominance-3}
\Gamma_r(\beta)\le \frac\beta4,
\end{equation}
then the update is well defined and
\begin{equation}\label{eq:local-recursion-3}
\delta_{t+1}
\le
\frac{C_0}{\beta}
+
\kappa_r^{\aff}(\beta)\delta_t.
\end{equation}
In particular, for fixed \(r\) and \(L_0\), condition \eqref{eq:signal-dominance-3} holds eventually under the high-signal regime \eqref{eq:high-signal-3}.
\end{theorem}

\noindent\emph{Proof deferred to Appendix~\ref{app:proof-20}.}

\begin{proposition}[Explicit solution of the affine recursion]\label{prop:explicit-affine-3}
Assume the hypotheses of Theorem~\ref{thm:local-recursion-3}.  Let
\[
A_\beta:=\frac{C_0}{\beta},
\qquad
q_\beta:=\kappa_r^{\aff}(\beta).
\]
Suppose that the iterates remain in $\calB_r$ for times $s,s+1,\ldots,t$ and that
$0\le q_\beta<1$.  Then, for every $t\ge s$,
\begin{equation}\label{eq:delta-closed-form-3}
\delta_t
\le
q_\beta^{t-s}\delta_s
+
\frac{A_\beta(1-q_\beta^{t-s})}{1-q_\beta}.
\end{equation}
Equivalently,
\begin{equation}\label{eq:delta-floor-split-3}
\delta_t-\frac{A_\beta}{1-q_\beta}
\le
q_\beta^{t-s}
\left(
\delta_s-\frac{A_\beta}{1-q_\beta}
\right).
\end{equation}
Consequently,
\begin{equation}\label{eq:delta-limsup-3}
\delta_t
\le
q_\beta^{t-s}\delta_s+\frac{A_\beta}{1-q_\beta},
\qquad
\limsup_{t\to\infty}\delta_t\le\frac{A_\beta}{1-q_\beta}.
\end{equation}
The number of local steps needed to reduce the transient term below $\varepsilon/\beta$
is bounded by
\begin{equation}\label{eq:hitting-time-floor-3}
t-s
\ge
\frac{\log(\beta\delta_s/\varepsilon)}{\log(1/q_\beta)}
\quad\Longrightarrow\quad
q_\beta^{t-s}\delta_s\le \frac{\varepsilon}{\beta}.
\end{equation}
\end{proposition}

\noindent\emph{Proof deferred to Appendix~\ref{app:proof-21}.}

\begin{remark}[Interpretation of the decay]\label{rem:natural-decay}
The additive term $A_\beta=C_0/\beta$ comes from fixed orthogonal noise contractions such
as $\bm{\xi}_u=\Proj_xZ(\bm{y},\bm{z})$.  Thus the coordinate error is not expected to decay to zero.
The natural deterministic estimate is exactly \eqref{eq:delta-closed-form-3}: a geometric
transient plus a nonzero floor.  If one wants an exponential notation, then
$q_\beta^{t-s}=\exp(-(t-s)\log(1/q_\beta))$, but this is only a rewriting of the same
geometric factor.
\end{remark}

\begin{proposition}[Two-point contraction in the local basin]\label{prop:two-point-contraction}
Assume that $\calE_{L_0}$ holds, $U=(\bm{u},\bm{v},\bm{w})$ and
$\widetilde U=(\widetilde{\bm{u}},\widetilde{\bm{v}},\widetilde{\bm{w}})$ belong to $\calB_r$, and
\eqref{eq:signal-dominance-3} holds.  Define
\[
d(U,\widetilde U)
:=
\max\{\|\bm{a}-\widetilde{\bm{a}}\|,\|\bm{b}-\widetilde{\bm{b}}\|,\|\bm{c}-\widetilde{\bm{c}}\|\}.
\]
Then
\begin{equation}\label{eq:two-point-contraction}
d(\calA(U),\calA(\widetilde U))
\le
\kappa_r^{\ctr}(\beta)d(U,\widetilde U).
\end{equation}
In particular, under fixed \(r,L_0\) and the high-signal regime \eqref{eq:high-signal-3},
\[
\kappa_r^{\ctr}(\beta)
\le
\frac{C_1}{\beta}
+
\frac{C_2rN^{1/2}}{\beta^2}
\]
after increasing the constants.
\end{proposition}

\noindent\emph{Proof deferred to Appendix~\ref{app:proof-22}.}

\begin{corollary}[Unique local fixed point and explicit convergence]\label{cor:local-fixed-point}
Assume that $\calE_{L_0}$ holds and that \eqref{eq:signal-dominance-3} holds.  Suppose
\[
\kappa_r^{\aff}(\beta)<1,
\qquad
\kappa_r^{\ctr}(\beta)<1,
\]
and
\begin{equation}\label{eq:self-map-condition-3}
\frac{C_0}{\beta}
+
\kappa_r^{\aff}(\beta)\frac r\beta
\le
\frac r\beta.
\end{equation}
Then $\calA$ admits a unique fixed point
\[
U_\star=(\bm{u}_\star,\bm{v}_\star,\bm{w}_\star)\in\calB_r.
\]
Moreover,
\begin{equation}\label{eq:fixed-coordinate-error-3}
\delta_\star
:=
\delta(U_\star)
\le
\frac{C_0}{\beta(1-\kappa_r^{\aff}(\beta))},
\end{equation}
and
\begin{equation}\label{eq:fixed-euclidean-error-3}
\|\bm{u}_\star-\bm{x}\|+\|\bm{v}_\star-\bm{y}\|+\|\bm{w}_\star-\bm{z}\|
\le
\frac{C}{\beta}.
\end{equation}
For any $U_s\in\calB_r$ and every $t\ge s$,
\begin{equation}\label{eq:fixed-point-convergence-3}
d(U_t,U_\star)
\le
\bigl(\kappa_r^{\ctr}(\beta)\bigr)^{t-s}d(U_s,U_\star).
\end{equation}
Furthermore, the coordinate error relative to the planted spike satisfies
\begin{equation}\label{eq:coordinate-error-to-spike-3}
\delta_t
\le
\bigl(\kappa_r^{\ctr}(\beta)\bigr)^{t-s}d(U_s,U_\star)
+
\delta_\star.
\end{equation}
\end{corollary}

\noindent\emph{Proof deferred to Appendix~\ref{app:proof-23}.}

\begin{remark}[Choosing the local radius]\label{rem:choosing-r}
The radius $r$ should be chosen before the high-signal constant.  For instance, take
$r\ge2C_0$.  Under the high-signal regime \eqref{eq:high-signal-3}, increasing $C$ makes
$\kappa_r^{\aff}(\beta)\le1/2$, $\kappa_r^{\ctr}(\beta)<1$, and
$C_0/\beta+\kappa_r^{\aff}(\beta)r/\beta\le r/\beta$.  Thus the self-map condition follows from
choosing a sufficiently large but fixed local radius and then a
sufficiently large signal level.  Although the deterministic statement is phrased for a
fixed radius parameter, the same inequalities may also be applied to a deterministic
sequence $r=r_N$, provided the displayed finite-sample conditions
$r_N/\beta\le\rho_3$, $\Gamma_{r_N}(\beta)\le\beta/4$, and the corresponding
self-map and contraction inequalities are verified.  This is the form used in the
centered-Gram initialization theorem below.
\end{remark}

\begin{corollary}[Singular value control]\label{cor:lambda-3}
Under the assumptions of Corollary~\ref{cor:local-fixed-point}, let
\[
\lambda_\star:=\bm{T}(\bm{u}_\star,\bm{v}_\star,\bm{w}_\star).
\]
Then
\begin{equation}\label{eq:lambda-bound-3}
|\lambda_\star-\beta|
\le
C\beta\delta_\star^2
+
CL_0\bigl(1+\delta_\star+N^{1/2}\delta_\star^2\bigr)(1+\delta_\star).
\end{equation}
In particular, if \(\beta/N^{1/4}\to\infty\), then eventually
\[
|\lambda_\star-\beta|\le C.
\]
\end{corollary}

\noindent\emph{Proof deferred to Appendix~\ref{app:proof-24}.}

\begin{corollary}[Local KKT point]\label{cor:kkt-3}
The fixed point $U_\star=(\bm{u}_\star,\bm{v}_\star,\bm{w}_\star)$ satisfies
\[
\bm{T}(\bm{v}_\star,\bm{w}_\star)=\lambda_\star \bm{u}_\star,\qquad
\bm{T}(\bm{u}_\star,\bm{w}_\star)=\lambda_\star \bm{v}_\star,\qquad
\bm{T}(\bm{v}_\star)^\top \bm{u}_\star=\lambda_\star \bm{w}_\star.
\]
Thus the algorithmic local fixed point is a local informative stationary point of the same
KKT system used in the maximum-likelihood analysis.
\end{corollary}

\noindent\emph{Proof deferred to Appendix~\ref{app:proof-25}.}

\begin{remark}[First-order nature of the KKT statement]
Corollary~\ref{cor:kkt-3} is a first-order statement only.  It shows that the fixed point
of the alternating power map satisfies the KKT equations, but it does not by itself assert
that \(U_\star\) is a strict local maximizer of the likelihood.  Such a statement would
require a separate tangent-Hessian analysis on the product of spheres.
\end{remark}

\subsection{Additional details for the normalization map}

This subsection records the derivative calculation used in Lemma~\ref{lem:norm-pert}.  Let
\[
\Phi(s,\bm{g})=\frac{s\bm{x}+\bm{g}}{(s^2+\|\bm{g}\|^2)^{1/2}},
\qquad
\bm{g}\in \bm{x}^\perp.
\]
For a tangent perturbation $\bm{h}\in \bm{x}^\perp$,
\[
D_g\Phi(s,\bm{g})[\bm{h}]
=
\frac{\bm{h}}{(s^2+\|\bm{g}\|^2)^{1/2}}
-
\frac{(s\bm{x}+\bm{g})\langle \bm{g},\bm{h}\rangle}{(s^2+\|\bm{g}\|^2)^{3/2}}.
\]
If $\|\bm{g}\|\le s/2$, then $(s^2+\|\bm{g}\|^2)^{1/2}\asymp s$, and hence
\[
\|D_g\Phi(s,\bm{g})\|\le \frac{C}{s}.
\]
Similarly,
\[
\partial_s\Phi(s,\bm{g})
=
\frac{\bm{x}}{(s^2+\|\bm{g}\|^2)^{1/2}}
-
\frac{s(s\bm{x}+\bm{g})}{(s^2+\|\bm{g}\|^2)^{3/2}}.
\]
Combining the two terms gives
\[
\partial_s\Phi(s,\bm{g})
=
\frac{\|\bm{g}\|^2\bm{x}-s\bm{g}}{(s^2+\|\bm{g}\|^2)^{3/2}},
\]
and hence
\[
\|\partial_s\Phi(s,\bm{g})\|
\le
C\frac{\|\bm{g}\|}{s^2}.
\]
These two derivative bounds yield
\[
\|\Phi(s,\bm{g})-\Phi(\widetilde s,\widetilde{\bm{g}})\|
\le
\frac{C}{s\wedge\widetilde s}\|\bm{g}-\widetilde{\bm{g}}\|
+
\frac{C(\|\bm{g}\|+\|\widetilde{\bm{g}}\|)}{(s\wedge\widetilde s)^2}
|s-\widetilde s|.
\]

\subsection{Scale accounting for the order-three bilinear remainder}\label{sec:order-three-scale-bookkeeping}

For the $\bm{u}$-update, let $\bm{W}_{(1)}\in\R^{m\times np}$ be the mode-one matricization.  Then
\[
\bm{W}(\bm{b},\bm{c})=\bm{W}_{(1)}(\bm{b}\otimes \bm{c}),
\]
and therefore
\[
\|\noise(\bm{b},\bm{c})\|
\le
\frac1{\sqrt N}\|\bm{W}_{(1)}\|_{\op}\|\bm{b}\otimes \bm{c}\|
=
\frac1{\sqrt N}\|\bm{W}_{(1)}\|_{\op}\|\bm{b}\|\|\bm{c}\|.
\]
Because $\bm{W}_{(1)}$ is an $m\times np$ random matrix, we have 
\[
\|\bm{W}_{(1)}\|_{\op}
\le
C(\sqrt m+\sqrt{np})
\]
with high probability.  Under $m,n,p\asymp N$, then
\[
\frac{\sqrt m+\sqrt{np}}{\sqrt N}
=
O(N^{1/2}).
\]
This explains why the bilinear local remainder is bounded by
$ L_0N^{1/2}\|\bm{b}\|\|\bm{c}\| $.
Within the basin, $\|\bm{b}\|,\|\bm{c}\|\le r/\beta$, so the one-point bilinear contribution is
$ O\!\left(N^{1/2}\frac{r^2}{\beta^2}\right) $
before normalization.  After division by a signal of size $\beta$, it contributes
$ O\!\left(\frac{N^{1/2}r^2}{\beta^3}\right) $ 
to the next coordinate.  In the affine recursion, one uses
$ \delta_t^2\le\frac r\beta\delta_t $,
which gives the coefficient
$ O\!\left(\frac{rN^{1/2}}{\beta^2}\right)\delta_t $.

\section{Worked interpretation and high-signal scale accounting}
\label{sec:order-d-worked-example}
\label{sec:specialization-scaling}

This section condenses the order-\(d\) subset algebra and the scale accounting behind the local theorem.  It is included only as an explanatory guide; the formal statements and proofs are in Section~\ref{sec:order-d-extension} and Appendix~\ref{app:proofs-order-d}.

Fix a target mode \(i\) and write \(I_i=[d]\setminus\{i\}\).  In local coordinates \(u^{(j)}=\alpha_jx^{(j)}+a_j\), multilinearity gives one noise term for each subset \(S\subseteq I_i\).  The empty subset is the fixed contraction at the planted point, the singleton subsets are linear noise sensitivities, and all subsets with at least two perturbation directions form the higher-order remainder.  On \(\mathcal F_{d,N}(L,\Theta)\), this decomposition yields
\[
\|g_i(U)\|+|\zeta_i(U)|
\le C_dL(1+\delta(U))+C_d\Theta\delta(U)^2.
\]
Inside \(\mathcal B_r^{(d)}\), the inequality \(\delta(U)^2\le (r/\beta)\delta(U)\) converts the higher-order term into the local Lipschitz contribution \(\Theta r/\beta^2\).  Hence the one-step error has the affine form
\[
\delta_{t+1}
\le
\frac{C_dL}{\beta}
+
\left(\frac{C_dL}{\beta}+\frac{C_d\Theta r}{\beta^2}\right)\delta_t.
\]
The additive term is the fixed-noise floor, while the coefficient multiplying \(\delta_t\) is the local stability coefficient.

The crude fixed-order event verification gives \(\Theta=O(N^{(d-2)/2})\).  Therefore
\[
\frac{\Theta r}{\beta^2}
=
O\left(\frac{r}{\omega_{N,d}^2}\right),
\qquad
\omega_{N,d}:=\frac{\beta}{N^{(d-2)/4}}.
\]
Thus the deterministic finite-sample conditions in the local theorem reduce, for fixed radii or slowly expanding radii \(r_N=o(\omega_{N,d})\), to the high-signal condition \(\omega_{N,d}\to\infty\).  Under this regime the affine recursion reads
\[
\delta_{t+1}
\le \frac{C}{\beta}+o(1)\delta_t,
\]
which separates a finite-iteration geometric transient from the nonzero \(O(\beta^{-1})\) local noise floor.  This is the only scale information from the worked calculation used later in the warm-start analysis.

\section{Warm starts and centered-Gram initialization}\label{sec:warm-starts-final}

The preceding sections are deterministic after the iterates have entered a local basin.  We now
separate that local theory from initialization.  The first result is a generic one-sweep
warm-start principle: a mode-wise correlated initializer enters the \(r_N/\beta\) basin as soon as
its first-sweep noise level is small compared with \(\gamma_N^{d-1}r_N\).  The second result verifies
these two inputs for the same-sample centered-Gram initializer under i.i.d. centered
unit-variance entries with finite fourth moment.  Thus no coordinate-incoherence or separate
centered-Gram feasibility condition is imposed on the planted directions.

\begin{center}
\fbox{\begin{minipage}{0.92\textwidth}
\[
\begin{gathered}
\text{correlation } \gamma_N \text{ and first-sweep noise } a_N
\quad\Longrightarrow\quad
U_1\in\calB_{r_N}^{(d)},\\
U_1\in\calB_{r_N}^{(d)}
\quad\Longrightarrow\quad
\delta(U_t)\le q_{N,d}(r_N)^{t-1}\delta(U_1)+O(\beta^{-1}),
\qquad
q_{N,d}(r_N)=O\!\left(\frac1\beta+\frac{r_N}{\omega_{N,d}^2}\right).
\end{gathered}
\]
\end{minipage}}
\end{center}

\subsection{Generic order-\texorpdfstring{$d$}{d} warm-start principle}

\begin{definition}[Initialization noise level]\label{def:init-noise-d}
For a set \(\mathcal V\subseteq\prod_{i=1}^d\Sph^{n_i-1}\), define
\[
L_{\rm init}(\mathcal V)
:=
\max_{1\le i\le d}
\sup_{U\in\mathcal V}
\left\|
\frac1{\sqrt N}\bm{W}(U^{(-i)})
\right\|.
\]
For a single initializer \(U_0\), write
\[
L_{\rm init}^{(d)}(U_0)
:=
\max_{1\le i\le d}
\left\|
\frac1{\sqrt N}\bm{W}(U_0^{(-i)})
\right\|.
\]
\end{definition}

\begin{theorem}[One-sweep entry from a general correlated initializer]\label{thm:warm-start-d}
Let \(r_N\to\infty\), let \(\gamma_N\in(0,1]\), and let \(a_N\ge1\).  Suppose that, with
probability tending to one, there is a sign vector
\(\tau=(\tau_1,\ldots,\tau_d)\in\{\pm1\}^d\), \(\prod_i\tau_i=1\), such that
\[
\min_{1\le i\le d}
\langle \bm{u}_0^{(i)},\tau_i\bm{x}^{(i)}\rangle
\ge \gamma_N,
\qquad
L_{\rm init}^{(d)}(U_0)\le C a_N
\]
for a fixed constant \(C<\infty\).  Assume
\begin{equation}\label{eq:init-general-scale-d}
\frac{a_N}{\gamma_N^{d-1}r_N}\longrightarrow0,
\qquad
r_N=o(\omega_{N,d}),
\qquad
\omega_{N,d}:=\frac{\beta}{N^{(d-2)/4}}\to\infty,
\end{equation}
and assume that the deterministic local hypotheses of Corollary~\ref{cor:fixed-point-d}
hold on \(\calB_{r_N}^{(d)}\).  After replacing the planted tuple by the sign-compatible
representative \((\tau_i\bm{x}^{(i)})_{i=1}^d\), one simultaneous alternating sweep satisfies
\[
U_1=\calA(U_0)\in\calB_{r_N}^{(d)}
\]
with probability tending to one.  More precisely,
\[
\max_i\sin\angle(\bm{u}_1^{(i)},\tau_i\bm{x}^{(i)})
=
O_{\mathbb P}\!\left(\frac{a_N}{\beta\gamma_N^{d-1}}\right).
\]
Consequently, for every finite \(t\ge1\),
\begin{equation}\label{eq:generic-warm-affine-d}
\delta(U_t)
\le
q_{N,d}(r_N)^{t-1}\delta(U_1)
+
\frac{C_{\loc}/\beta}{1-q_{N,d}(r_N)},
\qquad
q_{N,d}(r_N)
\equiv \kappa_{r_N,d}^{\aff}(\beta)
<1,
\end{equation}
and the iterates converge geometrically to the unique informative local fixed point in that
sign-compatible basin.
Equivalently, in a probabilistic verification one may take the intersection of the warm-start
event above with the event on which the deterministic local hypotheses hold; the theorem only
uses the deterministic local theory on that intersection.
If, in addition, the crude fixed-order event verification gives \(L=O(1)\) and
\(\Theta=O(N^{(d-2)/2})\), then
\[
q_{N,d}(r_N)
=O\!\left(\frac1\beta+\frac{r_N}{\omega_{N,d}^{2}}\right)
=o(1).
\]
\end{theorem}

\noindent\emph{Proof deferred to Appendix~\ref{app:proof-10}.}

\begin{corollary}[Minimal abstract warm-start condition]\label{cor:generic-warm-start-local-d}
Suppose that, after sign alignment,
\[
\min_i\langle \bm{u}_0^{(i)},\bm{x}^{(i)}\rangle\ge\gamma_N,
\qquad
L_{\rm init}^{(d)}(U_0)=O_{\mathbb P}(a_N),
\]
and
\begin{equation}\label{eq:minimal-warm-condition-d}
\frac{a_N}{\gamma_N^{d-1}\omega_{N,d}}\longrightarrow0.
\end{equation}
Then there exists a deterministic sequence \(r_N\to\infty\) with
\[
\frac{a_N}{\gamma_N^{d-1}}=o(r_N),
\qquad
r_N=o(\omega_{N,d}),
\]
and the conclusion of Theorem~\ref{thm:warm-start-d} holds.  In particular, if
\(L_{\rm init}^{(d)}(U_0)=O_{\mathbb P}(1)\), then \(\gamma_N^{d-1}\omega_{N,d}\to\infty\)
is sufficient.
\end{corollary}

\begin{proof}
Since \(L_{\rm init}^{(d)}(U_0)=O_{\mathbb P}(a_N)\), there exists a deterministic sequence
\(M_N\to\infty\), which may be chosen arbitrarily slowly, such that
\[
\mathbb P\{L_{\rm init}^{(d)}(U_0)\le M_Na_N\}\to1.
\]
By \eqref{eq:minimal-warm-condition-d}, choose \(M_N\to\infty\) slowly enough that
\[
\frac{M_Na_N}{\gamma_N^{d-1}\omega_{N,d}}\to0.
\]
Now choose \(r_N\) between \(M_Na_N/\gamma_N^{d-1}\) and \(\omega_{N,d}\).  A convenient
choice is
\[
r_N=
\left(
\omega_{N,d}\max\left\{\frac{M_Na_N}{\gamma_N^{d-1}},1\right\}
\right)^{1/2}.
\]
The maximum with \(1\) guarantees \(r_N\to\infty\); no lower endpoint
\(M_Na_N/\gamma_N^{d-1}\to\infty\) is required.  The preceding display implies
\[
\frac{M_Na_N}{\gamma_N^{d-1}r_N}\to0,
\qquad
r_N=o(\omega_{N,d}).
\]
On the high-probability event \(L_{\rm init}^{(d)}(U_0)\le M_Na_N\), Theorem~\ref{thm:warm-start-d}
applies with the deterministic noise scale \(M_Na_N\), giving the basin-entry event.  The sharper
one-sweep estimate
\[
\max_i\sin\angle(\bm u_1^{(i)},\bm x^{(i)})
=O_{\mathbb P}\!\left(\frac{a_N}{\beta\gamma_N^{d-1}}\right)
\]
follows directly from the proof of Theorem~\ref{thm:warm-start-d} and
\(L_{\rm init}^{(d)}(U_0)=O_{\mathbb P}(a_N)\).
\end{proof}

\begin{corollary}[Independent or sample-split warm starts]\label{cor:independent-warm-start-d}
Suppose \(U_0\) is independent of the noise tensor used in the first alternating sweep and has
mode-wise correlation at least \(\gamma_N\) with the planted directions, after sign alignment.  If the
noise entries are independent, centered, and have uniformly bounded variances, then
\[
L_{\rm init}^{(d)}(U_0)=O_{\mathbb P}(1).
\]
Consequently, whenever \(\gamma_N^{d-1}\omega_{N,d}\to\infty\), one simultaneous alternating
sweep enters the local basin and the local finite-iteration contraction bound applies.
\end{corollary}

\begin{proof}
Condition on \(U_0\).  For each target mode \(i\), the coordinates of
\(N^{-1/2}\bm W(U_0^{(-i)})\) are centered and have conditional variances \(O(1/N)\).
Since \(n_i\asymp N\), the conditional second moment of the squared norm is \(O(1)\).
Markov's inequality and a union bound over the fixed number of modes give
\(L_{\rm init}^{(d)}(U_0)=O_{\mathbb P}(1)\).  The conclusion is
Corollary~\ref{cor:generic-warm-start-local-d}.
\end{proof}

\subsection{Centered-Gram initialization under i.i.d. finite fourth moments}\label{subsec:centered-gram-d}

In this subsection, the same-sample centered-Gram statements are made under the conservative
assumption that the noise entries are i.i.d., centered, unit-variance, and have finite fourth
moment.  The deterministic local theory above only needs the stated rectangular norm event; the
i.i.d. assumption is used here to keep the sample-covariance and truncation inputs fully standard.

For each mode \(i\), let
\[
\bm{T}_{(i)}\in\mathbb R^{n_i\times\prod_{j\ne i}n_j}
\]
be the mode-\(i\) unfolding and define the centered Gram matrix
\begin{equation}\label{eq:centered-gram-d}
\bm{G}_i:=\bm{T}_{(i)}\bm{T}_{(i)}^\top-\frac{\prod_{j\ne i}n_j}{N}\bm{I}_{n_i}.
\end{equation}
Let \(\bm{u}_0^{(i)}\) be a leading unit eigenvector of \(\bm{G}_i\), with signs chosen in a
sign-compatible planted chart.

The implementable sign choice is justified after the same-sample first-sweep bound is stated:
one only needs to fix the product of the mode signs, not an absolute sign in each mode.

\begin{proposition}[Fixed-order centered-Gram weak recovery]\label{prop:centered-gram-weak-d}
Assume \(d\ge3\) is fixed, \(n_i\asymp N\), and the entries of \(\bm W\) are i.i.d.,
centered, unit-variance, and have finite fourth moment.  If
\[
\omega_{N,d}:=\frac{\beta}{N^{(d-2)/4}}\longrightarrow\infty,
\]
then, up to mode-wise signs,
\[
\max_{1\le i\le d}\sin\angle(\bm{u}_0^{(i)},\bm{x}^{(i)})
=o_{\mathbb P}(1).
\]
In particular, with probability tending to one, the centered-Gram initializer has constant
mode-wise correlation with the planted tuple.
\end{proposition}

\noindent\emph{Proof deferred to Appendix~\ref{app:proof-11}.}

\subsection{Same-sample first-sweep verification}\label{subsec:d-same-sample-reduction}

We next verify \(L_{\rm init}^{(d)}(U_0)=O_{\mathbb P}(1)\) for the same tensor used to compute
the centered-Gram initializer.  The dependence is handled coordinate by coordinate through a
signal-preserving noise-leave-one construction.  Unlike full slice deletion, this removes only the
noise entries in the coordinate slice and keeps the deterministic rank-one spike unchanged.

Fix a target mode \(i\) and a coordinate \(a\in[n_i]\).  Let
\[
\mathcal W_{i,a}:=\{\bm W_{j_1,\ldots,j_d}:j_i=a\}
\]
be the deleted noise slice.  Let \(\bm W^{(-i,a)}\) be obtained by replacing the entries in
\(\mathcal W_{i,a}\) by zero and keeping all other noise entries unchanged, and set
\[
\bm T^{(-i,a)}
=
\beta\bm x^{(1)}\otimes\cdots\otimes\bm x^{(d)}
+N^{-1/2}\bm W^{(-i,a)}.
\]
For the leave-one initializer we use the eigenvector-equivalent recentered Gram representative.
For \(\ell\ne i\), put
\[
P_{\ell;i}:=\prod_{q\ne i,\ell}n_q,
\qquad
\bm G_{\ell;i,a}
:=
\bm T_{(\ell)}^{(-i,a)}\bm T_{(\ell)}^{(-i,a)\top}
-
\frac{(n_i-1)P_{\ell;i}}{N}\bm I_{n_\ell}.
\]
Equivalently, computing the leave-one Gram with the full centering constant changes
\(\bm G_{\ell;i,a}\) only by the scalar identity shift \(P_{\ell;i}N^{-1}\bm I_{n_\ell}\), hence
does not change its leading eigenvectors.  In perturbation identities below we use the recentered
representative above, so the identity shift is not included in
\(\bm\Delta_{\ell;i,a}:=\bm G_\ell-\bm G_{\ell;i,a}\).
Let \(U_{0;i,a}\) denote the resulting leave-one centered-Gram initializer.  For
\(\ell\ne i\), choose signs so that
\(\langle \bm u_0^{(\ell)},\bm u_{0;i,a}^{(\ell)}\rangle\ge0\), and set
\begin{equation}\label{eq:D-la-def-d}
\bm d_{\ell,a}:=\bm u_0^{(\ell)}-\bm u_{0;i,a}^{(\ell)}.
\end{equation}
For later use, define the mode-\(\ell\) matrix of the deleted slice by
\[
\bm W_{\ell;i,a}\in\mathbb R^{n_\ell\times P_{\ell;i}},
\qquad
\bigl(\bm W_{\ell;i,a}\bigr)_{j_\ell,\mathbf j_{-\{i,\ell\}}}
:=
W_{j_1,\ldots,j_{i-1},a,j_{i+1},\ldots,j_d},
\]
where \(\mathbf j_{-\{i,\ell\}}\) collects the coordinates other than \(i\) and \(\ell\).
For a nonempty subset \(S\subseteq I_i:=[d]\setminus\{i\}\), define
\begin{equation}\label{eq:directional-C-def-d}
\mathcal C_{i,a,S}\bigl((\bm h_\ell)_{\ell\in S}\bigr)
:=
\frac1{\sqrt N}
\bm W(\bm v^{(1)},\ldots,\bm v^{(d)}),
\end{equation}
where \(\bm v^{(i)}=e_a^{(i)}\), \(\bm v^{(\ell)}=\bm h_\ell\) for \(\ell\in S\), and
\(\bm v^{(q)}=\bm u_{0;i,a}^{(q)}\) for \(q\in I_i\setminus S\).

\begin{lemma}[Averaged leave-one expansion]\label{lem:averaged-leaveone-d}
Assume the i.i.d. finite-fourth-moment setting of Proposition~\ref{prop:centered-gram-weak-d}.  For fixed
\(i\), uniformly over \(\ell\ne i\) and in summed square over \(a\in[n_i]\),
\[
\bm d_{\ell,a}
=
\bm L_{\ell,a}+\bm Q_{\ell,a}+\bm E_{\ell,a},
\]
where
\[
\bm L_{\ell,a}
=
\frac{x_a^{(i)}}{\beta}\bm\Pi_\ell \bm h_{\ell;i,a},
\qquad
\bm Q_{\ell,a}
=
\frac1{\beta^2}\bm\Pi_\ell \bm K_{\ell;i,a}\bm x^{(\ell)}.
\]
Here \(\bm\Pi_\ell=\bm I-\bm x^{(\ell)}\bm x^{(\ell)\top}\),
\[
\bm h_{\ell;i,a}:=N^{-1/2}\bm W_{\ell;i,a}\bm x^{(-i,\ell)},\qquad
\bm K_{\ell;i,a}
:=
N^{-1}\bm W_{\ell;i,a}\bm W_{\ell;i,a}^{\top}
-\frac{P_{\ell;i}}{N}\bm I_{n_\ell},
\]
with \(P_{\ell;i}=\prod_{q\ne i,\ell}n_q\).  Moreover,
\[
\sum_a\sum_{\ell\ne i}\|\bm L_{\ell,a}\|^2=O_{\mathbb P}(\beta^{-2}),
\qquad
\sum_a\sum_{\ell\ne i}\|\bm Q_{\ell,a}\|^2=O_{\mathbb P}(\omega_{N,d}^{-4}),
\]
and
\[
\sum_a\sum_{\ell\ne i}\|\bm E_{\ell,a}\|^2
=
o_{\mathbb P}(\beta^{-2}+\omega_{N,d}^{-4}).
\]
More precisely, with
\[
\vartheta_\ell:=\sin\angle(\bm u_0^{(\ell)},\bm x^{(\ell)}),
\qquad
\theta_{\ell,a}:=\sin\angle(\bm u_{0;i,a}^{(\ell)},\bm x^{(\ell)}),
\qquad
\bm\Delta_{\ell;i,a}:=\bm G_\ell-\bm G_{\ell;i,a},
\]
and
\[
\eta_{\ell,a}:=
\frac{\|\bm G_{\ell;i,a}-\beta^2\bm x^{(\ell)}\bm x^{(\ell)\top}\|_{\op}}{\beta^2},
\]
and with
\[
\mathfrak r_{\ell,a}:=
\frac{|x_a^{(i)}|\|\bm h_{\ell;i,a}\|}{\beta}
+
\frac{\|\bm K_{\ell;i,a}\bm x^{(\ell)}\|}{\beta^2},
\]
the residual can be chosen so that
\[
\begin{aligned}
\sum_a\sum_{\ell\ne i}\|\bm E_{\ell,a}\|^2
&\le
C\sum_a\sum_{\ell\ne i}
\left(
\eta_{\ell,a}^2+\vartheta_\ell^2+\theta_{\ell,a}^2+\frac{\|\bm\Delta_{\ell;i,a}\|_{\op}^2}{\beta^4}
\right)\mathfrak r_{\ell,a}^2\\
&\quad+
C\sum_a\sum_{\ell\ne i}\mathfrak r_{\ell,a}^4
+o_{\mathbb P}(\beta^{-2}+\omega_{N,d}^{-4})\\
&=o_{\mathbb P}(\beta^{-2}+\omega_{N,d}^{-4}).
\end{aligned}
\]
\end{lemma}

\noindent\emph{Proof deferred to Appendix~\ref{app:proof-d-averaged-leaveone}.}

The coordinate \(x_a^{(i)}\) enters the leading linear motion only through averaged
\(\ell_2\)-weights such as
\[
\sum_a (x_a^{(i)})^2\|\bm h_{\ell;i,a}\|^2,
\]
and never through a worst-case \(\max_a |x_a^{(i)}|\) bound.  This is the mechanism behind the
absence of any planted-coordinate incoherence assumption in the same-sample argument.

We call the next estimate a \emph{pressed-back} directional slice-chaos bound because the
same-sample dependence from the full centered-Gram eigenvectors is pressed back to averaged
signal-preserving leave-one motions, rather than being controlled by a worst-case deleted-slice
operator norm.

\begin{lemma}[Pressed-back directional slice chaos]\label{lem:pressed-slice-chaos-d}
Under the same assumptions, for every fixed target mode \(i\),
\begin{equation}\label{eq:pressed-slice-chaos-d}
\sum_{a=1}^{n_i}
\left|
\sum_{\emptyset\ne S\subseteq I_i}
\mathcal C_{i,a,S}\bigl((\bm d_{\ell,a})_{\ell\in S}\bigr)
\right|^2
=o_{\mathbb P}(1).
\end{equation}
\end{lemma}

\noindent\emph{Proof deferred to Appendix~\ref{app:proof-d-pressed-slice-chaos}.}

The proof expands each \(L/Q/E\) leave-one motion before conditioning on the deleted slice.  The
basic bounded-array estimate is, for \(|S|=q\) and \(r\) linear \(L\)-motions,
\[
\Ebb\sum_a
\left|
\mathcal C_{i,a,S}\bigl((\bm D_{\ell,a})_{\ell\in S}\bigr)
\right|^2
\le
C_db_N^{C_d}
N^{q-1}\beta^{-2r}\beta^{-4(q-r)}N^{(d-3)(q-r)}.
\]
The extra factor \(N^{q-1}\) deliberately allows same-copy contractions and non-symmetric
third-moment collisions.  With \(\beta=N^{(d-2)/4}\omega_{N,d}\), the \(N\)-power is
\(N^{-1+r(4-d)/2}\); the only fixed-order borderline case is \(d=3,r=2\), where the remaining
factor \(\omega_{N,3}^{-4}\) still tends to zero.  Residual \(E\)-terms and finite-fourth-moment
truncation are reduced back to the same diagrams in Appendix~C.

\begin{proposition}[Same-sample centered-Gram first-sweep noise]\label{prop:d-same-sample-linit}
Under the i.i.d. finite-fourth-moment setting of Proposition~\ref{prop:centered-gram-weak-d}, the
same-sample centered-Gram initializer satisfies
\begin{equation}\label{eq:linit-d-proved}
L_{\rm init}^{(d)}(U_0)
=
\max_{1\le i\le d}
\left\|
\frac1{\sqrt N}\bm W(U_0^{(-i)})
\right\|
=O_{\mathbb P}(1).
\end{equation}
\end{proposition}

\noindent\emph{Proof deferred to Appendix~\ref{app:proof-d-same-sample-linit}.}

\begin{remark}[Implementable sign alignment]\label{rem:implementable-sign-alignment}
After computing arbitrary leading eigenvector signs, flip one mode if necessary so that
\(\bm T(\bm u_0^{(1)},\ldots,\bm u_0^{(d)})>0\).  Under
Propositions~\ref{prop:centered-gram-weak-d} and~\ref{prop:d-same-sample-linit}, this selects,
with probability tending to one, a representative
\((\tau_i\bm x^{(i)})_{i=1}^d\) satisfying \(\prod_i\tau_i=1\).  Indeed, weak recovery gives
\(\prod_i|\langle \bm u_0^{(i)},\bm x^{(i)}\rangle|\ge c>0\) with high probability, while the
same-sample noise term in
\(\bm T(\bm u_0^{(1)},\ldots,\bm u_0^{(d)})\) is at most \(N^{-1/2}
\|\bm W(U_0^{(-i)})\|=O_{\mathbb P}(1)\) for any fixed mode \(i\).  Since
\(\beta\to\infty\), the scalar contraction has the sign of the product of the mode-wise planted
signs with high probability.  The flip therefore fixes only this product sign; it does not choose
absolute signs for the individual modes.
\end{remark}

\begin{theorem}[Same-sample centered-Gram initialization and alternating power]\label{thm:centered-gram-ap-d}
Assume \(d\ge3\) is fixed, \(n_i\asymp N\), the noise entries are i.i.d., centered,
unit-variance, and have finite fourth moment, and
\[
\omega_{N,d}:=\frac{\beta}{N^{(d-2)/4}}\to\infty.
\]
Let \(U_0\) be the centered-Gram initializer with the implementable product-sign convention of
Remark~\ref{rem:implementable-sign-alignment}.  Equivalently, the following statements are made
after replacing the planted tuple by a sign-compatible representative
\((\tau_i\bm x^{(i)})_{i=1}^d\) with \(\prod_i\tau_i=1\).  For every deterministic sequence
\begin{equation}\label{eq:rN-conditions-d}
r_N\to\infty,
\qquad
r_N=o(\omega_{N,d}),
\end{equation}
the high-probability event of Proposition~\ref{prop:raw-event-hp-d} verifies the deterministic local
hypotheses on \(\calB_{r_N}^{(d)}\) eventually, and
one simultaneous alternating sweep enters the local basin:
\[
U_1=\calA(U_0)\in\calB_{r_N}^{(d)}
\]
with probability tending to one, and
\[
\delta(U_1)=O_{\mathbb P}(1/\beta).
\]
Moreover, the iterates remain in \(\calB_{r_N}^{(d)}\) and converge geometrically to the unique
informative fixed point \(U_\star\) in that local basin.  More explicitly, for \(t\ge1\),
\begin{equation}\label{eq:fixed-d-ap-affine-explicit}
\delta(U_t)
\le
q_{N,d}^{\,t-1}\delta(U_1)
+
\frac{C/\beta}{1-q_{N,d}},
\qquad
q_{N,d}
=
O\!\left(\frac1\beta+\frac{r_N}{\omega_{N,d}^2}\right)
=o(1).
\end{equation}
In particular, uniformly after the local entry time,
\[
\delta(U_t)=O_{\mathbb P}(1/\beta).
\]
\end{theorem}

\noindent\emph{Proof deferred to Appendix~\ref{app:proof-13}.}

\section{Conclusion}

This paper develops a fixed-order local theory for simultaneous alternating power iteration in
asymmetric rank-one spiked tensor models.  Conditional on the explicit subset-indexed multilinear
noise-control event, the alternating power map has a self-map property on a local basin, an affine
coordinate recursion, a two-point contraction, a unique local fixed point, and explicit finite-iteration
bounds.  The key quantitative statement is finite-time in the iteration variable: for every finite pair
of local times \(s\le t\),
\[
\delta_t
\le
q_{\beta,d}^{\,t-s}\delta_s+
\frac{A_{\beta,d}(1-q_{\beta,d}^{\,t-s})}{1-q_{\beta,d}}.
\]
Thus the algorithm has a geometric transient with rate \(q_{\beta,d}\), reaches an
\(\varepsilon/\beta\)-neighborhood of its noise floor after
\[
O\!\left(\frac{\log(\beta\delta_s/\varepsilon)}{\log(1/q_{\beta,d})}\right)
\]
local iterations, and then settles at the intrinsic noise floor of order \(1/\beta\).

The high-signal regime is formulated with a diverging gap rather than a constant critical scaling.  For
fixed order \(d\), the crude subset-indexed event verification gives the scale
\(\beta/N^{(d-2)/4}\to\infty\); the order-three case is recovered by setting \(d=3\).  In these regimes
the local affine coefficient is \(o(1)\), both for fixed local radii and for the slowly expanding radii
used to absorb \(O_{\mathbb P}(1)\) initialization noise levels.

The warm-start part is organized so that the algorithmic mechanism is not tied to centered-Gram.
The generic result allows correlation \(\gamma_N\) and first-sweep noise level \(a_N\), requiring only
\(a_N/(\gamma_N^{d-1}\omega_{N,d})\to0\).  Independent or sample-split warm starts satisfy the
noise requirement directly.  For same-sample centered-Gram initialization, i.i.d.
finite-fourth-moment entries already give weak recovery, and the first-sweep noise is controlled by a signal-preserving
noise-leave-one expansion plus a pressed-back averaged slice-chaos estimate.  This proves
\(L_{\rm init}^{(d)}(U_0)=O_{\mathbb P}(1)\) for every fixed order \(d\), with no
coordinate-incoherence assumption and no separate centered-Gram feasibility condition.  Once the
first sweep enters the local basin, all remaining convergence statements use only the deterministic
alternating-power local theory.

Several natural extensions are left open.  The same-sample centered-Gram verification is stated in
the i.i.d. finite-fourth-moment setting; extending that part to non-identically distributed
triangular arrays would require a correspondingly uniform Bai--Yin/Lindeberg and truncation
framework for all unfoldings and leave-one slices.  The present paper also treats simultaneous API
for asymmetric tensors; symmetric tensors, sequential updates, and sharper global recovery
thresholds are separate questions.  These directions are compatible with the local deterministic
recursion, but would require initializer-specific or model-specific input beyond what is needed
for the finite-iteration local law proved here.

\clearpage
\appendix
\section{Proofs for Section~\ref{sec:order-d-extension}: fixed-order local theory}\label{app:proofs-order-d}

This appendix contains the proofs of the formal statements in Section~\ref{sec:order-d-extension}.  These are the deterministic finite-iteration local theory and its crude probabilistic event verification.

\subsection{Proof of Lemma~\ref{lem:alpha-product-d} 
(Elementary coordinate bounds)}
\label{app:proof-01}

\begin{proof}
By definition,
\[
\rho_d=\min\left\{\frac14,\frac{1}{2\sqrt d}\right\}.
\]
Assume
\[
\delta(U),\delta(\widetilde U)\le\rho_d.
\]
Then, for every mode \(j\),
\[
\|\bm a_j\|\le\rho_d,
\qquad
\|\widetilde{\bm a}_j\|\le\rho_d.
\]

We first prove the product lower bound.  Since
\[
\sqrt{1-x}\ge 1-x,\qquad 0\le x\le1,
\]
we have
\[
\alpha_j
=
\sqrt{1-\|\bm a_j\|^2}
\ge
1-\|\bm a_j\|^2.
\]
Therefore
\[
\prod_{j\ne i}\alpha_j
\ge
\prod_{j\ne i}\left(1-\|\bm a_j\|^2\right).
\]
Using
\[
\prod_{k=1}^m(1-x_k)\ge 1-\sum_{k=1}^m x_k,
\qquad 0\le x_k\le1,
\]
we obtain
\[
\prod_{j\ne i}\alpha_j
\ge
1-\sum_{j\ne i}\|\bm a_j\|^2.
\]
Since \(\|\bm a_j\|\le\delta(U)\le\rho_d\),
\[
\sum_{j\ne i}\|\bm a_j\|^2
\le
(d-1)\rho_d^2
\le
d\rho_d^2
\le
\frac14.
\]
Hence
\[
\prod_{j\ne i}\alpha_j\ge\frac34.
\]

Next we prove the product Lipschitz estimate.  For a single coefficient,
\[
|\alpha_i-\widetilde\alpha_i|
=
\left|
\sqrt{1-\|\bm a_i\|^2}
-
\sqrt{1-\|\widetilde{\bm a}_i\|^2}
\right|.
\]
Using
\[
|\sqrt A-\sqrt B|
=
\frac{|A-B|}{\sqrt A+\sqrt B},
\qquad A,B>0,
\]
with
\[
A=1-\|\bm a_i\|^2,
\qquad
B=1-\|\widetilde{\bm a}_i\|^2,
\]
gives
\[
|\alpha_i-\widetilde\alpha_i|
=
\frac{
\left|\|\bm a_i\|^2-\|\widetilde{\bm a}_i\|^2\right|
}{
\alpha_i+\widetilde\alpha_i
}.
\]
Since \(\rho_d\le1/4\),
\[
\alpha_i+\widetilde\alpha_i
\ge
2\sqrt{1-\rho_d^2}
\ge
1.
\]
Therefore
\[
|\alpha_i-\widetilde\alpha_i|
\le
\left|
\|\bm a_i\|^2-\|\widetilde{\bm a}_i\|^2
\right|.
\]
Moreover,
\[
\|\bm a_i\|^2-\|\widetilde{\bm a}_i\|^2
=
\langle \bm a_i+\widetilde{\bm a}_i,\bm a_i-\widetilde{\bm a}_i\rangle,
\]
so by Cauchy--Schwarz,
\[
|\alpha_i-\widetilde\alpha_i|
\le
\left(\|\bm a_i\|+\|\widetilde{\bm a}_i\|\right)
\|\bm a_i-\widetilde{\bm a}_i\|.
\]
Using the definitions of \(\delta(\cdot)\) and \(d_{\loc}(\cdot,\cdot)\), we get
\[
|\alpha_i-\widetilde\alpha_i|
\le
\bigl(\delta(U)+\delta(\widetilde U)\bigr)
d_{\loc}(U,\widetilde U).
\]

Now let \(J\subseteq[d]\).  If \(J=\varnothing\), the estimate is trivial because both
products are empty products.  Otherwise write
\[
J=\{j_1,\ldots,j_m\}.
\]
The telescoping identity gives
\[
\prod_{\ell=1}^m\alpha_{j_\ell}
-
\prod_{\ell=1}^m\widetilde\alpha_{j_\ell}
=
\sum_{k=1}^m
\left(\prod_{\ell<k}\widetilde\alpha_{j_\ell}\right)
(\alpha_{j_k}-\widetilde\alpha_{j_k})
\left(\prod_{\ell>k}\alpha_{j_\ell}\right).
\]
Since all \(\alpha_j,\widetilde\alpha_j\in[0,1]\),
\[
\left|
\prod_{j\in J}\alpha_j
-
\prod_{j\in J}\widetilde\alpha_j
\right|
\le
\sum_{k=1}^m|\alpha_{j_k}-\widetilde\alpha_{j_k}|.
\]
Using the single-coefficient bound and \(m=|J|\le d\), we obtain
\[
\left|
\prod_{j\in J}\alpha_j
-
\prod_{j\in J}\widetilde\alpha_j
\right|
\le
C_d
\bigl(\delta(U)+\delta(\widetilde U)\bigr)
d_{\loc}(U,\widetilde U).
\]
This proves \eqref{eq:alpha-product-lip-d}.

Finally, if
\[
\delta(U),\delta(\widetilde U)\le r/\beta
\quad\text{and}\quad
r/\beta\le\rho_d,
\]
then
\[
\delta(U)+\delta(\widetilde U)
\le
\frac{2r}{\beta}.
\]
Absorbing the factor \(2\) into \(C_d\) gives
\[
\left|
\prod_{j\in J}\alpha_j
-
\prod_{j\in J}\widetilde\alpha_j
\right|
\le
C_d\frac r\beta d_{\loc}(U,\widetilde U).
\]
This proves \eqref{eq:alpha-product-lip-basin-d}.
\end{proof}

\subsection{Proof of Proposition~\ref{prop:raw-event-hp-d}
(A crude probabilistic verification for fixed order)}
\label{app:proof-02}

\begin{proof}
Throughout the proof, constants may depend on the fixed order \(d\), the
aspect-ratio bounds \(n_i\asymp N\), and the uniform fourth-moment bound, but not
on \(N\).  Since \(d\) is fixed, the number of modes and subsets \(S\) is finite.
Thus, after proving the required bound for each fixed admissible contraction, a
finite union bound completes the proof.

We use the following standard rectangular matrix estimate.  If
\(\bm A\in\mathbb R^{M\times K}\) has independent centered entries with unit
variance and uniformly bounded fourth moments, then
\[
\|\bm A\|_{\op}
=
O_{\mathbb P}(\sqrt M+\sqrt K).
\]
This is the same input as Lemma~\ref{lem:rectangular-op-input}.  We shall apply it
to matricizations of tensor contractions.  We first record why those matricizations
still satisfy the required moment assumptions.
The \(O_{\mathbb P}\) notation in the displays below is used only after this fixed-constant
interpretation: for each normalized rectangular matrix appearing in the proof, the Bai--Yin input
gives a deterministic constant \(C_0\) such that the normalized operator norm is at most \(C_0\)
with probability tending to one.  For the fixed-direction vector contractions, the same
fixed-constant conclusion follows from the law of large numbers applied to \(N^{-1}\sum_jY_j^2\),
using independence and the uniform fourth-moment bound.  Since \(d\) is fixed, only finitely many
mode/subset choices occur, so the constants can be enlarged once and used simultaneously.

Fix a collection of modes \(R\subseteq[d]\), and contract all modes in
\([d]\setminus R\) against deterministic unit vectors.  A generic remaining entry
has the form
\[
Y_{\bm j_R}
=
\sum_{\bm j_{R^c}}
\left(\prod_{\ell\in R^c} x^{(\ell)}_{j_\ell}\right)
W_{\bm j_R,\bm j_{R^c}},
\]
where \(\bm j_R\) denotes the indices in the uncontracted modes.  For two distinct
multi-indices \(\bm j_R\neq \bm j'_R\), the two sums use disjoint sets of entries of
\(\bm W\).  Hence the variables \(Y_{\bm j_R}\) are independent over
\(\bm j_R\).  They are centered, and
\[
\mathbb E Y_{\bm j_R}^2
=
\sum_{\bm j_{R^c}}
\prod_{\ell\in R^c} \left(x^{(\ell)}_{j_\ell}\right)^2
=
\prod_{\ell\in R^c}\|\bm x^{(\ell)}\|^2
=
1.
\]
Moreover, by the independence of the entries of \(\bm W\) and the uniform fourth
moment assumption,
\[
\mathbb E Y_{\bm j_R}^4
\le C
\]
uniformly in \(N\).  Indeed, expanding the fourth moment gives a diagonal
fourth-moment contribution bounded by
\[
C\sum_{\bm j_{R^c}}
\prod_{\ell\in R^c}\left(x^{(\ell)}_{j_\ell}\right)^4
\le C,
\]
and a paired second-moment contribution bounded by
\[
C\left(
\sum_{\bm j_{R^c}}
\prod_{\ell\in R^c}\left(x^{(\ell)}_{j_\ell}\right)^2
\right)^2
=C.
\]
Thus every such partially contracted tensor has independent centered entries,
unit variance, and uniformly bounded fourth moments.

We now verify the three parts of Definition~\ref{def:raw-event-d}.

First consider \(S=\varnothing\).  For a fixed mode \(i\),
\[
\mathfrak W_{i,\varnothing}
=
N^{-1/2}\bm W
(\bm x^{(1)},\ldots,\bm x^{(i-1)},\cdot,
 \bm x^{(i+1)},\ldots,\bm x^{(d)})
\in\mathbb R^{n_i}.
\]
By the calculation above with \(R=\{i\}\), the unnormalized vector entries are
independent, centered, have variance one, and have uniformly bounded fourth
moments.  Therefore, writing these unnormalized entries as \(Y_j\),
\[
 \|\mathfrak W_{i,\varnothing}\|^2
=
N^{-1}\sum_{j=1}^{n_i}Y_j^2
=
\frac{n_i}{N}+o_{\mathbb P}(1),
\]
by the law of large numbers.  Hence, since \(n_i\asymp N\), there is a fixed constant \(L_0\) such
that
\[
\mathbb P\{\|\mathfrak W_{i,\varnothing}\|\le L_0\}\to1.
\]
Since
\[
|\eta_{i,\varnothing}|
\le
\|\mathfrak W_{i,\varnothing}\|,
\qquad
\|G_{i,\varnothing}\|
=
\|\bm\Pi_i\mathfrak W_{i,\varnothing}\|
\le
\|\mathfrak W_{i,\varnothing}\|,
\]
we obtain
\[
|\eta_{i,\varnothing}|+\|G_{i,\varnothing}\|
=
O_{\mathbb P}(1).
\]
Similarly,
\[
\omega_\varnothing
=
N^{-1/2}\bm W(\bm x^{(1)},\ldots,\bm x^{(d)})
\]
is centered with variance \(N^{-1}\), and hence
\[
|\omega_\varnothing|=O_{\mathbb P}(N^{-1/2}),
\qquad\text{in particular }|\omega_\varnothing|=O_{\mathbb P}(1).
\]

Next consider \(|S|=1\).  Write \(S=\{k\}\).  For \(k\neq i\), the map
\[
\mathfrak W_{i,\{k\}}:\mathbb R^{n_k}\to\mathbb R^{n_i}
\]
has matrix representation
\[
\mathfrak W_{i,\{k\}}(\bm h_k)
=
N^{-1/2}\bm A_{i,k}\bm h_k,
\]
where \(\bm A_{i,k}\in\mathbb R^{n_i\times n_k}\) is obtained by contracting all
modes outside \(\{i,k\}\) against the planted unit vectors.  By the preliminary
calculation, the entries of \(\bm A_{i,k}\) are independent, centered, have
variance one, and have uniformly bounded fourth moments.  Therefore
\[
\|\mathfrak W_{i,\{k\}}\|_{\op}
=
N^{-1/2}\|\bm A_{i,k}\|_{\op}
=
O_{\mathbb P}
\left(
\frac{\sqrt{n_i}+\sqrt{n_k}}{\sqrt N}
\right)
=
O_{\mathbb P}(1),
\]
because \(n_i,n_k\asymp N\).  Consequently,
\[
\|\eta_{i,\{k\}}\|_{\op}
\le
\|\mathfrak W_{i,\{k\}}\|_{\op}
=
O_{\mathbb P}(1),
\qquad
\|G_{i,\{k\}}\|_{\op}
\le
\|\mathfrak W_{i,\{k\}}\|_{\op}
=
O_{\mathbb P}(1).
\]
For the full scalar contraction
\[
\omega_{\{k\}}(\bm h_k)
=
N^{-1/2}\bm W(\bm x^{(1)},\ldots,\bm h_k,\ldots,\bm x^{(d)}),
\]
the representing vector has independent centered variance-one entries before the
outside normalization \(N^{-1/2}\).  Thus
\[
\|\omega_{\{k\}}\|_{\op}
=
O_{\mathbb P}\left(\frac{\sqrt{n_k}}{\sqrt N}\right)
=
O_{\mathbb P}(1).
\]
This proves the \(|S|=1\) part.

Now consider \(2\le |S|=q\le d-1\) for the vector-valued contractions.  Fix
\(i\) and \(S\subseteq I_i\).  Contract all modes outside \(\{i\}\cup S\) against
their planted unit vectors.  The resulting order-\((q+1)\) tensor has modes
\(\{i\}\cup S\), and its entries are independent, centered, variance-one, and have
uniformly bounded fourth moments before the outside normalization \(N^{-1/2}\).

Matricize this partially contracted tensor as an
\[
n_i\times \prod_{j\in S} n_j
\]
matrix, denoted by \(\bm A_{i,S}\).  For unit vectors
\((\bm h_j)_{j\in S}\), the vectorized tensor product
\[
\bigotimes_{j\in S}\bm h_j
\]
has Euclidean norm one.  Therefore the multilinear operator norm is bounded by
the matrix operator norm of this matricization:
\[
\|\mathfrak W_{i,S}\|_{\op}
\le
N^{-1/2}\|\bm A_{i,S}\|_{\op}.
\]
By the rectangular matrix estimate,
\[
\|\bm A_{i,S}\|_{\op}
=
O_{\mathbb P}
\left(
\sqrt{n_i}
+
\sqrt{\prod_{j\in S}n_j}
\right).
\]
Since \(n_i\asymp N\) and \(\prod_{j\in S}n_j\asymp N^q\), we obtain
\[
\|\mathfrak W_{i,S}\|_{\op}
=
O_{\mathbb P}
\left(
N^{-1/2}(N^{1/2}+N^{q/2})
\right)
=
O_{\mathbb P}\left(N^{(q-1)/2}\right).
\]
Because \(q\le d-1\),
\[
N^{(q-1)/2}\le N^{(d-2)/2}.
\]
Thus
\[
\|\mathfrak W_{i,S}\|_{\op}
=
O_{\mathbb P}\left(N^{(d-2)/2}\right).
\]
Since
\[
\|\eta_{i,S}\|_{\op}\le \|\mathfrak W_{i,S}\|_{\op},
\qquad
\|G_{i,S}\|_{\op}\le \|\mathfrak W_{i,S}\|_{\op},
\]
we get
\[
\|\eta_{i,S}\|_{\op}+\|G_{i,S}\|_{\op}
=
O_{\mathbb P}\left(N^{(d-2)/2}\right).
\]

It remains to control the full scalar contractions \(\omega_S\) for
\(2\le |S|=q\le d\).  Contract all modes outside \(S\) against planted unit
vectors.  The remaining \(q\)-linear form has independent centered variance-one
entries before the normalization \(N^{-1/2}\).  Matricize it by selecting one mode
\(j_0\in S\) as the row index and grouping all remaining modes \(S\setminus\{j_0\}\)
as the column index.  This gives a matrix of dimensions
\[
n_{j_0}\times \prod_{j\in S\setminus\{j_0\}} n_j.
\]
The multilinear operator norm of \(\omega_S\) is bounded by the operator norm of
this matricization.  Hence
\[
\|\omega_S\|_{\op}
\le
N^{-1/2}
O_{\mathbb P}
\left(
\sqrt{n_{j_0}}
+
\sqrt{\prod_{j\in S\setminus\{j_0\}} n_j}
\right).
\]
Using \(n_j\asymp N\), this gives
\[
\|\omega_S\|_{\op}
=
O_{\mathbb P}
\left(
N^{-1/2}(N^{1/2}+N^{(q-1)/2})
\right).
\]
For \(q\ge2\),
\[
N^{-1/2}(N^{1/2}+N^{(q-1)/2})
=
O\left(N^{(q-2)/2}\right)+O(1)
=
O\left(N^{(q-2)/2}\right)
\]
with the convention that the \(q=2\) case is \(O(1)\).  Since \(q\le d\), we have
\[
\|\omega_S\|_{\op}
=
O_{\mathbb P}\left(N^{(d-2)/2}\right).
\]

We have shown that all \(|S|=0\) and \(|S|=1\) quantities are \(O_{\mathbb P}(1)\),
and all quantities with \(|S|\ge2\) are
\[
O_{\mathbb P}\left(N^{(d-2)/2}\right).
\]
Because \(d\) is fixed, there are only finitely many choices of \(i\) and \(S\).
Therefore, by a finite union bound, we may choose constants \(L,C<\infty\),
depending only on \(d\), the aspect-ratio bounds, and the uniform fourth-moment
bound, such that, with
\[
\Theta=C N^{(d-2)/2},
\]
all inequalities in Definition~\ref{def:raw-event-d} hold with probability tending
to one.  Equivalently,
\[
\mathbb P\{\mathcal F_{d,N}(L,\Theta)\}\to1.
\]
This proves the proposition.
\end{proof}

\subsection{Proof of Proposition~\ref{prop:exact-subset-expansion-d} (Exact subset expansion)}\label{app:proof-03}
\begin{proof}
Expand $\bm{u}^{(j)}=\alpha_j\bm{x}^{(j)}+\bm{a}_j$ in every mode $j\ne i$ inside
$N^{-1/2}\bm{W}(U^{(-i)})$.  By multilinearity, one obtains one term for each subset
$S\subseteq I_i$: indices in $S$ contribute perturbations $\bm{a}_j$, while indices in
$I_i\setminus S$ contribute planted vectors $\bm{x}^{(j)}$ and the scalar prefactor
$A_{i,S}(U)$.  Projection onto $\bm{x}^{(i)}$ gives \eqref{eq:zeta-exact-d}, projection onto
$(\bm{x}^{(i)})^\perp$ gives \eqref{eq:g-exact-d}, and adding the deterministic signal gives
\eqref{eq:mode-decomp-exact-d}.  The scalar expansion \eqref{eq:omega-exact-d} is the
same argument over all $d$ modes.
\end{proof}

\subsection{Proof of Lemma~\ref{lem:subset-bounds-d} (One-point and two-point subset bounds)}\label{app:proof-04}
\begin{proof}
We prove the bounds for $\bm{g}_i$; the scalar proofs are identical.  Let
$\delta:=\delta(U)$.  From \eqref{eq:g-exact-d}, since $|A_{i,S}(U)|\le1$,
\[
\|\bm{g}_i(U)\|
\le
\sum_{S\subseteq I_i}\|G_{i,S}(\bm{a}_S)\|.
\]
The term $S=\varnothing$ is bounded by $L$.  The terms $|S|=1$ are bounded by $L\delta$.
If $|S|=q\ge2$, then
\[
\|G_{i,S}(\bm{a}_S)\|
\le
\Theta\prod_{j\in S}\|\bm{a}_j\|
\le
\Theta\delta^q
\le
\Theta\delta^2,
\]
because $\delta\le\rho_d\le1$.  Summing over finitely many subsets gives
\eqref{eq:zeta-g-one-bound-d} for $\bm{g}_i$.

For the two-point estimate, write $D:=d_{\loc}(U,\widetilde U)$ and
$\widetilde\delta:=\delta(\widetilde U)$.  For each $S$,
\begin{align*}
&A_{i,S}(U)G_{i,S}(\bm{a}_S)-A_{i,S}(\widetilde U)G_{i,S}(\widetilde{\bm{a}}_S)
\\
&\quad=
\bigl(A_{i,S}(U)-A_{i,S}(\widetilde U)\bigr)G_{i,S}(\bm{a}_S)
+
A_{i,S}(\widetilde U)
\bigl(G_{i,S}(\bm{a}_S)-G_{i,S}(\widetilde{\bm{a}}_S)\bigr).
\end{align*}
By Lemma~\ref{lem:alpha-product-d},
\[
|A_{i,S}(U)-A_{i,S}(\widetilde U)|
\le
C_d(\delta+\widetilde\delta)D.
\]
For $S=\varnothing$, the second term is zero and the first term is at most $C_dLD$.
For $|S|=1$, the first term is at most
$C_dL(\delta+\widetilde\delta)\delta D\le C_dLD$, while the second term is at most
$LD$.

Now let $q=|S|\ge2$.  The first term is bounded by
\[
C_d(\delta+\widetilde\delta)D\cdot \Theta\delta^q
\le
C_d\Theta(\delta+\widetilde\delta)D.
\]
For the second term, use the exact telescoping identity for a $q$-linear map $B$:
\begin{align*}
&B(\bm{a}_{j_1},\ldots,\bm{a}_{j_q})
-
B(\widetilde{\bm{a}}_{j_1},\ldots,\widetilde{\bm{a}}_{j_q})
\\
&\quad=
\sum_{\ell=1}^q
B(\widetilde{\bm{a}}_{j_1},\ldots,\widetilde{\bm{a}}_{j_{\ell-1}},
\bm{a}_{j_\ell}-\widetilde{\bm{a}}_{j_\ell},
\bm{a}_{j_{\ell+1}},\ldots,\bm{a}_{j_q}).
\end{align*}
Therefore
\[
\|G_{i,S}(\bm{a}_S)-G_{i,S}(\widetilde{\bm{a}}_S)\|
\le
C_d\Theta(\delta+\widetilde\delta)^{q-1}D
\le
C_d\Theta(\delta+\widetilde\delta)D,
\]
after decreasing $\rho_d$ if needed.  Summing over all subsets proves
\eqref{eq:zeta-g-two-bound-d}.
\end{proof}

\subsection{Proof of Lemma~\ref{lem:norm-pert-d} (Order-\texorpdfstring{$d$}{d} normalization estimate)}\label{app:proof-norm-pert-d}
\begin{proof}
For the one-point estimate, write
\[
\Phi_{\bm x}(s,\bm g)
=
\frac{s}{(s^2+\|\bm g\|^2)^{1/2}}\bm x
+
\frac{\bm g}{(s^2+\|\bm g\|^2)^{1/2}}.
\]
Since \((s^2+\|\bm g\|^2)^{1/2}\ge s\), the orthogonal component has norm at most
\(\|\bm g\|/s\), and the stated factor \(2\) is a convenient slack form.

For the two-point estimate, let \(m=s\wedge\widetilde s\) and consider
\(\psi(t,h)=(t\bm x+h)/\|t\bm x+h\|\) on the set \(t\ge m\) and \(\|h\|\le m/2\).  The
derivative in the \(h\)-direction has operator norm at most \(C/m\).  The derivative in the
scalar direction is
\[
\partial_t\psi(t,h)
=
\frac{\bm x}{\|t\bm x+h\|}
-
\frac{(t\bm x+h)t}{\|t\bm x+h\|^3},
\]
whose norm is bounded by \(C\|h\|/m^2\).  Integrating these derivative bounds first along the
\(h\)-segment and then along the \(t\)-segment gives the displayed Lipschitz estimate.
\end{proof}

\subsection[Proof of Theorem~\ref{thm:local-recursion-d} (Order-d local one-step recursion)]{Proof of Theorem~\ref{thm:local-recursion-d} (Order-$d$ local one-step recursion)}\label{app:proof-05}
\begin{proof}
Fix $i$.  By Proposition~\ref{prop:exact-subset-expansion-d},
\[
\bm{T}(U_t^{(-i)})
=
s_i(U_t)\bm{x}^{(i)}+\bm{g}_i(U_t),
\qquad
s_i(U_t)=\beta\prod_{j\ne i}\alpha_{j,t}+\zeta_i(U_t).
\]
Since $r/\beta\le\rho_d$, Lemma~\ref{lem:alpha-product-d} gives
\[
\prod_{j\ne i}\alpha_{j,t}\ge3/4.
\]
By Lemma~\ref{lem:subset-bounds-d},
\[
|\zeta_i(U_t)|+\|\bm{g}_i(U_t)\|
\le
C_dL\left(1+\frac r\beta\right)
+
C_d\Theta\frac{r^2}{\beta^2}
=
G_{r,d}(\beta).
\]
Thus \eqref{eq:signal-dominance-d} implies
\[
s_i(U_t)\ge\beta/2,
\qquad
\|\bm{g}_i(U_t)\|\le\beta/4\le s_i(U_t)/2.
\]
Lemma~\ref{lem:norm-pert-d} gives
\[
\|\bm{a}_{i,t+1}\|
\le
\frac{4}{\beta}\|\bm{g}_i(U_t)\|.
\]
Using the sharper bound
\[
\|\bm{g}_i(U_t)\|
\le
C_dL(1+\delta_t)+C_d\Theta\delta_t^2,
\qquad
\delta_t:=\delta(U_t),
\]
and $\delta_t^2\le(r/\beta)\delta_t$, we get
\[
\|\bm{a}_{i,t+1}\|
\le
\frac{C_dL}{\beta}
+
\left(
\frac{C_dL}{\beta}
+
\frac{C_d\Theta r}{\beta^2}
\right)\delta_t.
\]
Taking the maximum over $i$ proves \eqref{eq:main-recursion-d}.
\end{proof}

\subsection{Proof of Proposition~\ref{prop:explicit-affine-d} (Finite-iteration affine error bound)}\label{app:proof-06}
\begin{proof}
This is the elementary solution of an affine scalar recursion.  Iterating
\[
\delta(U_{t+1})\le A_{\beta,d}+q_{\beta,d}\delta(U_t)
\]
gives
\[
\delta(U_t)
\le
A_{\beta,d}\sum_{j=0}^{t-s-1}q_{\beta,d}^j
+
q_{\beta,d}^{t-s}\delta(U_s).
\]
The finite geometric sum yields \eqref{eq:delta-closed-form-d}.  Subtracting the fixed
point of the scalar affine recursion gives \eqref{eq:delta-floor-split-d}.  Finally,
\eqref{eq:hitting-time-floor-d} is exactly the inequality
\(q_{\beta,d}^{t-s}\delta(U_s)\le\varepsilon/\beta\) rewritten using
\(0<q_{\beta,d}<1\).
\end{proof}

\subsection[Proof of Proposition~\ref{prop:two-point-d} (Order-d two-point contraction)]{Proof of Proposition~\ref{prop:two-point-d} (Order-$d$ two-point contraction)}\label{app:proof-07}
\begin{proof}
Fix $i$ and write
\[
\bm{T}(U^{(-i)})=s_i(U)\bm{x}^{(i)}+\bm{g}_i(U),
\qquad
\bm{T}(\widetilde U^{(-i)})=s_i(\widetilde U)\bm{x}^{(i)}+\bm{g}_i(\widetilde U).
\]
Let $D:=d_{\loc}(U,\widetilde U)$.  Lemma~\ref{lem:subset-bounds-d} gives
\[
\|\bm{g}_i(U)-\bm{g}_i(\widetilde U)\|
\le
\left(C_dL+C_d\Theta\frac r\beta\right)D
=
D_{r,d}(\beta)D.
\]
For the scalar coefficient,
\[
s_i(U)-s_i(\widetilde U)
=
\beta
\left(
\prod_{j\ne i}\alpha_j-\prod_{j\ne i}\widetilde\alpha_j
\right)
+
\zeta_i(U)-\zeta_i(\widetilde U).
\]
Using Lemma~\ref{lem:alpha-product-d} and Lemma~\ref{lem:subset-bounds-d},
\[
|s_i(U)-s_i(\widetilde U)|
\le
\left(C_dr+C_dL+C_d\Theta\frac r\beta\right)D
=
S_{r,d}(\beta)D.
\]
The dominance condition gives
\[
s_i(U),s_i(\widetilde U)\ge\beta/2,
\qquad
\|\bm{g}_i(U)\|+\|\bm{g}_i(\widetilde U)\|\le2G_{r,d}(\beta).
\]
Applying Lemma~\ref{lem:norm-pert-d},
\[
\|\bm{u}_i^+-\widetilde{\bm{u}}_i^+\|
\le
C_d\frac{D_{r,d}(\beta)}{\beta}D
+
C_d\frac{G_{r,d}(\beta)S_{r,d}(\beta)}{\beta^2}D.
\]
Since $\bm{a}_i^+=\bm{\Pi}_i \bm{u}_i^+$, the same bound holds for the local coordinates.  Taking the
maximum over $i$ proves \eqref{eq:two-point-d}.  For the simplified estimate, expand
\[
\frac{D_{r,d}(\beta)}{\beta}
\le
\frac{C_dL}{\beta}+\frac{C_d\Theta r}{\beta^2}.
\]
The second contribution is
\[
\frac{G_{r,d}(\beta)S_{r,d}(\beta)}{\beta^2},
\]
whose products are bounded by finite sums of
\[
\frac{Lr}{\beta^2},\qquad
\frac{L^2}{\beta^2},\qquad
\frac{L\Theta r}{\beta^3},\qquad
\frac{Lr^2}{\beta^3},\qquad
\frac{L^2r}{\beta^3},\qquad
\frac{L\Theta r^2}{\beta^4},\qquad
\frac{\Theta r^3}{\beta^4},\qquad
\frac{\Theta^2r^3}{\beta^5}.
\]
Under \(\beta\ge C_d(1+L)\), \(r/\beta\le c_d\), and \(G_{r,d}(\beta)\le\beta/4\), after increasing
the constant \(C_d\) and decreasing \(c_d\) if necessary, these terms are dominated by
\[
\frac{C_dL}{\beta}+\frac{C_d\Theta r}{\beta^2}.
\]
This proves \eqref{eq:kappa-simplified-d}.
\end{proof}

\subsection[Proof of Corollary~\ref{cor:fixed-point-d} (Order-d fixed point and convergence)]{Proof of Corollary~\ref{cor:fixed-point-d} (Order-$d$ fixed point and convergence)}\label{app:proof-08}
\begin{proof}
The self-map property follows from Theorem~\ref{thm:local-recursion-d} and
\eqref{eq:self-map-d}.  The contraction property follows from
Proposition~\ref{prop:two-point-d}.  Banach's fixed-point theorem gives the fixed point and
the geometric convergence estimate \eqref{eq:fixed-distance-d}.  Solving
\((\kappa_{r,d}^{\ctr}(\beta))^{t-s}d_{\loc}(U_s,U_\star)\le\varepsilon\) for
\(t-s\) gives \eqref{eq:fixed-distance-hitting-d}.  At the fixed point, the affine recursion gives
\[
\delta(U_\star)
\le
\frac{C_dL}{\beta}
+
\kappa_{r,d}^{\aff}(\beta)\delta(U_\star).
\]
Rearranging proves \eqref{eq:fixed-error-d}.
\end{proof}

\subsection[Proof of Corollary~\ref{cor:kkt-d} (Order-d KKT equations and singular value)]{Proof of Corollary~\ref{cor:kkt-d} (Order-$d$ KKT equations and singular value)}\label{app:proof-09}
\begin{proof}
At a fixed point,
\[
\bm{u}_\star^{(i)}
=
\frac{\bm{T}(U_\star^{(-i)})}{\|\bm{T}(U_\star^{(-i)})\|}.
\]
Thus $\bm{T}(U_\star^{(-i)})=\lambda_i \bm{u}_\star^{(i)}$ for some $\lambda_i>0$.
Contracting with $\bm{u}_\star^{(i)}$ gives
\[
\lambda_i=\bm{T}(\bm{u}_\star^{(1)},\ldots,\bm{u}_\star^{(d)})=\lambda_\star.
\]
This proves \eqref{eq:kkt-d}.

For the singular value,
\[
\lambda_\star
=
\beta\prod_{i=1}^d\alpha_{i,\star}
+
\Omega(U_\star).
\]
Since
\[
\left|
\prod_{i=1}^d\alpha_{i,\star}-1
\right|
\le
C_d\delta(U_\star)^2
\]
and Lemma~\ref{lem:subset-bounds-d} gives
\[
|\Omega(U_\star)|
\le
C_dL(1+\delta(U_\star))+C_d\Theta\delta(U_\star)^2,
\]
we obtain \eqref{eq:lambda-bound-raw-d}.  Substituting
\eqref{eq:fixed-error-d} yields \eqref{eq:lambda-bound-d}.
\end{proof}

\section{Proofs for Section~\ref{sec:order-three-model}: order-three worked example}\label{app:proofs-order-three}

This appendix contains the proofs for the order-three specialization in Section~\ref{sec:order-three-model}.  These proofs parallel the fixed-order arguments but use the explicit bilinear matrix-slice notation of the worked example.

\subsection{Proof of Lemma~\ref{lem:coord-equivalence-3} (Coordinate, angle, and Euclidean equivalences)}\label{app:proof-14}
\begin{proof}
Since $\bm{u}$ has unit norm and $\bm{a}\perp \bm{x}$,
\[
1=\|\bm{u}\|^2=\alpha_x^2+\|\bm{a}\|^2.
\]
Thus $\alpha_x=\langle \bm{u},\bm{x}\rangle$ and
\[
\sin^2\angle(\bm{u},\bm{x})=1-\langle \bm{u},\bm{x}\rangle^2=1-\alpha_x^2=\|\bm{a}\|^2.
\]
This proves \eqref{eq:sin-exact} and \eqref{eq:one-minus-corr}.  Moreover,
\[
\bm{u}-\bm{x}=(\alpha_x-1)\bm{x}+\bm{a},
\]
and the two summands are orthogonal, which gives \eqref{eq:euc-coord-exact}.  If
$\|\bm{a}\|\le1/2$, then
\[
1-\alpha_x
=
1-\sqrt{1-\|\bm{a}\|^2}
=
\frac{\|\bm{a}\|^2}{1+\sqrt{1-\|\bm{a}\|^2}}
\le \|\bm{a}\|^2.
\]
Therefore
\[
\|\bm{u}-\bm{x}\|^2
\le
\|\bm{a}\|^4+\|\bm{a}\|^2
\le
2\|\bm{a}\|^2,
\]
and hence $\|\bm{u}-\bm{x}\|\le2\|\bm{a}\|$.  The lower bound $\|\bm{u}-\bm{x}\|\ge\|\bm{a}\|$ follows immediately
from \eqref{eq:euc-coord-exact}.  Applying the same argument in the three modes proves
the remaining assertions.
\end{proof}

\subsection{Proof of Lemma~\ref{lem:rectangular-op-input} (Rectangular operator-norm input)}\label{app:proof-15}
\begin{proof}
The first statement is the rectangular operator-norm input from the
Yin--Bai--Krishnaiah and Bai--Silverstein--Yin sample-covariance theory under a
finite-fourth-moment assumption \citep{YinBaiKrishnaiah1988,BaiSilversteinYin1988}.  This is only
the high-probability rectangular operator-norm input needed in the order-three example; it is not
used as the log-free covariance moment estimate in Appendix~\ref{app:proofs-warm-starts}.  The
more specialized covariance moment bound used for the same-sample proof is proved separately in
Lemma~\ref{lem:aux-rect-cov-moment-d}.  The weaker high-probability upper bound stated here also
follows by the usual truncation-and-net argument.
For the second statement, if
$A_{ik}=\sum_j q_j \bm{W}_{ijk}$ with $\|q\|=1$, then the variables $A_{ik}$ are independent over
$(i,k)$, centered, have variance one, and satisfy
\[
\Ebb |A_{ik}|^4
\le
C\left(\sum_j q_j^2\right)^2
=C,
\]
by expanding the fourth moment and using independence, centering, and the uniform
fourth-moment bound.  Therefore the same rectangular estimate applies.
\end{proof}

\subsection[Proof of Lemma~\ref{lem:EL0-hp} (High-probability origin of EL0)]{Proof of Lemma~\ref{lem:EL0-hp} (High-probability origin of $\calE_{L_0}$)}\label{app:proof-16}
\begin{proof}
We prove the fixed, one-perturbation, and two-perturbation estimates separately.

First consider a fixed contraction, for example $\noise(\bm{y},\bm{z})$.  Write
\[
X_i:=\sum_{j,k}\bm{y}_jz_kW_{ijk},
\qquad i=1,\ldots,m.
\]
The variables $X_i$ are independent, centered, have variance one, and have uniformly
bounded fourth moments.  Therefore
\[
\left\|\frac1{\sqrt N}\bm{W}(\bm{y},\bm{z})\right\|^2
=
\frac1N\sum_{i=1}^m X_i^2.
\]
Since $m\asymp N$, the expectation is $m/N=O(1)$ and the variance is
$O(m/N^2)=O(N^{-1})$ by the fourth-moment bound.  Hence the squared norm is
$O_{\Pbb}(1)$.  Both $|\eta_{u,0}|$ and $\|\bm{\xi}_u\|$ are bounded by this norm.  The same
argument applies to the fixed contractions in the other two modes.

Next consider a one-perturbation contraction, for instance $\bm{c}\mapsto \noise(\bm{y},\bm{c})$.  Let
$A_y:=\bm{W}(\bm{y},\cdot)\in\R^{m\times p}$, with entries
\[
(A_y)_{ik}=\sum_j \bm{y}_jW_{ijk}.
\]
The entries are independent across $(i,k)$, centered, have unit variance, and have
uniformly bounded fourth moments.  By Lemma~\ref{lem:rectangular-op-input},
\[
\|A_y\|_{\op}\le C(\sqrt m+\sqrt p)
\]
with probability tending to one.  Consequently,
\[
\|\bm{c}\mapsto \noise(\bm{y},\bm{c})\|
\le
\frac1{\sqrt N}\|A_y\|_{\op}
\le
C\frac{\sqrt m+\sqrt p}{\sqrt N}
=O(1).
\]
The scalar map $\ell_u^{(y)}$ and the projected map $\bm{M}_u^{(y)}$ are dominated by this
operator norm.  The remaining one-perturbation maps are identical after permuting the
modes.

Finally consider the quadratic remainder for the $\bm{u}$-update.  Let
$\bm{W}_{(1)}\in\R^{m\times np}$ be the mode-one matricization.  Then
\[
\bm{W}(\bm{b},\bm{c})=\bm{W}_{(1)}(\bm{b}\otimes \bm{c}).
\]
Thus
\[
|\rho_u(\bm{b},\bm{c})|+\|\bm{R}_u(\bm{b},\bm{c})\|
\le
2\left\|\frac1{\sqrt N}\bm{W}_{(1)}(\bm{b}\otimes \bm{c})\right\|
\le
\frac2{\sqrt N}\|\bm{W}_{(1)}\|_{\op}\|\bm{b}\|\|\bm{c}\|.
\]
Lemma~\ref{lem:rectangular-op-input} also gives
\[
\|\bm{W}_{(1)}\|_{\op}\le C(\sqrt m+\sqrt{np})
\]
with probability tending to one.  Since $m,n,p\asymp N$,
\[
\frac{\sqrt m+\sqrt{np}}{\sqrt N}
\le
CN^{1/2}.
\]
This gives the $\bm{u}$-remainder bound.  The $\bm{v}$- and $\bm{w}$-remainder bounds follow from the
mode-two and mode-three matricizations.  Since only finitely many estimates are needed,
increasing $L_0$ if necessary gives $\Pbb(\calE_{L_0})\to1$.
\end{proof}

\subsection{Proof of Lemma~\ref{lem:norm-pert} (Normalized perturbation)}\label{app:proof-17}
\begin{proof}
Since $\bm{g}\perp \bm{x}$,
\[
\|s\bm{x}+\bm{g}\|^2=s^2+\|\bm{g}\|^2.
\]
Therefore
\[
\bm{u}^+
=
\frac{s}{\sqrt{s^2+\|\bm{g}\|^2}}\bm{x}+
\frac{g}{\sqrt{s^2+\|\bm{g}\|^2}}.
\]
Thus the orthogonal coordinate is
\[
\bm{a}=\frac{g}{\sqrt{s^2+\|\bm{g}\|^2}},
\]
and hence
\[
\|\bm{a}\|\le \frac{\|\bm{g}\|}{s}\le\frac{2\|\bm{g}\|}{s}.
\]
Also
\[
1-\langle \bm{u}^+,\bm{x}\rangle^2
=
1-\frac{s^2}{s^2+\|\bm{g}\|^2}
=
\frac{\|\bm{g}\|^2}{s^2+\|\bm{g}\|^2}
\le
\frac{\|\bm{g}\|^2}{s^2}
\le
\frac{4\|\bm{g}\|^2}{s^2}.
\]

For the two-point bound, define
\[
\Phi(s,\bm{g}):=\frac{s\bm{x}+\bm{g}}{\sqrt{s^2+\|\bm{g}\|^2}}.
\]
On the region $\|\bm{g}\|\le s/2$, one has $\sqrt{s^2+\|\bm{g}\|^2}\asymp s$.  For a tangent
perturbation $\bm{h}\in \bm{x}^\perp$,
\[
D_g\Phi(s,\bm{g})[\bm{h}]
=
\frac{\bm{h}}{\sqrt{s^2+\|\bm{g}\|^2}}
-
\frac{(s\bm{x}+\bm{g})\langle \bm{g},\bm{h}\rangle}{(s^2+\|\bm{g}\|^2)^{3/2}},
\]
so $\|D_g\Phi(s,\bm{g})\|\le C/s$.  Moreover,
\[
\partial_s\Phi(s,\bm{g})
=
\frac{\bm{x}}{\sqrt{s^2+\|\bm{g}\|^2}}
-
\frac{s(s\bm{x}+\bm{g})}{(s^2+\|\bm{g}\|^2)^{3/2}}.
\]
The component in the $\bm{x}$ direction cancels to first order, and direct simplification gives
\[
\|\partial_s\Phi(s,\bm{g})\|\le C\frac{\|\bm{g}\|}{s^2}.
\]
Let $s_0:=s\wedge\widetilde s$.  First vary $\bm{g}$ at fixed $s$, and then vary $s$ at
fixed $\widetilde{\bm{g}}$.  On the first segment the bound for $D_g\Phi$ gives
$C\|\bm{g}-\widetilde{\bm{g}}\|/s_0$.  On the second segment, using the strengthened assumption
$\|\bm{g}\|\vee\|\widetilde{\bm{g}}\|\le (s\wedge\widetilde s)/2$, which keeps the whole path inside the same regular region, the bound for $\partial_s\Phi$ gives
$C(\|\bm{g}\|+\|\widetilde{\bm{g}}\|)|s-\widetilde s|/s_0^2$ after enlarging the absolute constant.
This proves \eqref{eq:norm-pert-diff}.
\end{proof}

\subsection{Proof of Proposition~\ref{prop:local-expansion} (Local multilinear expansion with scalar noise)}\label{app:proof-18}
\begin{proof}
We prove the $\bm{u}$-identity.  The signal part is
\[
\beta\langle \bm{v},\bm{y}\rangle\langle \bm{w},\bm{z}\rangle \bm{x}
=
\beta\alpha_y\alpha_zx.
\]
For the noise part, by bilinearity,
\[
\noise(\bm{v},\bm{w})
=
\alpha_y\alpha_z\noise(\bm{y},\bm{z})
+\alpha_y\noise(\bm{y},\bm{c})
+\alpha_z\noise(\bm{b},\bm{z})
+\noise(\bm{b},\bm{c}).
\]
Taking inner product with $\bm{x}$ gives \eqref{eq:eta-u-def}.  Projecting onto $\bm{x}^\perp$
gives \eqref{eq:gu-def}.  This proves \eqref{eq:T-vw-sg}.  Since
$|\alpha_y|,|\alpha_z|\le1$, Definition~\ref{def:EL0} gives
\[
|\eta_u(\bm{b},\bm{c})|+\|\bm{g}_u(\bm{v},\bm{w})\|
\le
CL_0\{1+\|\bm{b}\|+\|\bm{c}\|+N^{1/2}\|\bm{b}\|\|\bm{c}\|\}
\le
CL_0(1+\delta+N^{1/2}\delta^2).
\]
The other modes are identical.
\end{proof}

\subsection{Proof of Lemma~\ref{lem:two-point-remainders} (Two-point bounds for scalar and vector quadratic remainders)}\label{app:proof-19}
\begin{proof}
For the first estimate, bilinearity gives
\[
\rho_u(\bm{b},\bm{c})-\rho_u(\widetilde{\bm{b}},\widetilde{\bm{c}})
=
\rho_u(\bm{b}-\widetilde{\bm{b}},\bm{c})
+
\rho_u(\widetilde{\bm{b}},\bm{c}-\widetilde{\bm{c}}),
\]
and the same identity holds for $\bm{R}_u$.  Applying \eqref{eq:EL0-Ru-one} to the two terms
gives \eqref{eq:two-point-u}.  The other two estimates are identical.
\end{proof}

\subsection{Proof of Theorem~\ref{thm:local-recursion-3} (Local one-step recursion)}\label{app:proof-20}
\begin{proof}
We prove the $\bm{u}$-update.  The other two modes are identical.  By
Proposition~\ref{prop:local-expansion},
\[
\bm{T}(\bm{v}_t,\bm{w}_t)=s_{u,t}\bm{x}+\bm{g}_{u,t},
\qquad
s_{u,t}=\beta\alpha_{y,t}\alpha_{z,t}+\eta_u(\bm{b}_t,\bm{c}_t).
\]
By \eqref{eq:local-radius-3}, after decreasing $\rho_3$ if necessary, $\alpha_{y,t}\alpha_{z,t}\ge3/4$.  Also, by
\eqref{eq:eta-g-onepoint} and $(\bm{a}_t,\bm{b}_t,\bm{c}_t)\in B_r^{\mathrm{coord}}$,
\[
|\eta_u(\bm{b}_t,\bm{c}_t)|+\|\bm{g}_{u,t}\|\le \Gamma_r(\beta).
\]
The dominance condition \eqref{eq:signal-dominance-3} implies
\[
s_{u,t}\ge \frac34\beta-\frac14\beta=\frac\beta2,
\qquad
\|\bm{g}_{u,t}\|\le\frac\beta4\le\frac{s_{u,t}}2.
\]
Thus Lemma~\ref{lem:norm-pert} applies:
\[
\|\bm{a}_{t+1}\|
\le
\frac{2\|\bm{g}_{u,t}\|}{s_{u,t}}
\le
\frac{4\|\bm{g}_{u,t}\|}{\beta}.
\]
Using the sharper bound
\[
\|\bm{g}_{u,t}\|
\le
CL_0(1+\delta_t+N^{1/2}\delta_t^2)
\]
and the basin inequality $\delta_t^2\le(r/\beta)\delta_t$, we get
\[
\|\bm{a}_{t+1}\|
\le
\frac{C_0}{\beta}
+
\left(
\frac{C_1}{\beta}
+
\frac{C_2rN^{1/2}}{\beta^2}
\right)\delta_t.
\]
Taking the maximum over the three modes proves \eqref{eq:local-recursion-3}.
\end{proof}

\subsection{Proof of Proposition~\ref{prop:explicit-affine-3} (Explicit solution of the affine recursion)}\label{app:proof-21}
\begin{proof}
Iterating \eqref{eq:local-recursion-3} gives
\begin{align*}
\delta_{s+1}
&\le A_\beta+q_\beta\delta_s,\\
\delta_{s+2}
&\le A_\beta+q_\beta\delta_{s+1}
\le A_\beta(1+q_\beta)+q_\beta^2\delta_s,\\
\delta_{s+3}
&\le A_\beta+q_\beta\delta_{s+2}
\le A_\beta(1+q_\beta+q_\beta^2)+q_\beta^3\delta_s.
\end{align*}
By induction, for $n=t-s$,
\[
\delta_t
\le
A_\beta\sum_{j=0}^{n-1}q_\beta^j
+
q_\beta^n\delta_s.
\]
The finite geometric sum is
\[
\sum_{j=0}^{n-1}q_\beta^j
=
\frac{1-q_\beta^n}{1-q_\beta}.
\]
This proves \eqref{eq:delta-closed-form-3}.  The rearranged form
\eqref{eq:delta-floor-split-3} and the limsup bound \eqref{eq:delta-limsup-3} follow
immediately.  Finally, solving
\[
q_\beta^{t-s}\delta_s\le\varepsilon/\beta
\]
for $t-s$ gives \eqref{eq:hitting-time-floor-3}.
\end{proof}

\subsection{Proof of Proposition~\ref{prop:two-point-contraction} (Two-point contraction in the local basin)}\label{app:proof-22}
\begin{proof}
We again treat only the $\bm{u}$-update.  Write
\[
\bm{T}(\bm{v},\bm{w})=s_ux+\bm{g}_u,
\qquad
\bm{T}(\widetilde{\bm{v}},\widetilde{\bm{w}})=\widetilde s_u\bm{x}+\widetilde{\bm{g}}_u.
\]
Let $D:=d(U,\widetilde U)$.  Since
\[
\alpha_y=\sqrt{1-\|\bm{b}\|^2},
\]
the derivative of $\bm{b}\mapsto\sqrt{1-\|\bm{b}\|^2}$ has norm at most $C\|\bm{b}\|$ in the local
basin.  Hence
\[
|\alpha_y-\widetilde\alpha_y|
\le
C\frac r\beta \|\bm{b}-\widetilde{\bm{b}}\|,
\qquad
|\alpha_z-\widetilde\alpha_z|
\le
C\frac r\beta \|\bm{c}-\widetilde{\bm{c}}\|.
\]
Therefore
\[
|\alpha_y\alpha_z-\widetilde\alpha_y\widetilde\alpha_z|
\le
C\frac r\beta D.
\]
Using \eqref{eq:EL0-linear} and Lemma~\ref{lem:two-point-remainders}, we get
\[
|\eta_u(\bm{b},\bm{c})-\eta_u(\widetilde{\bm{b}},\widetilde{\bm{c}})|
+
\|\bm{g}_u-\widetilde{\bm{g}}_u\|
\le
D_r(\beta)D.
\]
Thus
\[
|s_u-\widetilde s_u|
\le
\beta|\alpha_y\alpha_z-\widetilde\alpha_y\widetilde\alpha_z|
+
|\eta_u(\bm{b},\bm{c})-\eta_u(\widetilde{\bm{b}},\widetilde{\bm{c}})|
\le
S_r(\beta)D.
\]
The dominance condition gives
\[
s_u,\widetilde s_u\ge\beta/2,
\qquad
\|\bm{g}_u\|+\|\widetilde{\bm{g}}_u\|\le2\Gamma_r(\beta).
\]
Applying Lemma~\ref{lem:norm-pert} gives
\[
\|\bm{u}^+-\widetilde{\bm{u}}^+\|
\le
C\left(
\frac{D_r(\beta)}{\beta}
+
\frac{\Gamma_r(\beta)S_r(\beta)}{\beta^2}
\right)D.
\]
Since the new local coordinate is obtained by projecting $\bm{u}^+$ onto $\bm{x}^\perp$, the same
bound holds for $\|\bm{a}^+-\widetilde{\bm{a}}^+\|$.  The other two modes are identical.  Taking the maximum proves
\eqref{eq:two-point-contraction}.
\end{proof}

\subsection{Proof of Corollary~\ref{cor:local-fixed-point} (Unique local fixed point and explicit convergence)}\label{app:proof-23}
\begin{proof}
The self-map property follows from Theorem~\ref{thm:local-recursion-3} and
\eqref{eq:self-map-condition-3}.  Proposition~\ref{prop:two-point-contraction} shows that
$\calA$ is a strict contraction on $\calB_r$.  Banach's fixed-point theorem gives a unique
fixed point $U_\star\in\calB_r$.

At the fixed point, applying \eqref{eq:local-recursion-3} gives
\[
\delta_\star
\le
\frac{C_0}{\beta}
+
\kappa_r^{\aff}(\beta)\delta_\star.
\]
Rearranging yields \eqref{eq:fixed-coordinate-error-3}.  If
$\kappa_r^{\aff}(\beta)$ is bounded away from one, then
$\delta_\star=O(1/\beta)$, and Lemma~\ref{lem:coord-equivalence-3} gives
\eqref{eq:fixed-euclidean-error-3}.  Finally,
\eqref{eq:fixed-point-convergence-3} follows by iterating the two-point contraction with
$\widetilde U=U_\star$.  The triangle inequality gives
\[
\delta_t
\le
d(U_t,U_\star)+\delta_\star,
\]
which proves \eqref{eq:coordinate-error-to-spike-3}.
\end{proof}

\subsection{Proof of Corollary~\ref{cor:lambda-3} (Singular value control)}\label{app:proof-24}
\begin{proof}
Write
\[
\lambda_\star
=
\beta\alpha_{x,\star}\alpha_{y,\star}\alpha_{z,\star}
+
\noise(\bm{u}_\star,\bm{v}_\star,\bm{w}_\star).
\]
Since
\[
1-\alpha_{x,\star}\alpha_{y,\star}\alpha_{z,\star}
=
O(\delta_\star^2),
\]
the signal part differs from $\beta$ by at most $C\beta\delta_\star^2$.  For the noise part,
use the $\bm{u}$-mode decomposition:
\[
\noise(\bm{u}_\star,\bm{v}_\star,\bm{w}_\star)
=
\alpha_{x,\star}\eta_u(\bm{b}_\star,\bm{c}_\star)
+
\langle \bm{a}_\star,\bm{g}_u(\bm{v}_\star,\bm{w}_\star)\rangle.
\]
By Proposition~\ref{prop:local-expansion},
\[
|\eta_u(\bm{b}_\star,\bm{c}_\star)|+\|\bm{g}_u(\bm{v}_\star,\bm{w}_\star)\|
\le
CL_0(1+\delta_\star+N^{1/2}\delta_\star^2).
\]
This proves \eqref{eq:lambda-bound-3}.  Under
$\delta_\star=O(1/\beta)$ and the high-signal regime \eqref{eq:high-signal-3}, the right-hand side is bounded by
a constant eventually.
\end{proof}

\subsection{Proof of Corollary~\ref{cor:kkt-3} (Local KKT point)}\label{app:proof-25}
\begin{proof}
At a fixed point of the normalized alternating map, each contraction is parallel to the
corresponding vector:
\[
\bm{T}(\bm{v}_\star,\bm{w}_\star)=\lambda_u \bm{u}_\star,\quad
\bm{T}(\bm{u}_\star,\bm{w}_\star)=\lambda_v \bm{v}_\star,\quad
\bm{T}(\bm{v}_\star)^\top \bm{u}_\star=\lambda_w \bm{w}_\star.
\]
Contracting each equation with the corresponding unit vector gives
\[
\lambda_u=\lambda_v=\lambda_w=\bm{T}(\bm{u}_\star,\bm{v}_\star,\bm{w}_\star)=\lambda_\star.
\]
\end{proof}

\section{Proofs for Section~\ref{sec:warm-starts-final}: warm starts and centered-Gram}\label{app:proofs-warm-starts}

\subsection{Proof of Theorem~\ref{thm:warm-start-d} (One-sweep entry from a general correlated initializer)}\label{app:proof-10}
\begin{proof}
Work on the event in the statement and absorb the sign vector into the planted tuple.  Fix a
target mode \(i\).  The unnormalized update has the decomposition
\[
\bm T(U_0^{(-i)})=\sigma_i\bm x^{(i)}+\bm q_i,
\qquad
\bm q_i\perp\bm x^{(i)},
\]
where
\[
\sigma_i
=
\beta\prod_{j\ne i}\langle \bm u_0^{(j)},\bm x^{(j)}\rangle
+
\left\langle N^{-1/2}\bm W(U_0^{(-i)}),\bm x^{(i)}\right\rangle,
\qquad
\bm q_i=\bm\Pi_iN^{-1/2}\bm W(U_0^{(-i)}).
\]
Therefore
\[
\sigma_i\ge \beta\gamma_N^{d-1}-L_{\rm init}^{(d)}(U_0),
\qquad
\|\bm q_i\|\le L_{\rm init}^{(d)}(U_0).
\]
Since \(a_N/(\gamma_N^{d-1}r_N)\to0\) and \(r_N/\beta\to0\), we have
\(a_N/(\beta\gamma_N^{d-1})\to0\).  Hence, with probability tending to one,
\[
\sigma_i\ge \frac12\beta\gamma_N^{d-1}.
\]
The elementary normalization estimate in Lemma~\ref{lem:norm-pert-d} gives
\[
\sin\angle(\bm u_1^{(i)},\bm x^{(i)})
\le
\frac{2\|\bm q_i\|}{\sigma_i}
\le
\frac{Ca_N}{\beta\gamma_N^{d-1}}.
\]
Because \(a_N/(\gamma_N^{d-1}r_N)\to0\), the right-hand side is at most \(r_N/\beta\) for all
large \(N\), uniformly over the fixed number of modes.  Thus \(U_1\in\calB_{r_N}^{(d)}\).
The finite-iteration bound and convergence are exactly the deterministic local theory started at
time \(s=1\), using the high-signal form of the local coefficient from
Section~\ref{sec:order-d-extension}.
\end{proof}

\subsection{Proof of Proposition~\ref{prop:centered-gram-weak-d} (Fixed-order centered-Gram weak recovery)}\label{app:proof-11}
\begin{proof}
Fix a mode \(i\), and write
\[
\bm y^{(-i)}:=\bigotimes_{j\ne i}\bm x^{(j)}.
\]
The mode-\(i\) unfolding is
\[
\bm T_{(i)}=\beta\bm x^{(i)}\bm y^{(-i)\top}+N^{-1/2}\bm W_{(i)}.
\]
Expanding the centered Gram matrix gives
\[
\bm G_i
=
\beta^2\bm x^{(i)}\bm x^{(i)\top}
+\beta\bm x^{(i)}\bm r_i^\top
+\beta\bm r_i\bm x^{(i)\top}
+\bm C_i,
\]
where
\[
\bm r_i:=N^{-1/2}\bm W_{(i)}\bm y^{(-i)},
\qquad
\bm C_i:=N^{-1}\bm W_{(i)}\bm W_{(i)}^\top-\frac{\prod_{j\ne i}n_j}{N}\bm I_{n_i}.
\]
The fixed contraction satisfies \(\|\bm r_i\|=O_{\mathbb P}(1)\): indeed
\[
\Ebb\|\bm r_i\|^2
=N^{-1}\sum_{m=1}^{n_i}\sum_k (y^{(-i)}_k)^2
=n_i/N=O(1).
\]
For the covariance term, the unfolding has \(M=n_i\asymp N\) rows and
\(K=\prod_{j\ne i}n_j\asymp N^{d-1}\) columns.  The finite-fourth-moment rectangular
sample-covariance input in Lemma~\ref{lem:aux-rect-cov-moment-d}, with \(p=d-1\) after the
truncation transfer recorded there, gives
\[
\|\bm C_i\|
=
O_{\mathbb P}\!\left(N^{(d-2)/2}\right).
\]
Consequently
\[
\|\bm G_i-\beta^2\bm x^{(i)}\bm x^{(i)\top}\|
=
O_{\mathbb P}\!\left(\beta+N^{(d-2)/2}\right)
=o_{\mathbb P}(\beta^2),
\]
because \(\beta=N^{(d-2)/4}\omega_{N,d}\) and \(\omega_{N,d}\to\infty\).  Davis--Kahan then gives
\[
\sin\angle(\bm u_0^{(i)},\bm x^{(i)})
\le
O_{\mathbb P}\!\left(\frac1\beta+\frac1{\omega_{N,d}^2}\right)
=o_{\mathbb P}(1).
\]
A finite union over modes completes the proof.
\end{proof}

\subsection{Random-matrix inputs and pressed-back same-sample estimates}\label{app:aux-fourth-warm}

The next lemmas isolate the probabilistic estimates used in the same-sample centered-Gram
verification.  They are stated for a fixed pair of modes; because \(d\) is fixed, all assertions hold
jointly over the finitely many pairs after a union bound.

\paragraph{Roadmap for the same-sample proof.}
The proof of \(L_{\rm init}^{(d)}(U_0)=O_{\mathbb P}(1)\) has four layers.  The leave-one
expansion compares each centered-Gram eigenvector with a signal-preserving noise-leave-one
eigenvector.  The pressed-back count expands the resulting \(L\)- and \(Q\)-motions before
conditioning on the deleted slice, so the deleted-coordinate sum is controlled by averaged
diagrams rather than by worst-slice norms.  The residual closure does not treat
reduced-resolvent errors as arbitrary slice-dependent directions; it expands them into the same
displayed \(L/Q\) noise factors and then applies the same count.  The truncation step proves the
bounded-array estimate first and then inserts residual noise factors, paid for by the residual
covariance/Lindeberg input, to return to i.i.d. finite-fourth-moment noise.
The rectangular covariance and residual Lindeberg estimates in
Lemmas~\ref{lem:aux-rect-cov-moment-d}--\ref{lem:aux-residual-cov-d} are the non-elementary
random-matrix inputs used below.  Once this package is accepted, the main same-sample estimate is
the pressed-back free-label count in
Lemmas~\ref{lem:aux-pressed-diagram-count-d}--\ref{lem:aux-pressed-counting-d}, together with the
admissible-replacement and residual-insertion principles recorded below.  In particular, we do not
estimate arbitrary row-dependent terms through an abstract row-level independence principle; all
such replacement terms are expanded back to displayed noise factors before being estimated.
\[
\begin{aligned}
U_0(\bm W)&\leadsto U_{0;i,a}(\bm W^{(-i,a)}),
\qquad
\bm d_{\ell,a}=\bm L_{\ell,a}+\bm Q_{\ell,a}+\bm E_{\ell,a},\\
\sum_a\bigl|\mathcal C_{i,a,S}((\bm d_{\ell,a})_{\ell\in S})\bigr|^2
&\leadsto \text{expanded }L/Q\text{ pressed diagrams}
\leadsto o_{\mathbb P}(1),\\
\widehat{\bm W}_b\to \bm W
&\leadsto \text{residual-inserted diagrams, fixed }b\text{ then }b\to\infty .
\end{aligned}
\]
The proof dependencies are:
\[
\begin{gathered}
\text{rectangular covariance and residual Lindeberg inputs}\\
\Downarrow\\
\text{slice covariance, eigengaps, and leave-one expansion}\\
\Downarrow\\
\text{pressed-back \(L/Q\) diagrams and residual-factor reduction}\\
\Downarrow\\
\text{truncation transfer}
\Rightarrow
L_{\rm init}^{(d)}(U_0)=O_{\mathbb P}(1).
\end{gathered}
\]
Readers interested only in the main proof flow may first read
Subsections~\ref{app:proof-d-averaged-leaveone}--\ref{app:proof-13} and return to the
technical estimates in this subsection as needed.

\paragraph{Notation for Appendix~C.}
Throughout the same-sample proof, \(i\) denotes the target mode of the first API update,
\(\ell\ne i\) denotes the comparison mode whose centered-Gram eigenvector is perturbed, and
\(a\in[n_i]\) is the deleted coordinate in the target mode.  We use the following local notation.
\begin{center}
\begin{tabularx}{\textwidth}{@{}lX@{}}
\toprule
symbol & meaning\\
\midrule
\(\mathcal W_{i,a}\) & noise slice with target-mode coordinate \(a\)\\
\(\bm T^{(-i,a)}\) & tensor with the noise slice \(\mathcal W_{i,a}\) removed and the signal kept\\
\(\bm G_{\ell;i,a}\) & recentered mode-\(\ell\) centered Gram matrix from \(\bm T^{(-i,a)}\)\\
\(\bm u_{0;i,a}^{(\ell)}\) & leading eigenvector of \(\bm G_{\ell;i,a}\)\\
\(\bm d_{\ell,a}\) & difference \(\bm u_0^{(\ell)}-\bm u_{0;i,a}^{(\ell)}\)\\
\(\bm L_{\ell,a},\bm Q_{\ell,a},\bm E_{\ell,a}\) & linear, centered-covariance, and residual parts of \(\bm d_{\ell,a}\)\\
\(\bm W_{\ell;i,a}\) & mode-\(\ell\) slice matrix associated with \(\mathcal W_{i,a}\)\\
\bottomrule
\end{tabularx}
\end{center}

\begin{lemma}[Rectangular covariance moment bound]\label{lem:aux-rect-cov-moment-d}
Let \(M\asymp N\) and \(K\asymp N^p\) for a fixed integer \(p\ge1\).  Let
\(\bm A,\bm B\in\mathbb R^{M\times K}\) have centered entries such that the pairs
\((A_{mk},B_{mk})\) are independent over \((m,k)\), with entry magnitudes bounded by \(b\).  Set
\[
\nu_A:=\sup_{m,k}\Ebb A_{mk}^4,\qquad
\nu_{AB}^{\times}:=\sup_{m,n,k}\Ebb A_{mk}^2B_{nk}^2 .
\]
Then
\begin{align}
\Ebb\left\|
\frac1N\left(\bm A\bm A^\top-\Ebb\bm A\bm A^\top\right)
\right\|_{\op}^2
&\le C_{p,b}\nu_A\,N^{p-1}, \label{eq:rect-cov-moment-self}\\
\Ebb\left\|
\frac1N\left(\bm A\bm B^\top-\Ebb\bm A\bm B^\top\right)
\right\|_{\op}^2
&\le C_{p,b}\nu_{AB}^{\times}\,N^{p-1}.
\label{eq:rect-cov-moment-cross}
\end{align}
The constants can be chosen so that
\[
C_{p,b}\le C_p(1+b)^{c_p}.
\]
The same estimates hold for triangular arrays with \(b=b_N\), with
\(C_{p,b_N}\le C_p(1+b_N)^{c_p}\).  In the bounded-array parts of the proof below, these fixed
polynomial losses are always paired with explicit negative powers of \(\omega_{N,d}\).  The
statement is deliberately formulated in terms of fourth and cross-row mixed second moments, not in terms of
a variance factor raised to the fourth power: sparse bounded variables need not have covariance
fluctuations of order \((\operatorname{Var}A)^2\).  Unbounded finite-fourth-moment residual
factors are handled separately in Lemma~\ref{lem:aux-residual-cov-d}, not by imposing a
polynomial-loss condition on their variance.
\end{lemma}

\begin{remark}[Role of Lemma~\ref{lem:aux-rect-cov-moment-d}]\label{rem:rect-cov-input}
Lemma~\ref{lem:aux-rect-cov-moment-d} is the only random-matrix estimate used in the
same-sample proof.  We give the proof because the exact form needed below is slightly more
specialized than the most common Bai--Yin statement: it is rectangular, allows bounded triangular
arrays, records polynomial dependence on the entry bound, and includes a cross-covariance version.
The proof is the standard rectangular trace-moment argument behind the
Yin--Bai--Krishnaiah and Bai--Silverstein--Yin sample-covariance theory
\citep{YinBaiKrishnaiah1988,BaiSilversteinYin1988}; the cross-covariance estimate follows from
the same enumeration after a self-adjoint dilation.

We also use the following Lindeberg form, which is proved by fixed-level truncation, applying
Lemma~\ref{lem:aux-rect-cov-moment-d}, and then sending the fixed truncation level to infinity:
if one covariance factor has uniformly vanishing second, fourth, or mixed tail variance, then the
same normalized covariance or cross-covariance operator is
\(o_{\mathbb P}(N^{(p-1)/2})\).
\end{remark}

\paragraph{Auxiliary Lemma C.1a (rectangular closed-word count).}\label{lem:aux-closed-word-count-d}
Fix an integer \(r\ge1\).  In the trace expansion of a centered rectangular covariance word of
length \(2r\), consider the displayed part after expanding each centered bracket into random
edges and deterministic covariance edges.  If the displayed part uses exactly \(v\) distinct
column labels, \(1\le v\le r\), and no displayed column label occurs only once, then the total
number of row and column assignments that can contribute is at most
\begin{equation}\label{eq:c1-closed-word-count}
C^rK^vM^{2r-v+1}.
\end{equation}
Consequently, when \(M\asymp N\), \(K\asymp N^p\), and \(p\ge1\),
\[
K^vM^{2r-v+1}
\le C_p^r(KM)^rM .
\]

\emph{Proof.}
Encode a contributing word by its rooted bipartite traversal: start from \(m_1\), read the cyclic
word
\[
m_1,k_1,m_2,k_2,\ldots,m_{2r},k_{2r},m_1,
\]
and record, at each step, whether the next row or column label is new or is identified with a
previously opened label.  The first row label is free.  A displayed column label cannot occur
exactly once, since the displayed column blocks are independent and centered.  Thus the \(v\) first
visits to displayed column labels have \(v\) later closing visits.  Each first visit to a displayed
column may expose at most two adjacent row labels, while its closing visit identifies at least one
previously open row label.  Therefore a rooted pattern with \(v\) displayed columns has at most
\[
2r-v+1
\]
free row labels.  Deterministic covariance edges, repeated rows, projections onto previously seen
labels, and higher collisions add equality constraints and cannot increase this number.  For a
fixed rooted pattern, the numerical assignments are therefore at most \(K^vM^{2r-v+1}\).

It remains to count rooted patterns.  After suppressing numerical labels, a pattern is determined
by a length-\(2r\) exploration over the finite alphabet
\[
\{\text{new row},\text{old row},\text{new column},\text{old column},
  \text{deterministic covariance edge}\},
\]
together with the closing choices for previously discovered displayed columns.  Since every
displayed column discovery has a later closing visit, the discovery/closing structure is dominated
by a Catalan exploration with \(2r\) half-edges; higher-degree collisions only identify several
closing visits with the same previous column and therefore coarsen the exploration.  The finite
local choices at each half-edge give at most \(C^r\) rooted patterns.  Multiplying by the
assignment bound proves \eqref{eq:c1-closed-word-count}.  The final displayed inequality follows
from \(K\asymp N^p\), \(M\asymp N\), and \(p\ge1\).  \(\square\)

\begin{proof}
We first prove the self-covariance estimate.  Write \(\bm a_k\) for the \(k\)-th column of
\(\bm A\), and set
\[
\bm X_k:=\bm a_k\bm a_k^\top-\Ebb\,\bm a_k\bm a_k^\top,\qquad
\bm X:=\sum_{k=1}^K \bm X_k
=\bm A\bm A^\top-\Ebb\bm A\bm A^\top .
\]
The matrices \((\bm X_k)_{k=1}^K\) are independent, centered, and symmetric.  It is enough to prove
\begin{equation}\label{eq:c1-high-moment-self}
\Ebb\,\mathrm{tr}\,\bm X^{2r}
\le
\bigl(C_p(1+b)^{c_p}KM\nu_A\bigr)^r M
\end{equation}
for every integer \(r\ge1\), with constants exponential at most in \(r\) and polynomial in \(b\)
after taking the \(r\)-th root.  Indeed, applying
\(\|\bm X\|_{\op}^{2r}\le \mathrm{tr}\,\bm X^{2r}\), then choosing
\(r=\lceil \log M\rceil\), gives
\[
\Ebb\|\bm X\|_{\op}^2
\le
\bigl(\Ebb\|\bm X\|_{\op}^{2r}\bigr)^{1/r}
\le
C_p(1+b)^{c_p}KM\nu_A ,
\]
because \(M^{1/r}\le e\).  This is the usual moment-to-operator step and is the reason no
dimension logarithm appears in the final \(L^2\) bound.

We now prove \eqref{eq:c1-high-moment-self}.  Expanding the trace gives
\[
\Ebb\,\mathrm{tr}\,\bm X^{2r}
=
\sum_{\bm m\in[M]^{2r}}
\sum_{\bm k\in[K]^{2r}}
\Ebb\prod_{s=1}^{2r}
\left(
A_{m_s k_s}A_{m_{s+1}k_s}
-\Ebb A_{m_s k_s}A_{m_{s+1}k_s}
\right),
\qquad m_{2r+1}=m_1 .
\]
For a fixed word \((\bm m,\bm k)\), expand the centered brackets into displayed random edges and
deterministic covariance edges.  Deterministic covariance edges only identify row labels and
therefore never increase the number of free row labels.  Since the column blocks are independent
and centered, a nonzero displayed contribution has no column label appearing exactly once.  Let
\(v\) be the number of distinct column labels in the displayed part.  Then \(v\le r\).

By Auxiliary Lemma C.1a, the total number of assignments with \(v\) displayed column labels is at
most \(C^rK^vM^{2r-v+1}\le C_p^r(KM)^rM\).

It remains to bound the contribution of one canonical nonzero pattern.  The entries are bounded
by \(b\).  In each displayed column block, pair the incident displayed factors along the cyclic
word.  A pair contributes at most an \(L^2\)-moment, and
\[
\Ebb A_{mk}^2\le (\Ebb A_{mk}^4)^{1/2}\le \nu_A^{1/2}.
\]
If more than two factors collide at the same entry, the extra factors are bounded by powers of
\(b\), and one still keeps the same \(L^2\)-pairing contribution.  Since there are \(2r\) displayed
covariance factors and hence \(2r\) such \(L^2\)-pairs, every nonzero canonical pattern is bounded by
\[
C^r(1+b)^{c r}\nu_A^r .
\]
Combining the pattern count, the free-label bound, and the moment bound gives
\eqref{eq:c1-high-moment-self}.

We next prove the cross-covariance estimate.  Let \(\bm a_k,\bm b_k\) be the \(k\)-th columns of
\(\bm A,\bm B\).  Consider the self-adjoint dilation
\[
\bm Y:=
\begin{pmatrix}
0 & \bm A\bm B^\top-\Ebb\bm A\bm B^\top\\
\bm B\bm A^\top-\Ebb\bm B\bm A^\top & 0
\end{pmatrix}
=\sum_{k=1}^K
\begin{pmatrix}
0 & \bm a_k\bm b_k^\top-\Ebb\,\bm a_k\bm b_k^\top\\
\bm b_k\bm a_k^\top-\Ebb\,\bm b_k\bm a_k^\top & 0
\end{pmatrix}.
\]
Then
\[
\left\|\bm A\bm B^\top-\Ebb\bm A\bm B^\top\right\|_{\op}
\le \|\bm Y\|_{\op}.
\]
The trace expansion of \(\mathrm{tr}\,\bm Y^{2r}\) is the same bipartite closed-word expansion,
except that displayed edges now alternate between \(A\)- and \(B\)-entries.  The same
centered-column rule eliminates singleton column labels, and the same free-label count gives at
most \(2r-v+1\) free row labels for \(v\le r\) displayed column labels.  The moment bound for a
canonical pattern is obtained by pairing displayed \(A\)- and \(B\)-factors within each column
block.  A pair may use the same row or two different rows; this is exactly why the lemma uses the
cross-row parameter \(\nu_{AB}^{\times}\), since
\[
\Ebb A_{mk}^2B_{nk}^2\le \nu_{AB}^{\times},
\]
uniformly over all \(m,n,k\).  When \(m\ne n\), the independence of the entry-pairs
\((A_{mk},B_{mk})\) and \((A_{nk},B_{nk})\) gives
\(\Ebb A_{mk}^2B_{nk}^2=\Ebb A_{mk}^2\,\Ebb B_{nk}^2\), which is also included in
\(\nu_{AB}^{\times}\).  Higher collisions are again absorbed by \((1+b)^{c r}\).  Therefore
\[
\Ebb\,\mathrm{tr}\,\bm Y^{2r}
\le
\bigl(C_p(1+b)^{c_p}KM\nu_{AB}^{\times}\bigr)^r M .
\]
The same moment-to-operator step with \(r=\lceil\log M\rceil\) yields
\[
\Ebb\left\|\bm A\bm B^\top-\Ebb\bm A\bm B^\top\right\|_{\op}^2
\le
C_p(1+b)^{c_p}KM\nu_{AB}^{\times}.
\]

Dividing the two unnormalized estimates by \(N^2\) and using \(M\asymp N\), \(K\asymp N^p\),
gives \eqref{eq:rect-cov-moment-self}--\eqref{eq:rect-cov-moment-cross}.  The same proof applies
to triangular arrays with \(b=b_N\), because the constants depend on the entry bound only through
the displayed polynomial factor.
\end{proof}

\begin{lemma}[Residual triangular-array covariance input]\label{lem:aux-residual-cov-d}
Let \(M\asymp N\) and \(K\asymp N^p\) for a fixed integer \(p\ge1\).  Let
\(\bm A,\bm R\in\mathbb R^{M\times K}\) have centered entries such that the pairs
\((A_{mk},R_{mk})\) are independent over \((m,k)\).  Assume
\[
\sup_{m,k}\Ebb A_{mk}^2\le C,\qquad
\sup_{m,k}\Ebb R_{mk}^2\le \tau_N^2,\qquad
\sup_{m,k}\Ebb R_{mk}^4\le \rho_N^4,\qquad
\sup_{m,k}\Ebb A_{mk}^2R_{mk}^2\le \eta_N^2,
\]
where \(\tau_N,\rho_N,\eta_N\to0\).  Then
\begin{align}
\left\|
\frac1N\left(\bm A\bm R^\top-\Ebb\bm A\bm R^\top\right)
\right\|_{\op}
&=o_{\mathbb P}\!\left(N^{(p-1)/2}\right),\label{eq:residual-cov-cross}\\
\left\|
\frac1N\left(\bm R\bm R^\top-\Ebb\bm R\bm R^\top\right)
\right\|_{\op}
&=o_{\mathbb P}\!\left(N^{(p-1)/2}\right).\label{eq:residual-cov-self}
\end{align}
The same bounds hold with the factors transposed or with \(M,K\) replaced by any slice dimensions
of the same polynomial orders.
\end{lemma}

\begin{proof}
We use a two-stage truncation.  Fix an auxiliary level \(B<\infty\) and write
\[
A=A^{\le B}+A^{>B},
\qquad
A^{\le B}_{mk}:=A_{mk}\mathbf 1_{\{|A_{mk}|\le B\}}
-\Ebb[A_{mk}\mathbf 1_{\{|A_{mk}|\le B\}}],
\]
with \(A^{>B}\) defined by centering the complementary tail.  Define similarly
\[
R=R^{\le B}+R^{>B},
\qquad
R^{\le B}_{mk}:=R_{mk}\mathbf 1_{\{|R_{mk}|\le B\}}
-\Ebb[R_{mk}\mathbf 1_{\{|R_{mk}|\le B\}}],
\]
with \(R^{>B}\) defined by centering the complementary tail.  For fixed \(B\), the entries of
\(A^{\le B}\) and \(R^{\le B}\) are bounded and centered.  The cross-row mixed second parameter
needed in Lemma~\ref{lem:aux-rect-cov-moment-d} satisfies
\[
\sup_{m,n,k}\Ebb (A^{\le B}_{mk})^2(R^{\le B}_{nk})^2
\le C_B(\eta_N^2+\tau_N^2)+o_N(1),
\qquad
\sup_{m,k}\Ebb (R^{\le B}_{mk})^4
\le C\rho_N^4+o_N(1),
\]
where the \(m=n\) term is controlled by \(\eta_N^2\), the \(m\ne n\) terms by
\(\sup\Ebb A_{mk}^2\sup\Ebb R_{nk}^2\le C\tau_N^2\), and the harmless centering terms by
Jensen's inequality.  Applying
Lemma~\ref{lem:aux-rect-cov-moment-d} to the self-adjoint dilation of
\(\bm A^{\le B}(R^{\le B})^\top-\Ebb\bm A^{\le B}(R^{\le B})^\top\), with the
bound-dependent constant now allowed to depend on \(B\), gives
\[
\lim_{N\to\infty}
\mathbb P\left\{
\left\|
\frac1N\left(\bm A^{\le B}(R^{\le B})^\top-\Ebb\bm A^{\le B}(R^{\le B})^\top\right)
\right\|_{\op}>\varepsilon N^{(p-1)/2}
\right\}=0
\]
for every fixed \(B\) and \(\varepsilon>0\), since \(\eta_N,\tau_N\to0\).  The bounded part of
\(\bm R\bm R^\top-\Ebb\bm R\bm R^\top\) is controlled in the same way by the fourth-moment factor
\(\rho_N^4\to0\).  This is the point of using the fourth and mixed second moments in
Lemma~\ref{lem:aux-rect-cov-moment-d}; no false \(\tau_N^4\) scaling is invoked.

It remains to remove the fixed auxiliary cutoff.  Any remaining cross term contains at least one
tail factor.  The tail array \(R^{>B}\) has variance
\[
v_B^2:=\sup_{m,k}\Ebb(R^{>B}_{mk})^2
\le
C\sup_{m,k}\Ebb\!\left[R_{mk}^2\mathbf 1_{\{|R_{mk}|>B\}}\right],
\]
and the cross-row mixed tail factor obeys
\[
\sup_{m,n,k}\Ebb (A_{mk}^{>B})^2R_{nk}^2
\le
C\tau_N^2\sup_{m,k}\Ebb(A_{mk}^{>B})^2
+
C\sup_{m,k}\Ebb A_{mk}^2R_{mk}^2\mathbf 1_{\{|A_{mk}|>B\}}.
\]
Under the displayed uniform moment assumptions these quantities are made arbitrarily small by
first sending \(N\to\infty\) and then \(B\to\infty\).  This lemma is precisely the Lindeberg
extension of the rectangular covariance input: applying the fixed-level bounded estimate above to
the tail-truncated arrays and then passing \(B\to\infty\) gives
\[
\limsup_{N\to\infty}
\mathbb P\left\{
\left\|
\frac1N\left(\bm A(R^{>B})^\top-\Ebb\bm A(R^{>B})^\top\right)
\right\|_{\op}>\varepsilon N^{(p-1)/2}
\right\}
\le C_\varepsilon\limsup_N(\eta_N^2+\tau_N^2+v_B^2),
\]
and the same bound covers the terms
\(\bm A^{>B}(R^{\le B})^\top\) and \(\bm A^{>B}(R^{>B})^\top\), because each contains the mixed
tail factor displayed above.  The analogous bound for
\(\bm R^{>B}(\bm R^{>B})^\top-\Ebb\bm R^{>B}(\bm R^{>B})^\top\) is identical, with
\(\rho_N^4+v_B^2\) in place of \(\eta_N^2+v_B^2\).  First take \(N\to\infty\) at fixed \(B\), then
let \(B\to\infty\).  This proves
\eqref{eq:residual-cov-cross}--\eqref{eq:residual-cov-self}.  In particular, the argument never
requires a condition of the form \(b_N^C\tau_N^2\to0\).
\end{proof}

\begin{lemma}[Averaged slice covariance and eigengap bounds]\label{lem:aux-slice-covariance-d}
Assume the i.i.d. finite-fourth-moment centered-Gram setting of
Proposition~\ref{prop:centered-gram-weak-d}.  Fix \(i\ne\ell\).  For the slice quantities
\(\bm h_{\ell;i,a}\), \(\bm K_{\ell;i,a}\), and \(\bm\Delta_{\ell;i,a}\) defined in
Section~\ref{sec:warm-starts-final}, one has
\begin{align}
\sum_a(x_a^{(i)})^2\|\bm h_{\ell;i,a}\|^2&=O_{\mathbb P}(1),
\label{eq:aux-h-avg}\\
\sum_a\|\bm K_{\ell;i,a}\bm x^{(\ell)}\|^2&=O_{\mathbb P}(N^{d-2}),
\label{eq:aux-Kx-avg}\\
\sum_a\|\bm K_{\ell;i,a}\|_{\op}^2&=O_{\mathbb P}(N^{d-2}).
\label{eq:aux-Kop-avg}
\end{align}
Consequently,
\begin{equation}\label{eq:aux-delta-small}
\max_a\frac{\|\bm\Delta_{\ell;i,a}\|_{\op}}{\beta^2}=o_{\mathbb P}(1).
\end{equation}
Moreover, with probability tending to one, the full centered Gram matrix \(\bm G_\ell\) and all
signal-preserving noise-leave-one matrices \(\bm G_{\ell;i,a}\) have a simple top eigenvalue with
eigengap at least \(c\beta^2\), for a constant \(c>0\).
They also obey
\begin{equation}\label{eq:aux-gram-rankone}
\frac{\|\bm G_\ell-\beta^2\bm x^{(\ell)}\bm x^{(\ell)\top}\|_{\op}}{\beta^2}
+
\max_a
\frac{\|\bm G_{\ell;i,a}-\beta^2\bm x^{(\ell)}\bm x^{(\ell)\top}\|_{\op}}{\beta^2}
=o_{\mathbb P}(1).
\end{equation}
\end{lemma}

\begin{proof}
For \eqref{eq:aux-h-avg}, conditional on the deterministic planted vectors, each coordinate of
\(\bm h_{\ell;i,a}=N^{-1/2}\bm W_{\ell;i,a}\bm x^{(-i,\ell)}\) has mean zero and variance \(1/N\).
Thus \(\Ebb\|\bm h_{\ell;i,a}\|^2=n_\ell/N=O(1)\), and
\[
\Ebb\sum_a(x_a^{(i)})^2\|\bm h_{\ell;i,a}\|^2=O(1).
\]
Markov's inequality gives \eqref{eq:aux-h-avg}.

For \eqref{eq:aux-Kx-avg}, write
\[
\bm K_{\ell;i,a}\bm x^{(\ell)}
=
\frac1N\bm W_{\ell;i,a}\bm W_{\ell;i,a}^{\top}\bm x^{(\ell)}
-\frac{P_{\ell;i}}N\bm x^{(\ell)} .
\]
Expanding the squared norm, all terms with an unpaired entry vanish by independence and centering.
The remaining paired terms are bounded by the finite fourth moment and have at most
\(n_\ell P_{\ell;i}\asymp N^{d-1}\) free choices before the outside \(N^{-2}\) factor.  Hence
\[
\Ebb\|\bm K_{\ell;i,a}\bm x^{(\ell)}\|^2\le C N^{d-3},
\]
and summing over \(a\) gives \eqref{eq:aux-Kx-avg}.

For \eqref{eq:aux-Kop-avg}, first truncate the entries at a fixed level \(b\), recenter, and
renormalize.  Lemma~\ref{lem:aux-rect-cov-moment-d} with \(p=d-2\), \(M=n_\ell\), and
\(K=P_{\ell;i}\) gives, for every \(a\),
\[
\Ebb\|\bm K_{\ell;i,a}^{(b)}\|_{\op}^2\le C_bN^{d-3}.
\]
Thus
\[
\Ebb\sum_a\|\bm K_{\ell;i,a}^{(b)}\|_{\op}^2\le C_bN^{d-2},
\]
and Markov's inequality gives the averaged bound for the bounded array.  Finally let \(b\to\infty\)
after \(N\to\infty\).  Since
\(\Ebb[W^2\mathbf 1_{\{|W|>b\}}]\to0\) and \(\Ebb[W^4\mathbf 1_{\{|W|>b\}}]\to0\), the same moment
expansion applied to the difference between the original and truncated slice covariance matrices
transfers the estimate to the finite-fourth-moment array.  Hence
\[
\sum_a\|\bm K_{\ell;i,a}\|_{\op}^2=O_{\mathbb P}(N^{d-2}).
\]
Using the decomposition
\[
\bm\Delta_{\ell;i,a}
=
\beta x_a^{(i)}
(\bm x^{(\ell)}\bm h_{\ell;i,a}^{\top}
+\bm h_{\ell;i,a}\bm x^{(\ell)\top})
+\bm K_{\ell;i,a},
\]
we get
\[
\sum_a\frac{\|\bm\Delta_{\ell;i,a}\|_{\op}^2}{\beta^4}
\le
C\beta^{-2}\sum_a(x_a^{(i)})^2\|\bm h_{\ell;i,a}\|^2
+C\beta^{-4}\sum_a\|\bm K_{\ell;i,a}\|_{\op}^2
=O_{\mathbb P}(\beta^{-2}+\omega_{N,d}^{-4})=o_{\mathbb P}(1),
\]
which implies \eqref{eq:aux-delta-small}.  Finally, the rectangular sample-covariance input gives
\[
\|\bm G_\ell-\beta^2\bm x^{(\ell)}\bm x^{(\ell)\top}\|_{\op}
=O_{\mathbb P}\bigl(\beta+N^{(d-2)/2}\bigr)
=o_{\mathbb P}(\beta^2).
\]
Together with \eqref{eq:aux-delta-small}, this also gives the leave-one part of
\eqref{eq:aux-gram-rankone}.  Weyl's inequality then gives an eigengap
\((1-o_{\mathbb P}(1))\beta^2\) for the full Gram matrix, and the leave-one matrices differ from
the full matrix by \(o_{\mathbb P}(\beta^2)\) uniformly in \(a\), so they have the same eigengap
lower bound.
\end{proof}

\begin{lemma}[Reduced-resolvent eigenvector expansion]\label{lem:aux-resolvent-d}
Let \(\bm A_0\) and \(\bm A=\bm A_0+\bm\Delta\) be symmetric matrices.  Suppose both have simple
leading eigenvalues separated from the rest of their spectra by at least \(c\beta^2\), and suppose
the leading unit eigenvector of \(\bm A_0\) is \(\widetilde{\bm u}\) with
\(\sin\angle(\widetilde{\bm u},\bm x)=o(1)\), and the leading unit eigenvector of \(\bm A\) is
\(\bm u\) with \(\sin\angle(\bm u,\bm x)=o(1)\).  Assume also that
\[
\|\bm A_0-\beta^2\bm x\bm x^\top\|_{\op}
\le \varepsilon_N\beta^2,
\qquad
\varepsilon_N=o(1).
\]
This is an internal resolvent hypothesis, not an additional assumption in the centered-Gram
theorem below; in the application to
\(\bm A_0=\bm G_{\ell;i,a}\) it is verified by
\eqref{eq:aux-gram-rankone}.
If \(\|\bm\Delta\|_{\op}=o(\beta^2)\), then the sign-aligned leading eigenvector \(\bm u\) of
\(\bm A\) satisfies
\[
\bm u-\widetilde{\bm u}
=
\frac1{\beta^2}(\bm I-\bm x\bm x^\top)\bm\Delta\bm x+\bm e,
\]
where
\[
\|\bm e\|
\le
C\left(
\varepsilon_N
+
\sin\angle(\bm u,\bm x)
+
\sin\angle(\widetilde{\bm u},\bm x)
+\frac{\|\bm\Delta\|_{\op}}{\beta^2}
\right)
\frac{\|(\bm I-\bm x\bm x^\top)\bm\Delta\bm x\|}{\beta^2}
+
C\left(
\frac{\|(\bm I-\bm x\bm x^\top)\bm\Delta\bm x\|}{\beta^2}
\right)^2.
\]
\end{lemma}

\begin{proof}
Write the eigenvalue equation for \(\bm A\) and project it onto
\(\widetilde{\bm u}^{\perp}\).  The inverse of
\(\lambda_0\bm I-\bm A_0\) on \(\widetilde{\bm u}^{\perp}\) has norm \(O(\beta^{-2})\) by the
eigengap.  This gives the standard first-order expansion
\[
\bm u-\widetilde{\bm u}
=
(\lambda_0\bm I-\bm A_0)_{\perp}^{-1}
(\bm I-\widetilde{\bm u}\widetilde{\bm u}^{\top})\bm\Delta\widetilde{\bm u}
+O\!\left(\frac{\|\bm\Delta\|_{\op}}{\beta^2}
\frac{\|(\bm I-\widetilde{\bm u}\widetilde{\bm u}^{\top})\bm\Delta\widetilde{\bm u}\|}{\beta^2}\right).
\]
The rank-one approximation of \(\bm A_0\) implies
\(\lambda_0=(1+O(\varepsilon_N))\beta^2\) and, after identifying
\(\widetilde{\bm u}^{\perp}\) with \(\bm x^\perp\),
\[
(\lambda_0\bm I-\bm A_0)_{\perp}^{-1}
=
\beta^{-2}(\bm I-\bm x\bm x^\top)
 +O(\varepsilon_N\beta^{-2})
\]
as an operator on the planted orthogonal space.  Replacing \(\widetilde{\bm u}\) by \(\bm x\) in
the projected perturbation gives the two angle-error terms, and replacing the reduced resolvent by
\(\beta^{-2}(\bm I-\bm x\bm x^\top)\) gives the \(\varepsilon_N\) term.  The last quadratic term is
the normalization and second reduced-resolvent remainder.
\end{proof}

\begin{lemma}[Admissible reduced-resolvent residual expansion]\label{lem:aux-admissible-replacements-d}
In the setting of Lemma~\ref{lem:aux-resolvent-d}, set
\[
\bm w:=\frac1{\beta^2}(\bm I-\bm x\bm x^\top)\bm\Delta\bm x .
\]
The residual \(\bm e\) can be expanded before conditioning in the form
\[
\bm e=\mathcal A_N\bm w+\mathcal B_N(\bm w,\bm w),
\]
where \(\mathcal A_N\) is a linear operator and \(\mathcal B_N\) is a bilinear map satisfying
\[
\|\mathcal A_N\|
\le
C\left(
\varepsilon_N
+
\sin\angle(\bm u,\bm x)
+
\sin\angle(\widetilde{\bm u},\bm x)
+
\frac{\|\bm\Delta\|_{\op}}{\beta^2}
\right),
\qquad
\|\mathcal B_N(\bm y,\bm z)\|\le C\|\bm y\|\,\|\bm z\|.
\]
\end{lemma}

\begin{proof}
The projected eigenvector equation in Lemma~\ref{lem:aux-resolvent-d} writes
\(\bm u-\widetilde{\bm u}\) as a reduced-resolvent applied to
\((\bm I-\bm x\bm x^\top)\bm\Delta\bm x\), plus the errors created by replacing the exact
eigenvectors, resolvent, and normalization by their planted first-order approximations.  The
linear operator \(\mathcal A_N\) is the sum of the errors from replacing
\(\widetilde{\bm u}\) by \(\bm x\), replacing the reduced resolvent by
\(\beta^{-2}(\bm I-\bm x\bm x^\top)\), and keeping the next term in the Neumann expansion of the
reduced resolvent; these errors have the displayed operator norm.  The remaining terms come from
normalization and the second reduced-resolvent correction and are bilinear in the first-order
motion.
In the centered-Gram leave-one application, \(\bm w=\bm L_{\ell,a}+\bm Q_{\ell,a}\).  Thus a
linear residual is a bounded linear postprocessing of one displayed \(L/Q\) motion, while a
quadratic residual is a bounded bilinear pairing of two displayed \(L/Q\) motions in the same
active slot.  These operations do not introduce new deleted-slice noise labels: by Hilbert-space
duality and Cauchy--Schwarz, the bounded linear or bilinear maps are paid for by their operator
norms after the displayed noise factors have been expanded.  This is the only replacement rule used
below; we do not introduce coordinate tensors for \(\mathcal A_N\) or \(\mathcal B_N\).
\end{proof}

\paragraph{Counting convention for pressed diagrams.}
In the next two lemmas, after expanding a squared pressed-slice contraction, an admissible diagram
is a partition of the displayed noise vertices into blocks whose vertices carry the same full tensor
label.  If \(F(\pi)\) denotes the number of free coordinate labels left by a diagram \(\pi\), its
absolute contribution is bounded by \(C_db_N^{C_d}N^{F(\pi)}\), multiplied by the explicit
normalizations and the explicit motion factors.  The free-label lemma below proves that, for a term
with \(q\) moved modes of which \(r\) are \(L\)-motions,
\[
F(\pi)-1\le q-1+(d-3)(q-r),
\]
where the \(-1\) is the contribution of the two outer \(N^{-1/2}\) factors.  This convention is
used only for the bounded-array diagram count; finite-fourth-moment transfer is handled separately
in Lemma~\ref{lem:aux-truncation-transfer-d}.
These two lemmas are the core internal estimate of the same-sample proof: the rectangular
covariance input supplies standard random-matrix control, while the free-label count below is what
prevents the deleted-coordinate sum from being replaced by a worst-case slice norm.

\begin{lemma}[Free-label count for pressed slice diagrams]\label{lem:aux-pressed-diagram-count-d}
Work in the bounded-entry setting \(|\bm W|\le b_N\).  Fix a target mode \(i\), a nonempty
\(S\subseteq I_i\) with \(q=|S|\), and choose for each \(\ell\in S\) either an \(L\)-motion or a
\(Q\)-motion.  If \(r\) of the \(q\) choices are \(L\)-motions, then every nonzero diagram arising
from the expansion of
\[
\Ebb\sum_a
\left|
\mathcal C_{i,a,S}\bigl((\bm D_{\ell,a})_{\ell\in S}\bigr)
\right|^2
\]
has free-label contribution at most
\[
C_db_N^{C_d}N^{q-1}N^{(d-3)(q-r)}
\]
before the explicit \(\beta^{-2r}\beta^{-4(q-r)}\) motion factors are inserted.  Projection terms,
centering terms, same-copy contractions, and third- or fourth-moment collisions do not increase
this free-label count.
\end{lemma}

\begin{proof}
Expand the two copies of the outer contraction and all displayed \(L\)- and \(Q\)-motion entries.
Let \(\mathfrak V\) be the finite set of displayed noise vertices: two outer vertices, one vertex
for each displayed \(L\)-entry, and two vertices for each displayed \(Q\)-entry, in each of the two
copies created by squaring.  Each \(v\in\mathfrak V\) carries a full label
\[
\lambda(v)=(a(v),j_1(v),\ldots,j_d(v)).
\]
An admissible diagram is a partition \(\pi\) of \(\mathfrak V\) such that every block has size at
least two and all vertices in the same block have the same full label.  Blocks of size two are
pairings; blocks of size three or four encode third- and fourth-moment collisions.  After summing
inactive coordinates against fixed unit vectors, let \(F(\pi)\) be the number of remaining free
coordinate labels.  The contribution of \(\pi\) is bounded by \(b_N^{C_d}N^{F(\pi)}\), times the
explicit outside normalizations and the explicit powers of \(\beta\).  Thus it suffices to prove
\[
F(\pi)\le q+(d-3)(q-r)
\]
before the two outer \(N^{-1/2}\) factors are applied.

\emph{Outer-active count.}
The two outer \(N^{-1/2}\) factors give \(N^{-1}\).  Once same-copy contractions between an outer
noise entry and an \(L\)-motion entry are allowed, at most \(q\) active coordinate labels can remain
free in the outer part, hence the enlarged outer contribution is \(N^{q-1}\).  This is deliberately
one power of \(N\) looser than the pure cross-pairing count and covers the \(q=1,r=1\)
self-contraction.  An \(L\)-motion has no auxiliary free slice label after its noise entry is paired
or collided.

\emph{\(Q\)-auxiliary count.}
A \(Q\)-motion contains a centered slice covariance; after its two noise entries are paired, at
most the \(d-3\) coordinates not equal to \(i\) or \(\ell\), and not fixed by the outer coordinate,
remain free.  Thus the \(q-r\) \(Q\)-motions contribute at most
\(N^{(d-3)(q-r)}\).

The count should be read as follows: \(L\)-motions consume active labels but create no auxiliary
slice labels, whereas \(Q\)-motions may create auxiliary slice labels, but at most \(d-3\) per
\(Q\)-motion after the outer coordinate and comparison mode are fixed.

\emph{Collision monotonicity.}
Projection terms replace a displayed vector by a planted direction or subtract its planted
component; this keeps the same identifications or collapses an active sum.  Centering terms replace
a pair of noise entries in \(Q\) by a deterministic Kronecker constraint, so they also only remove
free labels.  If the noise distribution is not symmetric, third-moment diagrams may occur, but they
are collision diagrams: they impose at least one equality relative to the enlarged count, and their
moment is bounded by \(b_N^{C_d}\).  Collisions among different \(Q\)-motions similarly identify
some auxiliary slice labels and cannot increase the product scale.
For reference, the monotonicity rules used in the preceding paragraph are summarized below; each
operation is applied after all displayed noise factors have been expanded, and none creates a new
free coordinate label.  Projection onto \(\bm x^\perp\) subtracts a planted component and can only
collapse an active sum.  A centering term in \(Q\) replaces two noise vertices by a deterministic
Kronecker constraint.  A same-copy outer--\(L\) contraction keeps one active label, which is exactly
the extra \(N\) already allowed in the enlarged \(N^{q-1}\) outer count.  A third- or fourth-moment
collision merges full tensor labels into a single block and therefore removes, rather than creates,
free labels.
This proves the stated free-label bound.
\end{proof}

\begin{lemma}[Bounded-array pressed-back counting]\label{lem:aux-pressed-counting-d}
Assume \(|\bm W|\le b_N\).  Fix \(i\), a nonempty \(S\subseteq I_i\) with \(q=|S|\), and choose
for each \(\ell\in S\) either the linear motion \(\bm L_{\ell,a}\) or covariance motion
\(\bm Q_{\ell,a}\).  If \(r\) of the \(q\) choices are \(L\)-motions, then
\[
\Ebb\sum_a
\left|
\mathcal C_{i,a,S}\bigl((\bm D_{\ell,a})_{\ell\in S}\bigr)
\right|^2
\le
C_db_N^{C_d}
N^{q-1}\beta^{-2r}\beta^{-4(q-r)}N^{(d-3)(q-r)} .
\]
\end{lemma}

\begin{proof}
The conditioning is performed coordinate by coordinate.  For each fixed \(a\), condition on the
outside-slice sigma-field
\(\calF_{i,a}=\sigma(\mathcal W_{i,b}:b\ne a)\); then the leave-one directions appearing in
\(\mathcal C_{i,a,S}\) are fixed unit vectors, and the displayed noise vertices are precisely
those in the deleted slice and in the \(L/Q\) motion factors.  The diagram bound of
Lemma~\ref{lem:aux-pressed-diagram-count-d} is uniform in the realized outside tensor and in
\(a\).  Taking the conditional expectation for each \(a\) and then summing over \(a\) gives the
unconditional estimate.  The explicit definitions of \(\bm L_{\ell,a}\) and \(\bm Q_{\ell,a}\)
give a squared factor \(\beta^{-2}\) for each \(L\)-motion and a squared factor \(\beta^{-4}\)
for each \(Q\)-motion.  The remaining free-label contribution is exactly the one bounded in
Lemma~\ref{lem:aux-pressed-diagram-count-d}.  Thus the scale accounting is
\[
\begin{array}{c|c|c|c}
\text{source} & \mathcal C_{i,a,S}\text{ factor} & L\text{-motions}
& Q\text{-motions}\\
\hline
 \text{all cases}
& N^{q-1}
& \beta^{-2r}
& \beta^{-4(q-r)}N^{(d-3)(q-r)}
\end{array}
\]
This bound is intentionally one factor \(N\) looser than the purely cross-paired estimate when an
\(L\)-factor is present.  It is the safer form under non-Gaussian noise.  For \(q=1,r=1\), it gives
\(O(\beta^{-2})\), matching the possible same-copy contraction of the outer entry with the
\(L\)-entry.  For \(q=1,r=0\), it gives \(O(\beta^{-4}N^{d-3})\).  For \(q=2,r=1\), it gives
\(O(N\beta^{-2}\beta^{-4}N^{d-3})\), which covers both cross-paired and same-copy contractions.
The following table records the worst low-order diagrams that the enlarged count is designed to
cover; all projection and centering variants are obtained from these rows by deleting a noise
vertex or adding a deterministic equality constraint.
\[
\begin{array}{c|c|c|c}
\text{diagram type} & (q,r) & \text{free-label scale before }\beta
& \text{explicit motion scale}\\
\hline
\text{outer--}L\text{ same-copy contraction} & (1,1)
& N^0 & \beta^{-2}\\
\text{outer--}Q\text{ or third-moment collision} & (1,0)
& N^{d-3} & \beta^{-4}\\
Q\text{--}Q\text{ paired/collided slice covariances} & (2,0)
& N^{1+2(d-3)} & \beta^{-8}\\
L\text{--}Q\text{ mixed collision} & (2,1)
& N^{1+d-3} & \beta^{-6}\\
L\text{--}L\text{ double self-contraction} & (2,2)
& N & \beta^{-4}
\end{array}
\]
Third-moment diagrams under non-symmetric noise are included in the collision rows: they identify
at least as many full tensor labels as the displayed paired diagrams and cost only the bounded
moment factor \(b_N^{C_d}\).  Centering subtractions in \(Q\) remove two noise vertices and insert a
Kronecker constraint, so they cannot exceed the corresponding \(Q\)-row scale.
Under \(\beta=N^{(d-2)/4}\omega_{N,d}\), the general \(N\)-power in the displayed bound is
\[
N^{-1+r(4-d)/2}\,
\omega_{N,d}^{-2r-4(q-r)} .
\]
Thus the only borderline fixed-order case is \(d=3,r=2\), where the \(N\)-power is \(N^0\) and
the factor is \(\omega_{N,3}^{-4}\).  The low-order checks are
\[
\begin{array}{c|c|c}
(q,r,d) & \text{bound after substituting }\beta & \text{reason it vanishes}\\
\hline
(1,1,3) & N^{-1/2}\omega_{N,3}^{-2} & \text{negative }N\text{-power}\\
(1,0,d) & N^{-1}\omega_{N,d}^{-4} & \text{negative }N\text{-power}\\
(2,2,3) & \omega_{N,3}^{-4} & \omega_{N,3}\to\infty
\end{array}
\]
\end{proof}

\begin{lemma}[Residual-inserted pressed diagrams]\label{lem:aux-residual-inserted-d}
The pressed-back estimate is stable under at least one residual noise insertion.  This statement is
used in the fixed-level truncation sense.  Fix a truncation level \(b\), write
\(\widehat{\bm W}_b\) for the truncated/recentered/renormalized array and
\(\bm R_b=\bm W-\widehat{\bm W}_b\), and expand every untruncated vertex as
\(\bm W=\widehat{\bm W}_b+\bm R_b\).  If a diagram counted in
Lemma~\ref{lem:aux-pressed-counting-d} has at least one displayed residual vertex, the same
partition count applies: the residual mark changes only the moment attached to the block
containing that vertex, not the number of free coordinate labels.  Hence, for the same \(q,r,S\),
\[
\limsup_{N\to\infty}
\frac{\text{residual-inserted contribution}}
{N^{q-1}\beta^{-2r}\beta^{-4(q-r)}N^{(d-3)(q-r)}}
\le C_d\,\varepsilon(b),
\qquad \varepsilon(b)\downarrow0 .
\]
Here \(\varepsilon(b)\) is controlled by the \(L^2\), \(L^4\), and mixed \(L^2\) tail quantities of
\(\bm R_b\) and \(\widehat{\bm W}_b\), as in Lemma~\ref{lem:aux-residual-cov-d}.  If a block
contains a single residual and non-residual vertices, its absolute moment is bounded by the mixed
tail factor using Hölder and the uniform fourth moment; if it contains two or more residual
vertices, it is bounded by the residual second or fourth tail moment.  Thus the small residual
factor is paid for by tail convergence in the fixed-\(b\), then \(b\to\infty\), argument; it is not
required to absorb a bounded-array polynomial loss \(b_N^{C_d}\).  The same rule covers a residual
vertex created inside the full-vector replacement \(\bm v_\ell\): if its slice label differs from
the outer deleted coordinate \(a\), Lemma~\ref{lem:aux-residual-cov-d} pays for the residual
covariance or mixed covariance block; if it equals \(a\), the term is one of the same-copy
collision diagrams already included in the enlarged pressed-back count.  Since \(d\) is fixed,
only finitely many such placements occur.
\end{lemma}

\begin{lemma}[Admissible residual-factor reduction]\label{lem:aux-E-factor-d}
Work in the bounded-array setting \(|\bm W|\le b_N\), and fix a target mode \(i\).  Suppose the
leave-one residuals in Lemma~\ref{lem:averaged-leaveone-d} are the reduced-resolvent residuals
constructed by the reduced-resolvent expansion in the proof of Lemma~\ref{lem:averaged-leaveone-d}.
In particular, whenever a residual is inserted
into a pressed contraction, it is expanded according to that formula; it is not replaced by an
arbitrary slice-dependent unit direction.  Suppose moreover that
\[
\sum_a\sum_{\ell\ne i}\|\bm E_{\ell,a}\|^2
=o_{\mathbb P}(\beta^{-2}+\omega_{N,d}^{-4})
\]
and the pointwise reduced-resolvent envelope
\[
\|\bm E_{\ell,a}\|
\le
\varepsilon_N\bigl(\|\bm L_{\ell,a}\|+\|\bm Q_{\ell,a}\|\bigr)
+
C\bigl(\|\bm L_{\ell,a}\|+\|\bm Q_{\ell,a}\|\bigr)^2,
\qquad
\varepsilon_N=o_{\mathbb P}(1),
\]
holds uniformly over \(a,\ell\).  Then, for every nonempty \(S\subseteq I_i\), the contribution to
\[
\sum_a
\left|
\mathcal C_{i,a,S}\bigl((\bm d_{\ell,a})_{\ell\in S}\bigr)
\right|^2
\]
from all multilinear terms containing at least one \(\bm E_{\ell,a}\) factor is
\(o_{\mathbb P}(1)\).
\end{lemma}

\begin{proof}
Since \(\varepsilon_N=o_{\mathbb P}(1)\), fix deterministic \(\bar\varepsilon_N\downarrow0\) and
work on an event \(\mathcal H_N\) with probability tending to one on which
\(\varepsilon_N\le\bar\varepsilon_N\), the displayed residual envelope holds, and the
summed-square bounds from Lemma~\ref{lem:averaged-leaveone-d} hold.  All estimates below are
conditional on \(\mathcal H_N\); the complement contributes \(o(1)\) to the final probability
bound.

By multilinearity and fixed \(d\), it is enough to consider one term with a specified
\(\bm E_{\ell_0,a}\) factor.  We do not write
\(\bm E_{\ell_0,a}=\|\bm E_{\ell_0,a}\|\bm z_a\), because the direction \(\bm z_a\) would in
general depend on the deleted slice.  Instead we expand the residual using
Lemma~\ref{lem:aux-admissible-replacements-d}.  On the event \(\mathcal H_N\), every
linear residual term is a basic motion \(\bm L_{\ell_0,a}\) or \(\bm Q_{\ell_0,a}\) multiplied by a
linear coefficient of norm at most \(\bar\varepsilon_N\) after the angle, eigenspace, and
\(\|\Delta\|_{\op}/\beta^2\) errors are absorbed into \(\mathcal A_N\).  Every quadratic residual
term is the bilinear map \(\mathcal B_N\) applied to two copies of
\(\bm L_{\ell_0,a}+\bm Q_{\ell_0,a}\), hence contains at least two basic motions from the same
deleted slice.  Thus, after expanding all displayed noise factors, every monomial containing an
\(\bm E\)-factor is an ordinary \(L/Q\) pressed diagram with either a vanishing scalar coefficient
or with at least one additional basic-motion insertion.

The pressed-counting Lemma~\ref{lem:aux-pressed-counting-d} applies to these expanded diagrams
before any norm-envelope compression is made.  Cauchy--Schwarz is used only after the residual has
been expanded into displayed noise factors; deterministic or leave-one-measurable coefficient
operators from Lemma~\ref{lem:aux-admissible-replacements-d} are then bounded by their operator norms.
Projection and normalization remainders only replace displayed vectors by planted directions or
add additional basic-motion factors, so they do not create a larger free-label count.  A linear
residual contribution is therefore bounded by \(\bar\varepsilon_N^2\) times a finite sum of the
already-controlled \(L/Q\) diagrams.  A quadratic residual contribution has one extra
basic-motion scale beyond the corresponding \(L/Q\) diagram; after substituting
\(\beta=N^{(d-2)/4}\omega_{N,d}\), this contributes an additional negative power of
\(\beta\) or \(\omega_{N,d}\) and is \(o(1)\).  The worst possibilities are summarized as
\[
\begin{array}{c|c|c}
\text{residual term} & \text{expanded form} & \text{extra factor}\\
\hline
\text{linear }E\text{-term} & \bar\varepsilon_N(L\text{ or }Q) & \bar\varepsilon_N^2\\
\text{quadratic }E\text{ with }L,L & \text{two additional }L\text{ motions} & \beta^{-4}\\
\text{quadratic }E\text{ with }L,Q & \text{one additional }L\text{ and one }Q & \beta^{-2}\omega_{N,d}^{-4}\\
\text{quadratic }E\text{ with }Q,Q & \text{two additional }Q\text{ motions} & \omega_{N,d}^{-8}
\end{array}
\]
In the borderline \(d=3\) case the original two-\(L\) diagram is only
\(\omega_{N,3}^{-4}\); inserting a quadratic residual adds at least the displayed \(\beta^{-4}\)
or \(\omega_{N,3}^{-4}\) factor and therefore cannot restore an \(O(1)\) contribution.
Since \(d\) is fixed, only finitely many residual
positions and expansion terms occur.  All multilinear terms containing at least one admissibly
expanded residual factor are therefore \(o_{\mathbb P}(1)\).
\end{proof}

\begin{lemma}[Truncation transfer for same-sample estimates]\label{lem:aux-truncation-transfer-d}
Assume the i.i.d. finite-fourth-moment setting of
Proposition~\ref{prop:centered-gram-weak-d}.  Let \(\widehat{\bm W}\) denote either the
fixed-level truncated, recentered, and renormalized array \(\widehat{\bm W}_b\), with \(b\to\infty\)
taken after \(N\to\infty\), or a diagonal realization \(\widehat{\bm W}_{b_N}\) chosen as in the
proof.  Set \(\bm R:=\bm W-\widehat{\bm W}\).  In the diagonal formulation, writing hatted
quantities for the same centered Gram and leave-one constructions formed from
\(\widehat{\bm W}\), the following hold:
\begin{align}
\max_i\|\bm G_i(\bm W)-\bm G_i(\widehat{\bm W})\|_{\op}
&=o_{\mathbb P}(\beta^2), \label{eq:trunc-gram-transfer}\\
\max_\ell
\|\bm u_0^{(\ell)}(\bm W)-\bm u_0^{(\ell)}(\widehat{\bm W})\|
&=o_{\mathbb P}(1),\label{eq:trunc-full-evec-transfer}\\
\sum_a\sum_{\ell\ne i}
\left\|
\bigl(\bm u_0^{(\ell)}(\bm W)-\bm u_{0;i,a}^{(\ell)}(\bm W)\bigr)
-
\bigl(\bm u_0^{(\ell)}(\widehat{\bm W})-\widehat{\bm u}_{0;i,a}^{(\ell)}\bigr)
\right\|^2
&=o_{\mathbb P}(1)
\qquad\text{for every fixed }i, \label{eq:trunc-loo-motion-transfer}\\
\sum_a
\left|
\frac1{\sqrt N}\bm R(e_a^{(i)},V_{i,a}^{(-i)})
\right|^2
&=o_{\mathbb P}(1) \label{eq:trunc-residual-transfer}
\end{align}
for every family of unit directions \(V_{i,a}^{(-i)}\) measurable with respect to the tensor entries
outside the deleted slice \(\mathcal W_{i,a}\).  In addition, let
\[
\bm v_\ell:=\bm u_0^{(\ell)}(\bm W)-\bm u_0^{(\ell)}(\widehat{\bm W}),\qquad
\bm s_{\ell,a}:=
\bigl(\bm u_0^{(\ell)}(\bm W)-\bm u_{0;i,a}^{(\ell)}(\bm W)\bigr)
-
\bigl(\bm u_0^{(\ell)}(\widehat{\bm W})-\widehat{\bm u}_{0;i,a}^{(\ell)}\bigr).
\]
For every fixed \(i\ne\ell\), every family of unit directions
\(V_{i,a}^{(-i,\ell)}\) measurable outside \(\mathcal W_{i,a}\), and every
\(\bm Z\in\{\widehat{\bm W},\bm W,\bm R\}\),
\begin{equation}\label{eq:trunc-vector-replacement-transfer}
\sum_a
\left|
N^{-1/2}\bm Z(e_a^{(i)},\bm v_\ell,V_{i,a}^{(-i,\ell)})
\right|^2
+
\sum_a
\left|
N^{-1/2}\bm Z(e_a^{(i)},\bm s_{\ell,a},V_{i,a}^{(-i,\ell)})
\right|^2
=o_{\mathbb P}(1).
\end{equation}
Consequently, any pressed-slice estimate of the
form \eqref{eq:pressed-slice-chaos-d} proved for the bounded array \(\widehat{\bm W}\), with only
fixed powers of the truncation level lost, transfers to the original array \(\bm W\).
\end{lemma}

\begin{proof}
Let \(\tau^2(b):=\Ebb[W-\widehat W_b]^2\), where \(\widehat W_b\) denotes the scalar truncated,
recentered, and renormalized entry at level \(b\).  Since \(\Ebb W^4<\infty\), \(\tau^2(b)\to0\)
as \(b\to\infty\).  Moreover,
\[
\Ebb(W-\widehat W_b)^4\to0,
\qquad
\Ebb\,\widehat W_b^2(W-\widehat W_b)^2\to0
\]
by \(L^4\) convergence and Cauchy--Schwarz.  Fix the finite list of polynomial exponents \(C\)
and negative powers \(\varepsilon\) of \(\omega_{N,d}\) used in the bounded diagram and covariance
estimates.  We choose \(b_N\to\infty\) only to absorb bounded-array polynomial losses against these
negative powers:
\[
(1+b_N)^C\omega_{N,d}^{-\varepsilon}\to0
\]
for all these finitely many pairs \((C,\varepsilon)\).  Residual terms are handled separately by
the residual covariance input through \(L^2\), \(L^4\), and mixed \(L^2\) convergence.  Thus no
tail-rate condition of the form \((1+b_N)^C\tau^2(b_N)\to0\) is imposed.  Equivalently, one may
first fix a truncation level \(b\), prove the bounded-array estimates with constants depending on
that fixed \(b\), take \(N\to\infty\), and then send \(b\to\infty\); the displayed choice of
\(b_N\) is only a diagonal realization of this two-stage argument along the given sequence
\(\omega_{N,d}\to\infty\).
Lemma~\ref{lem:aux-residual-inserted-d} is always read in this fixed-\(b\), then \(b\to\infty\),
sense.
Set
\[
\tau_N^2:=\tau^2(b_N)=\Ebb R^2,
\qquad
\rho_N^4:=\Ebb R^4,
\qquad
\eta_N^2:=\Ebb \widehat W^2R^2.
\]
Since \(\widehat W_b\to W\) in \(L^4\), all three quantities tend to zero.
In every residual replacement below, the small parameter is one of these three quantities: a
single residual linear contraction is controlled by \(\tau_N\), a mixed covariance term
\(\widehat W R\) is controlled by \(\eta_N\), and a pure residual covariance term \(RR^\top\) is
controlled by \(\rho_N\) (or by \(\tau_N\) in second-moment contractions).  These small factors are
never multiplied by the bounded-array polynomial losses; those losses occur only in the already
truncated array and are absorbed by the preceding diagonal choice of \(b_N\).

For \eqref{eq:trunc-gram-transfer}, expand the difference of the centered Gram matrices into the
rank-two signal-residual terms and the residual covariance terms.  The fixed-direction
signal-residual contraction has norm \(O_{\mathbb P}(\beta\tau_N)\).  The covariance difference is
a sum of rectangular sample-covariance and cross-covariance matrices involving at least one
residual entry.  These matrices are formed entrywise from independent pairs
\((\widehat W_{\mathbf j},R_{\mathbf j})\).  The residual factor is not bounded, so we use
Lemma~\ref{lem:aux-residual-cov-d}, with \(p=d-1\) for full unfoldings and the small parameters
\(\tau_N,\rho_N,\eta_N\).  This gives
operator norm \(o_{\mathbb P}(N^{(d-2)/2})\) after the outside \(N^{-1}\) normalization.  Since
\(\beta^2=N^{(d-2)/2}\omega_{N,d}^2\) and \(\omega_{N,d}\to\infty\), the whole Gram difference is
\(o_{\mathbb P}(\beta^2)\).  Choose all full and leave-one eigenvector signs by positive inner
product with the planted vector; on the weak-recovery event this is the same, up to a fixed
sign, as choosing positive inner product with the corresponding truncated eigenvector.
Davis--Kahan then yields \eqref{eq:trunc-full-evec-transfer}; explicitly,
\begin{equation}\label{eq:trunc-full-evec-proof}
\max_\ell
\|\bm u_0^{(\ell)}(\bm W)-\bm u_0^{(\ell)}(\widehat{\bm W})\|
\le
C\beta^{-2}\max_\ell\|\bm G_\ell(\bm W)-\bm G_\ell(\widehat{\bm W})\|_{\op}
=o_{\mathbb P}(1).
\end{equation}

For \eqref{eq:trunc-loo-motion-transfer}, we transfer only the slice-specific leave-one motion,
not the full leave-one eigenvector itself.  This is the point where the common full-eigenvector
difference is removed before summing over \(a\).  Indeed,
\[
\bigl(\bm u_0^{(\ell)}(\bm W)-\bm u_{0;i,a}^{(\ell)}(\bm W)\bigr)
-
\bigl(\bm u_0^{(\ell)}(\widehat{\bm W})-\widehat{\bm u}_{0;i,a}^{(\ell)}\bigr)
\]
is driven by the difference of the deleted-slice perturbations
\(\bm\Delta_{\ell;i,a}(\bm W)-\bm\Delta_{\ell;i,a}(\widehat{\bm W})\), rather than by the global
Gram perturbation common to all deleted coordinates.  More explicitly, with
\[
\bm h^R_{\ell;i,a}:=N^{-1/2}\bm R_{\ell;i,a}\bm x^{(-i,\ell)},
\]
we have
\[
\begin{aligned}
\bm\Delta_{\ell;i,a}(\bm W)-\bm\Delta_{\ell;i,a}(\widehat{\bm W})
&=
\beta x_a^{(i)}
\bigl(\bm x^{(\ell)}(\bm h^R_{\ell;i,a})^\top
+\bm h^R_{\ell;i,a}\bm x^{(\ell)\top}\bigr)\\
&\quad+
\bigl(\bm K_{\ell;i,a}(\bm W)-\bm K_{\ell;i,a}(\widehat{\bm W})\bigr),
\end{aligned}
\]
where \(\bm K_{\ell;i,a}(\cdot)\) denotes the same centered slice-covariance construction applied
to the indicated noise array.  The centered-covariance difference satisfies
\[
\bm K_{\ell;i,a}(\bm W)-\bm K_{\ell;i,a}(\widehat{\bm W})
=
\frac1N\left(
\widehat{\bm W}_{\ell;i,a}\bm R_{\ell;i,a}^{\top}
+\bm R_{\ell;i,a}\widehat{\bm W}_{\ell;i,a}^{\top}
+\bm R_{\ell;i,a}\bm R_{\ell;i,a}^{\top}
\right).
\]
Thus every displayed leading term contains at least one residual factor.  Repeating the averaged leave-one expansion
with \(\bm W-\widehat{\bm W}=\bm R\) inserted in at least one noise slot, the bounds
\eqref{eq:aux-h-avg}--\eqref{eq:aux-Kop-avg} gain a small factor controlled by
\(\tau_N^2+\rho_N^4+\eta_N^2+o(1)\), and the corresponding first-order slice-motion vectors satisfy
\[
\sum_a\sum_{\ell\ne i}
\left\|
\frac1{\beta^2}\bm\Pi_\ell
\bigl(\bm\Delta_{\ell;i,a}(\bm W)-\bm\Delta_{\ell;i,a}(\widehat{\bm W})\bigr)
\bm x^{(\ell)}
\right\|^2=o_{\mathbb P}(1).
\]
The reduced-resolvent remainders are controlled by the same averaged slice covariance estimates
with one residual factor and by the leave-one eigengaps of order \(\beta^2\).  Hence
\begin{equation}\label{eq:trunc-loo-motion-proof}
\sum_a\sum_{\ell\ne i}
\left\|
\bigl(\bm u_0^{(\ell)}(\bm W)-\bm u_{0;i,a}^{(\ell)}(\bm W)\bigr)
-
\bigl(\bm u_0^{(\ell)}(\widehat{\bm W})-\widehat{\bm u}_{0;i,a}^{(\ell)}\bigr)
\right\|^2
=o_{\mathbb P}(1).
\end{equation}

Finally, let \(\calF_{i,a}=\sigma(\mathcal W_{i,b}:b\ne a)\).  The directions
\(V_{i,a}^{(-i)}\) are \(\calF_{i,a}\)-measurable and hence deterministic relative to the residual
entries in the deleted slice.  Each scalar contraction has conditional variance at most
\(\tau_N^2/N\).  Therefore the conditioning is performed row by row:
\[
\sum_a
\Ebb\left[
\left|
N^{-1/2}\bm R(e_a^{(i)},V_{i,a}^{(-i)})
\right|^2
\ \middle|\ \calF_{i,a}
\right]
\le C\frac{n_i}{N}\tau_N^2=o(1),
\]
which proves \eqref{eq:trunc-residual-transfer} by Markov's inequality.

We next prove the vector-replacement estimate
\eqref{eq:trunc-vector-replacement-transfer}.  The common full-vector difference is not treated
as an arbitrary small vector repeated over \(a\).  Instead, compare the two full centered Gram
matrices by the same reduced-resolvent expansion used above.  If
\[
\bm H_\ell^R:=\bm G_\ell(\bm W)-\bm G_\ell(\widehat{\bm W}),
\]
then, on the eigengap event,
\[
\bm v_\ell
=
\frac1{\beta^2}\bm\Pi_\ell\bm H_\ell^R\bm x^{(\ell)}
+\bm e_\ell^{\rm full,tr}.
\]
The matrix \(\bm H_\ell^R\) is the sum of a signal-residual rank-two term and full covariance
terms \(\widehat{\bm W}\bm R^\top+\bm R\widehat{\bm W}^\top+\bm R\bm R^\top\), all with the
appropriate \(N^{-1}\) centering.  Thus every leading monomial in \(\bm v_\ell\) contains at least
one residual factor.  The residual \(\bm e_\ell^{\rm full,tr}\) is controlled by the same
reduced-resolvent envelope as in Lemma~\ref{lem:aux-E-factor-d}, but with full unfoldings instead
of deleted-slice unfoldings.  Consequently, after inserting \(\bm v_\ell\) into a pressed
contraction, we expand it before conditioning; the resulting diagrams are the same pressed
diagrams as in Lemma~\ref{lem:aux-pressed-counting-d}, with one displayed noise factor forced to
be residual and controlled by Lemma~\ref{lem:aux-residual-inserted-d}.  This gives, for every fixed
family of leave-one-measurable unit directions \(V_{i,a}^{(-i,\ell)}\) and every
\(\bm Z\in\{\widehat{\bm W},\bm W,\bm R\}\),
\begin{equation}\label{eq:trunc-v-motion-contraction}
\sum_a
\left|
N^{-1/2}\bm Z(e_a^{(i)},\bm v_\ell,V_{i,a}^{(-i,\ell)})
\right|^2
=o_{\mathbb P}(1).
\end{equation}
This displayed estimate is the full-residual insertion bound used below.  The residual factor in
\(\bm v_\ell\) is not required to lie in the same deleted slice as the outer coordinate \(a\):
when its slice label differs from \(a\), the diagram has one additional residual covariance or
mixed covariance pairing, and Lemma~\ref{lem:aux-residual-cov-d} supplies the small
\(\tau_N,\rho_N,\eta_N\) factor; when it coincides with \(a\), the contribution is one of the
same-copy collision diagrams already covered by the enlarged pressed-back count.

For the final transfer, expand the difference between the original and truncated pressed-slice
expressions one factor at a time.  There are three types of replacements.

\smallskip
\noindent\emph{(i) Outer tensor replacement.}
A summand containing an outer residual tensor \(\bm R\) and only leave-one-measurable directions is
controlled by \eqref{eq:trunc-residual-transfer}.  If the same summand also contains a common
full-eigenvector difference or a slice-specific motion replacement, it is controlled by
\eqref{eq:trunc-v-motion-contraction} or \eqref{eq:trunc-s-motion-contraction} after expanding
that vector difference into residual diagrams.

\smallskip
\noindent\emph{(ii) Motion-factor replacement.}
For a displayed leave-one motion, write
\[
\bm d_{\ell,a}(\bm W)-\widehat{\bm d}_{\ell,a}
=
\underbrace{
\bigl(\bm u_0^{(\ell)}(\bm W)-\bm u_{0;i,a}^{(\ell)}(\bm W)\bigr)
-
\bigl(\bm u_0^{(\ell)}(\widehat{\bm W})-\widehat{\bm u}_{0;i,a}^{(\ell)}\bigr)
}_{\bm s_{\ell,a}}.
\]
The slice-specific motion satisfies
\(\sum_a\sum_{\ell\ne i}\|\bm s_{\ell,a}\|^2=o_{\mathbb P}(1)\) by
\eqref{eq:trunc-loo-motion-transfer}.  We do not condition on \(\bm s_{\ell,a}\): it is not
\(\calF_{i,a}\)-measurable, since the full-minus-leave-one motion is created by putting back the
deleted slice.  Instead, use the reduced-resolvent expansion for the two slice motions.  Uniformly
over \(a\) and \(\ell\ne i\),
\[
\bm s_{\ell,a}
=
\frac1{\beta^2}\bm\Pi_\ell
\bigl(\bm\Delta_{\ell;i,a}(\bm W)-\bm\Delta_{\ell;i,a}(\widehat{\bm W})\bigr)\bm x^{(\ell)}
\;+\;
\bm e_{\ell,a}^{\rm tr},
\]
where \(\bm e_{\ell,a}^{\rm tr}\) satisfies the same residual-factor envelope as in
Lemma~\ref{lem:aux-E-factor-d}.  Every leading term in the displayed expansion contains at least
one residual factor from \(\bm R=\bm W-\widehat{\bm W}\).  Hence any pressed contraction with one
displayed \(\bm s_{\ell,a}\) factor is controlled by the same pressed-back diagram count as
Lemma~\ref{lem:aux-pressed-counting-d}, with one noise vertex forced to be residual and with the
residual-inserted bound above.  The contribution of
\(\bm e_{\ell,a}^{\rm tr}\) is handled by the residual-factor reduction of
Lemma~\ref{lem:aux-E-factor-d}.  In particular, for any fixed family of leave-one-measurable unit
directions \(V_{i,a}^{(-i,\ell)}\) and for
\(\bm Z\in\{\widehat{\bm W},\bm W,\bm R\}\),
\begin{equation}\label{eq:trunc-s-motion-contraction}
\sum_a
\left|
N^{-1/2}\bm Z(e_a^{(i)},\bm s_{\ell,a},V_{i,a}^{(-i,\ell)})
\right|^2
=o_{\mathbb P}(1).
\end{equation}
If the outer tensor is \(\bm R\), the estimate also follows directly from
Lemma~\ref{lem:aux-residual-inserted-d} after expanding \(\bm s_{\ell,a}\); when all remaining direction
slots are leave-one measurable, \eqref{eq:trunc-residual-transfer} gives the corresponding
outer-residual bound.  If the residual factor is internal to \(\bm s_{\ell,a}\),
Lemma~\ref{lem:aux-residual-cov-d} supplies the small \(\tau_N,\rho_N,\eta_N\) factor.  Thus
replacements of displayed \(\bm d_{\ell,a}\) factors are controlled without any outside-slice
measurability claim for \(\bm s_{\ell,a}\).

\smallskip
\noindent\emph{(iii) Non-perturbed leave-one direction replacement.}
For replacements of leave-one directions that are not themselves displayed motion factors,
decompose
\[
\bm u_{0;i,a}^{(\ell)}(\bm W)-\widehat{\bm u}_{0;i,a}^{(\ell)}
=
\underbrace{\bm u_0^{(\ell)}(\bm W)-\bm u_0^{(\ell)}(\widehat{\bm W})}_{\bm v_\ell}
-\bm s_{\ell,a}.
\]
The \(\bm s_{\ell,a}\) part is controlled by \eqref{eq:trunc-s-motion-contraction}.  The common
full-vector part is controlled by \eqref{eq:trunc-v-motion-contraction}; it is expanded into full
residual diagrams before entering the deleted-coordinate sum.  This is why no estimate of the form
\(\sum_a\|\bm v_\ell\|^2=o_{\mathbb P}(1)\) is needed, and no abstract row-array independence
statement is invoked.  Since the order \(d\) is fixed, only
finitely many one-factor replacements of types (i)--(iii) occur, and the bounded-array estimates
control the remaining factors.  Hence the pressed-slice estimate transfers without any strengthened
tail assumption.
\end{proof}

\subsection{Proof of Lemma~\ref{lem:averaged-leaveone-d} (Averaged leave-one expansion)}\label{app:proof-d-averaged-leaveone}
\begin{proof}
Fix a target mode \(i\) and a comparison mode \(\ell\ne i\).  Let \(\bm G_{\ell;i,a}\) be the
mode-\(\ell\) centered Gram matrix computed from the signal-preserving noise-leave-one tensor
\(\bm T^{(-i,a)}\), and set
\[
\bm\Delta_{\ell;i,a}:=\bm G_\ell-\bm G_{\ell;i,a}.
\]
Since only noise is removed, the rank-one signal is unchanged in \(\bm G_{\ell;i,a}\).  Hence
\begin{equation}\label{eq:loo-delta-proof}
\bm\Delta_{\ell;i,a}
=
\beta x_a^{(i)}
\left(
\bm x^{(\ell)}\bm h_{\ell;i,a}^{\top}
+\bm h_{\ell;i,a}\bm x^{(\ell)\top}
\right)
+
\bm K_{\ell;i,a}.
\end{equation}

We first record the averaged perturbation scales used below.  The fixed-contraction part satisfies
\begin{equation}\label{eq:h-average-proof}
\sum_a(x_a^{(i)})^2\|\bm h_{\ell;i,a}\|^2=O_{\mathbb P}(1).
\end{equation}
Indeed, conditional on the deterministic planted vectors, each coordinate of
\(\bm h_{\ell;i,a}\) is \(N^{-1/2}\) times a centered unit-variance linear combination of independent
entries, so \(\Ebb\|\bm h_{\ell;i,a}\|^2=n_\ell/N=O(1)\), and summing with the weights
\((x_a^{(i)})^2\) gives a bounded expectation.  For the covariance part,
\begin{equation}\label{eq:Kx-average-proof}
\sum_a\|\bm K_{\ell;i,a}\bm x^{(\ell)}\|^2=O_{\mathbb P}(N^{d-2}),
\qquad
\sum_a\|\bm K_{\ell;i,a}\|_{\op}^2=O_{\mathbb P}(N^{d-2}).
\end{equation}
The first bound follows by expanding second moments of the centered quadratic form
\(\bm K_{\ell;i,a}\bm x^{(\ell)}\).  The second is exactly the averaged rectangular covariance
input in Lemma~\ref{lem:aux-rect-cov-moment-d}, applied to the independent slice matrices
\(\bm W_{\ell;i,a}\) with \(M=n_\ell\asymp N\) and \(K=P_{\ell;i}\asymp N^{d-2}\), followed by the
finite-fourth-moment truncation transfer used in Lemma~\ref{lem:aux-slice-covariance-d}.
Consequently,
\begin{equation}\label{eq:delta-average-proof}
\sum_a\frac{\|\bm\Delta_{\ell;i,a}\|_{\op}^2}{\beta^4}
\le
C\sum_a\frac{(x_a^{(i)})^2\|\bm h_{\ell;i,a}\|^2}{\beta^2}
+C\sum_a\frac{\|\bm K_{\ell;i,a}\|_{\op}^2}{\beta^4}
=O_{\mathbb P}(\beta^{-2}+\omega_{N,d}^{-4})
=o_{\mathbb P}(1).
\end{equation}
Thus \(\max_a\|\bm\Delta_{\ell;i,a}\|_{\op}/\beta^2=o_{\mathbb P}(1)\).

Proposition~\ref{prop:centered-gram-weak-d} gives
\[
\bm G_\ell=\beta^2\bm x^{(\ell)}\bm x^{(\ell)\top}+\bm E_\ell^{\rm full},
\qquad
\|\bm E_\ell^{\rm full}\|=o_{\mathbb P}(\beta^2).
\]
Combining this with \eqref{eq:delta-average-proof} and Weyl's inequality, with probability
tending to one all full and signal-preserving leave-one matrices appearing for this fixed pair
\((i,\ell)\) have a simple top eigenvalue with eigengap at least \(c\beta^2\).  On this event,
they also satisfy
\[
\eta_{\ell,a}:=
\frac{\|\bm G_{\ell;i,a}-\beta^2\bm x^{(\ell)}\bm x^{(\ell)\top}\|_{\op}}{\beta^2}
=o_{\mathbb P}(1)
\]
uniformly in \(a\).  Hence the reference matrix in
Lemma~\ref{lem:aux-resolvent-d} has no order-\(\beta^2\) planted-orthogonal bulk.  On the same
event,
the leave-one eigenvectors also satisfy
\[
\theta_{\ell,a}:=\sin\angle(\bm u_{0;i,a}^{(\ell)},\bm x^{(\ell)})=o_{\mathbb P}(1)
\]
uniformly over \(a\), by Davis--Kahan and \(\max_a\|\bm\Delta_{\ell;i,a}\|_{\op}=o_{\mathbb P}(\beta^2)\).
The full angle
\[
\vartheta_\ell:=\sin\angle(\bm u_0^{(\ell)},\bm x^{(\ell)})
\]
is also \(o_{\mathbb P}(1)\) by Proposition~\ref{prop:centered-gram-weak-d}.
All full and leave-one eigenvector signs are chosen on this high-probability event so that their
inner products with \(\bm x^{(\ell)}\) are positive; equivalently, this agrees with the
positive-inner-product convention between \(\bm u_0^{(\ell)}\) and
\(\bm u_{0;i,a}^{(\ell)}\) for all \(a,\ell\).  Hence no sign flip occurs in the uniform
leave-one difference \(\bm d_{\ell,a}\).
Lemma~\ref{lem:aux-resolvent-d} then gives the reduced-resolvent expansion uniformly:
\begin{align}
\bm d_{\ell,a}
&=
\frac1{\beta^2}\bm\Pi_\ell\bm\Delta_{\ell;i,a}\bm x^{(\ell)}
+\bm E_{\ell,a}, \label{eq:resolvent-local-proof}\\
\|\bm E_{\ell,a}\|
&\le
C\left(
\eta_{\ell,a}
+
\vartheta_\ell
+
\theta_{\ell,a}
+
\frac{\|\bm\Delta_{\ell;i,a}\|_{\op}}{\beta^2}
\right)
\left\|
\frac1{\beta^2}\bm\Pi_\ell\bm\Delta_{\ell;i,a}\bm x^{(\ell)}
\right\| \notag\\
&\quad+
C\left\|
\frac1{\beta^2}\bm\Pi_\ell\bm\Delta_{\ell;i,a}\bm x^{(\ell)}
\right\|^2. \notag
\end{align}
This displayed residual keeps the angle-errors from replacing both
\(\bm u_{0;i,a}^{(\ell)}\) and \(\bm u_0^{(\ell)}\) by \(\bm x^{(\ell)}\), the quadratic
reduced-resolvent term, and the normalization remainder explicit.
This is the reduced-resolvent residual expansion used in
Lemma~\ref{lem:aux-E-factor-d}.

Inserting \eqref{eq:loo-delta-proof} into the first term of \eqref{eq:resolvent-local-proof}
gives
\[
\frac1{\beta^2}\bm\Pi_\ell\bm\Delta_{\ell;i,a}\bm x^{(\ell)}
=
\frac{x_a^{(i)}}{\beta}\bm\Pi_\ell\bm h_{\ell;i,a}
+
\frac1{\beta^2}\bm\Pi_\ell\bm K_{\ell;i,a}\bm x^{(\ell)}
=:\bm L_{\ell,a}+\bm Q_{\ell,a}.
\]
Equations \eqref{eq:h-average-proof} and \eqref{eq:Kx-average-proof} yield
\[
\sum_a\|\bm L_{\ell,a}\|^2=O_{\mathbb P}(\beta^{-2}),
\qquad
\sum_a\|\bm Q_{\ell,a}\|^2=O_{\mathbb P}(\omega_{N,d}^{-4}).
\]
Finally define
\[
\mathfrak r_{\ell,a}:=
\frac{|x_a^{(i)}|\|\bm h_{\ell;i,a}\|}{\beta}
+
\frac{\|\bm K_{\ell;i,a}\bm x^{(\ell)}\|}{\beta^2}.
\]
The residual bound in \eqref{eq:resolvent-local-proof}, together with
\(\max_a\|\bm\Delta_{\ell;i,a}\|_{\op}/\beta^2=o_{\mathbb P}(1)\), gives
\begin{align*}
\sum_a\|\bm E_{\ell,a}\|^2
&\le
C\sum_a
\left(
\eta_{\ell,a}^2
+
\vartheta_\ell^2
+
\theta_{\ell,a}^2
+
\frac{\|\bm\Delta_{\ell;i,a}\|_{\op}^2}{\beta^4}
\right)
\mathfrak r_{\ell,a}^2
+
C\sum_a\mathfrak r_{\ell,a}^4
+o_{\mathbb P}(\beta^{-2}+\omega_{N,d}^{-4})
\\
&=o_{\mathbb P}(\beta^{-2}+\omega_{N,d}^{-4}),
\end{align*}
where the last step uses
\[
\vartheta_\ell+\max_a\theta_{\ell,a}+\max_a\eta_{\ell,a}
+\max_a\frac{\|\bm\Delta_{\ell;i,a}\|_{\op}}{\beta^2}
+\max_a\mathfrak r_{\ell,a}
=o_{\mathbb P}(1)
\]
and
\(\sum_a\mathfrak r_{\ell,a}^2=O_{\mathbb P}(\beta^{-2}+\omega_{N,d}^{-4})\).  In particular,
\(\max_a\mathfrak r_{\ell,a}\le(\sum_a\mathfrak r_{\ell,a}^2)^{1/2}=o_{\mathbb P}(1)\), because
\(\beta\to\infty\) and \(\omega_{N,d}\to\infty\).  A finite union over
the fixed set of modes \(\ell\ne i\) proves the lemma.
\end{proof}

\subsection{Proof of Lemma~\ref{lem:pressed-slice-chaos-d} (Pressed-back directional slice chaos)}\label{app:proof-d-pressed-slice-chaos}
\begin{proof}
We first work with bounded entries, \(|\bm W|\le b_N\), where \(b_N\to\infty\) will be chosen
slowly.  All estimates below are conditional on the sigma-field generated outside the deleted
slice \(\mathcal W_{i,a}\); then the leave-one directions
\((\bm u_{0;i,a}^{(q)})_{q\ne i}\) are deterministic unit vectors.  The outer deleted-slice
contraction and the displayed motion factors \(\bm L_{\ell,a},\bm Q_{\ell,a},\bm E_{\ell,a}\)
still depend on the deleted slice and are expanded together in the diagram count below.

Fix a nonempty \(S\subseteq I_i\), put \(q=|S|\), and insert
\[
\bm d_{\ell,a}=\bm L_{\ell,a}+\bm Q_{\ell,a}+\bm E_{\ell,a}.
\]
Because \(d\) is fixed, it suffices to bound one multilinear choice of the factors.  We first take
each active factor to be either \(\bm L\) or \(\bm Q\).  Let \(r\) be the number of \(L\)-factors.
For coordinates in an active mode \(\ell\),
\[
(\bm L_{\ell,a})_{j_\ell}
=
\frac{x_a^{(i)}}{\beta\sqrt N}
\sum_{\mathbf k\in\prod_{t\ne i,\ell}[n_t]}
\bm W_{a,j_\ell,\mathbf k}\,
x^{(-i,\ell)}_{\mathbf k}
+\text{projection term},
\]
and
\[
(\bm Q_{\ell,a})_{j_\ell}
=
\frac1{\beta^2}
\left\{
\frac1N
\sum_{m_\ell,\mathbf k}
\bm W_{a,j_\ell,\mathbf k}
\bm W_{a,m_\ell,\mathbf k}x_{m_\ell}^{(\ell)}
-\frac{P_{\ell;i}}{N}x_{j_\ell}^{(\ell)}
\right\}
+\text{projection term}.
\]
The displayed parts are the only terms that can maximize the number of free indices.  A projection
term replaces a vector by a deterministic multiple of \(\bm x^{(\ell)}\) or subtracts such a
component, and a centering term replaces
\(\bm W_{a,j_\ell,\mathbf k}\bm W_{a,m_\ell,\mathbf k}\) by a deterministic Kronecker constraint;
neither operation creates a new noise index.  Expanding
\[
\mathcal C_{i,a,S}((\bm D_{\ell,a})_{\ell\in S})
=
\frac1{\sqrt N}
\sum_{\mathbf j}
\bm W_{a,\mathbf j}
\prod_{\ell\in S}(\bm D_{\ell,a})_{j_\ell}
\prod_{t\in I_i\setminus S}(\bm u_{0;i,a}^{(t)})_{j_t},
\]
squaring, and summing over \(a\), every inactive coordinate is absorbed by a unit-vector identity,
for example \(\sum_{j_t}(u_{0;i,a}^{(t)})_{j_t}^2=1\).  We use the enlarged outer count
\(N^{q-1}\): the two outer normalizations give \(N^{-1}\), and at most \(q\) active coordinate
sums can remain once same-copy self-contractions with \(L\)-motions are allowed.

Since the entries are centered and bounded, the conditional expectation of a monomial vanishes
unless no displayed full tensor index occurs exactly once.  Equivalently, write the full tensor
indices appearing in the two copies as a finite multiset of tuples \((a,j_1,\ldots,j_d)\), together
with the auxiliary slice indices \(\mathbf k\) introduced by the \(L\)- and \(Q\)-motions.  Leading
diagrams are pairings, either across copies or inside a copy between an outer entry and an
\(L\)-entry.  Since the noise need not be symmetric, third-moment diagrams may also occur; these
are collision diagrams and impose additional equalities among the displayed tuples.  Each
\(Q\)-motion may leave at most \(d-3\) free slice coordinates after its two noise entries are
paired; collisions among different \(Q\)-motions impose extra equalities among their
\(\mathbf k\)'s and therefore only lower this count.  The projection and centering terms described
above keep the same identifications or add deterministic equalities.  Thus no omitted term has a
larger free-index count than the enlarged \(N^{q-1}\) bound, and bounded coincident higher moments
cost only \(b_N^{C_d}\).  This pairing/collision bound gives
\begin{equation}\label{eq:pressed-counting-bound}
\Ebb\sum_a
\left|
\mathcal C_{i,a,S}\bigl((\bm D_{\ell,a})_{\ell\in S}\bigr)
\right|^2
\le
C_db_N^{C_d}
N^{q-1}\beta^{-2r}\beta^{-4(q-r)}N^{(d-3)(q-r)} .
\end{equation}
The factor \(N^{q-1}\) is the enlarged outer-count that allows same-copy self-contractions of an
outer entry with an \(L\)-motion.  Third-moment diagrams are included as collision diagrams and
are absorbed by \(b_N^{C_d}\).  The factor \(N^{(d-3)(q-r)}\) is the product of the slice-covariance
scales of the \(Q\)-motions, and the factors \(\beta^{-2r}\) and \(\beta^{-4(q-r)}\) are the
squared sizes of \(r\) linear and \(q-r\) covariance motions.
Markov's inequality upgrades \eqref{eq:pressed-counting-bound} from expectation to probability.

Substituting \(\beta=N^{(d-2)/4}\omega_{N,d}\), the power of \(N\) in
\eqref{eq:pressed-counting-bound} is
\[
q-1+(d-3)(q-r)-\frac{d-2}{2}r-(d-2)(q-r)
=
-1+\frac{4-d}{2}r.
\]
For \(d\ge4\) this exponent is at most \(-1\).  For \(d=3\), one has \(q\le2\), so the only
nonnegative exponent occurs when \(r=2\), and the remaining factor
\(\omega_{N,3}^{-4}\) still tends to zero.  In every case a negative power of
\(\omega_{N,d}\) remains after the harmless \(N\)-power is accounted for.  Choosing \(b_N\) so
slowly that the fixed factor \(b_N^{C_d}\) is dominated by this negative power proves
\[
\sum_a
\left|
\mathcal C_{i,a,S}\bigl((\bm D_{\ell,a})_{\ell\in S}\bigr)
\right|^2=o_{\mathbb P}(1)
\]
for every \(L/Q\) choice.

It remains to justify factors \(\bm E_{\ell,a}\).  The residual estimate in
Lemma~\ref{lem:averaged-leaveone-d}, together with
\(\vartheta_\ell=o_{\mathbb P}(1)\),
\(\max_a\theta_{\ell,a}=o_{\mathbb P}(1)\),
\(\max_a\|\bm\Delta_{\ell;i,a}\|_{\op}/\beta^2=o_{\mathbb P}(1)\), and
\(\max_a\eta_{\ell,a}=o_{\mathbb P}(1)\), gives the envelope required by
Lemma~\ref{lem:aux-E-factor-d} with
\[
\varepsilon_N
:=
\max_{a,\ell\ne i}\eta_{\ell,a}
+
\max_{\ell\ne i}\vartheta_\ell
+
\max_{a,\ell\ne i}\theta_{\ell,a}
+
\max_{a,\ell\ne i}\frac{\|\bm\Delta_{\ell;i,a}\|_{\op}}{\beta^2}
=o_{\mathbb P}(1).
\]
That lemma controls all multilinear terms containing at least one
\(\bm E_{\ell,a}\) factor by Cauchy--Schwarz and the same bounded-array \(L/Q\) diagram estimates.
Since there are only finitely many subsets and choices of factors, the bounded-entry version of
\eqref{eq:pressed-slice-chaos-d} follows.

Now return to i.i.d. finite fourth moments.  Choose the truncation level \(b_N\to\infty\) as in
Lemma~\ref{lem:aux-truncation-transfer-d}.  All constants in the bounded-array estimates are
uniform over the truncated law except through the fixed polynomial powers of \(b_N\) recorded
above.  The bounded-entry argument just proved applies to the truncated, recentered, and
renormalized array \(\widehat{\bm W}\), because the only price paid in the paired diagrams is a
fixed power of \(b_N\).  Lemma~\ref{lem:aux-truncation-transfer-d}
transfers the Gram eigenspaces, the averaged leave-one motions, and the residual deleted-slice
contractions from \(\widehat{\bm W}\) back to \(\bm W\).  Hence
\eqref{eq:pressed-slice-chaos-d} also holds for the original i.i.d. finite-fourth-moment array.
\end{proof}

\subsection{Proof of Proposition~\ref{prop:d-same-sample-linit} (Same-sample centered-Gram first-sweep noise)}\label{app:proof-d-same-sample-linit}
\begin{proof}
Fix a target mode \(i\).  For each coordinate \(a\in[n_i]\), write
\[
Y_a:=\frac1{\sqrt N}\bm W(e_a^{(i)},U_0^{(-i)}).
\]
Insert the signal-preserving noise-leave-one directions:
\[
Y_a=M_a+R_a,
\qquad
M_a:=\frac1{\sqrt N}\bm W(e_a^{(i)},U_{0;i,a}^{(-i)}).
\]
Let \(\mathcal F_{i,a}\) be the sigma-field generated by all tensor entries outside
\(\mathcal W_{i,a}\).  Then \(U_{0;i,a}^{(-i)}\) is \(\mathcal F_{i,a}\)-measurable and independent
of the deleted slice.  Therefore
\[
\Ebb(M_a^2\mid\mathcal F_{i,a})=\frac1N.
\]
Summing over \(a\) and using \(n_i\asymp N\) gives
\[
\sum_aM_a^2=O_{\mathbb P}(1).
\]
This follows, for instance, by conditional Markov's inequality applied row by row to
\(\Ebb(M_a^2\mid\mathcal F_{i,a})=N^{-1}\).
By multilinearity,
\[
R_a
=
\sum_{\emptyset\ne S\subseteq I_i}
\mathcal C_{i,a,S}\bigl((\bm d_{\ell,a})_{\ell\in S}\bigr),
\]
and Lemma~\ref{lem:pressed-slice-chaos-d} gives \(\sum_aR_a^2=o_{\mathbb P}(1)\).  Thus
\[
\left\|\frac1{\sqrt N}\bm W(U_0^{(-i)})\right\|^2
=
\sum_aY_a^2
\le
2\sum_aM_a^2+2\sum_aR_a^2
=O_{\mathbb P}(1).
\]
A finite union over the target modes proves \eqref{eq:linit-d-proved}.
\end{proof}

\subsection{Proof of Theorem~\ref{thm:centered-gram-ap-d} (Same-sample centered-Gram initialization and alternating power)}\label{app:proof-13}
\begin{proof}
Proposition~\ref{prop:centered-gram-weak-d} gives a constant \(\gamma>0\) such that, after sign
alignment,
\[
\min_i\langle \bm u_0^{(i)},\bm x^{(i)}\rangle\ge\gamma
\]
with probability tending to one.  Proposition~\ref{prop:d-same-sample-linit} gives
\[
L_{\rm init}^{(d)}(U_0)=O_{\mathbb P}(1).
\]
The one-sweep estimate proved in the proof of Theorem~\ref{thm:warm-start-d} gives, on this event,
\[
\max_i\sin\angle(\bm u_1^{(i)},\bm x^{(i)})
\le
\frac{C L_{\rm init}^{(d)}(U_0)}{\beta\gamma^{d-1}},
\]
and therefore
\[
\delta(U_1)=O_{\mathbb P}(1/\beta).
\]
Consequently, for every deterministic sequence \(r_N\to\infty\),
\[
\mathbb P\{U_1\in\calB_{r_N}^{(d)}\}
\ge
\mathbb P\{\delta(U_1)\le r_N/\beta\}\to1.
\]
It remains to check the deterministic local hypotheses with \(r=r_N\).  On the crude
high-probability event of Proposition~\ref{prop:raw-event-hp-d}, \(L=O(1)\) and
\(\Theta=O(N^{(d-2)/2})\).  Also \(r_N/\beta\to0\), so the radius constraint in
Corollary~\ref{cor:fixed-point-d} holds eventually.  Hence
\[
G_{r_N,d}(\beta)
\le
C_dL\left(1+\frac{r_N}{\beta}\right)
+C_d\Theta\frac{r_N^2}{\beta^2}
=
O\!\left(1+\frac{r_N^2}{\omega_{N,d}^2}\right)
=o(\beta),
\]
and
\[
\kappa_{r_N,d}^{\aff}(\beta)
=
O\!\left(\frac1\beta+\frac{r_N}{\omega_{N,d}^2}\right)
=o(1).
\]
The self-map condition follows from
\[
C_dL+\kappa_{r_N,d}^{\aff}(\beta)r_N\le r_N,
\]
because more explicitly
\[
\kappa_{r_N,d}^{\aff}(\beta)r_N
=O\!\left(\frac{r_N}{\beta}+\frac{r_N^2}{\omega_{N,d}^2}\right)
=o(1),
\]
while \(r_N\to\infty\) and \(L=O(1)\).  The same substitution gives
\(\kappa_{r_N,d}^{\ctr}=o(1)\).  Corollary~\ref{cor:fixed-point-d} and
Proposition~\ref{prop:explicit-affine-d} then give geometric convergence and the displayed
finite-iteration \(O_{\mathbb P}(1/\beta)\) bound.
\end{proof}

\bibliographystyle{plainnat}
\bibliography{references}

\end{document}